\documentclass{article} 
\usepackage[margin=1.25in]{geometry}

\usepackage{authblk,mymath}
\usepackage{times}
\usepackage{comment}
\usepackage{tikz,wrapfig}
\usepackage{booktabs,multirow,tabularx,array,caption}
\usepackage[round]{natbib}
\usepackage{tensor}

\newcommand{\R}{\mathbb{R}}

\DeclareMathOperator{\spn}{span}

\let\vec\relax\DeclareMathOperator{\vec}{Vec}

\DeclareMathOperator{\ran}{ran}

\theoremstyle{remark}

\newcommand{\eps}{\epsilon}
\newcommand{\lam}{\lambda}

\title{Principal Component Regression Dominates all Monotone Spectral Filters for Linear Regression}
\author[1]{Juno Kim}
\author[1]{Hengyu Fu}
\author[1,2]{Peter Bartlett}
\author[1]{Jason D. Lee}
\author[1]{Jingfeng Wu}
\affil[1]{University of California, Berkeley}
\affil[2]{Google DeepMind}

\date{\today}
\ifdefined\usebigfont

\usepackage{times}
\usepackage[fontsize=13pt]{scrextend}
\makeatletter
\@ifpackageloaded{geometry}{\AtBeginDocument{\newgeometry{letterpaper,left=1.56in,right=1.56in,top=1.71in,bottom=1.77in}}}{\usepackage[letterpaper,left=1.56in,right=1.56in,top=1.71in,bottom=1.77in]{geometry}}
\AtBeginDocument{\newgeometry{letterpaper,left=1.56in,right=1.56in,top=1.71in,bottom=1.77in}}
\usepackage{hyperref} 
\else
\fi

\begin{document}

\maketitle

\begin{abstract}
We compare the instance-wise, finite-sample risks of \emph{monotone spectral filters} for linear regression, a broad class of estimators including principal component regression (PCR), gradient descent (GD), and ridge regression. We show that PCR \emph{dominates} all monotone spectral filters: compared to any such filter, the risk of optimally tuned PCR is no bigger by a constant factor for all problems. Furthermore, the dominance is \emph{strong} if the filter is separated from step functions (e.g., GD and ridge): there exist problem instances for which the risk of PCR is smaller by a polynomial factor in sample size dependence. Our comparison results show that PCR is optimal and thus admissible among monotone filters, significantly extending \citet{wu2026risk}'s result that GD strongly dominates ridge. From a technical perspective, we establish new upper and lower bounds for general spectral filters, which are instance-wise sharp when specialized to ridge or GD, recovering or improving the best-known bounds.
\end{abstract}



\section{Introduction}

\begin{table}[t]
    \centering
    \captionsetup{skip=5pt}
    \caption{\textbf{Dominance results for linear regression.} For a problem class~$\Pbb$, we say algorithm~$\mathcal A$ \emph{dominates}~$\mathcal B$, or $\mathcal A \preceq_{\Pbb} \mathcal B$, if its (excess) risk is never more than a constant multiple of~$\mathcal B$'s on every problem instance in~$\Pbb$. We say~$\mathcal A$ \emph{strongly} dominates~$\mathcal B$, or $\mathcal A \prec_{\Pbb} \mathcal B$, if additionally its risk is polynomially smaller (w.r.t. sample size) for some instances. We establish that for all well-specified linear regression problems $\Lbb$~\eqref{eq:lbb} with bounded signal-to-noise ratio, PCR strongly dominates GD; previous work showed that GD strongly dominates ridge in the same class \citep{wu2026risk}. Thus, GD and ridge are both inadmissible.\\
    For all linear regression problems with Gaussian design $\Gbb$~\eqref{eq:gbb}, a subset of~$\Lbb$, we prove that PCR dominates all monotone spectral filters, a wide class of regularization-based methods including GD and ridge (\Cref{def:filter}). Thus, PCR is optimal and hence admissible among all such filters. We further prove that PCR strongly dominates all monotone filters that are uniformly separated from step functions (e.g., \Cref{lem:non-step-examples}).}
    \label{tab:known-dominance-results}
    \renewcommand{\arraystretch}{1.4}
    {
    \renewcommand{\tabularxcolumn}[1]{m{#1}}
\begin{tabularx}{0.83\linewidth}{
    @{}
    >{\raggedright\arraybackslash}m{0.25\linewidth}
    >{\raggedright\arraybackslash}X
    >{\raggedright\arraybackslash}m{0.16\linewidth}
    @{}
}
        \toprule
        \multirow{2}{=}{well-specified linear regression problems~$(\Lbb)$}
        & $\mathrm{GD}\prec_{\Lbb}\mathrm{ridge}$
        & \citet{wu2026risk} \\
        \cmidrule(lr){2-3}
        & $\mathrm{PCR}\prec_{\Lbb}\mathrm{GD}$
        & \Cref{thm:pcr-gd} \\
        \midrule
        Gaussian linear regression problems~$(\Gbb)$
        &
        $\mathrm{PCR}\preceq_{\Gbb}$ all monotone filters\vspace{3pt}\newline
        $\mathrm{PCR}\prec_{\Gbb}$ all non-step monotone filters
        & \Cref{thm:general-comparison}\vspace{3pt}\newline \Cref{thm:pcr-sd} \\
        \bottomrule
    \end{tabularx}
    }
\end{table}

In statistical decision theory, one estimator \emph{dominates} another if the risk of the former is always no larger than that of the latter, and \emph{strongly dominates} another if additionally it is sometimes strictly better \citep{wald1950statistical,berger1985statistical}. 
A seminal example is that the James--Stein estimator strongly dominates \emph{ordinary least squares} (OLS) for estimating the mean of Gaussian distributions in three or higher dimensions \citep{james1961estimation}, and hence the latter is inadmissible. Another classical example is Rao--Blackwellization, which always yields an estimator that dominates the original estimator \citep{rao1945information,blackwell1947conditional}.
Note that this comparison is required to hold over all problem instances of interest, providing a much stronger performance guarantee than worst-case performance.

The notion of ``dominance'' has recently been adapted to statistical learning contexts. In linear regression with fixed design, \citet{dhillon2013risk} showed that \emph{principal component regression} (PCR) strongly dominates ridge regression in terms of rate: with comparable regularization (the number of selected features for PCR and the $\ell_2$~penalty for ridge), the excess risk of PCR is always no larger than a constant times that of ridge, while it can be arbitrarily smaller even if ridge is tuned optimally. Similarly, early-stopped \emph{gradient descent} (GD) also strongly dominates ridge for linear regression in fixed design (see, e.g., \citet{ali2019continuous} and \citet[Appendix A]{wu2026risk}).
For linear regression in the more challenging random design setting, a recent work by \citet{wu2026risk} showed that GD strongly dominates ridge, while it is incomparable with online stochastic gradient descent (SGD).\footnote{We use the term ``strong dominance'' for what \citet{wu2026risk} defined as ``dominance.''} 
Such instance-wise risk comparisons offer a powerful perspective for evaluating estimators beyond their worst-case performance \citep{wainwright2019high}: estimators that attain the same minimax rates can now be compared according to their instance-wise rates, under the preorder defined by dominance.



\paragraph{Contributions.}
In this work, we evaluate and compare the instance-wise finite-sample risks of \emph{monotone spectral filters} in linear regression, a broad class of estimators including PCR, GD, and ridge regression (\Cref{def:filter}). 
We show that for linear regression with Gaussian random design and bounded signal-to-noise ratio, \textbf{PCR dominates all monotone spectral filters} (\Cref{thm:general-comparison}) in the same sense as described above (see \Cref{def:dominance}). Moreover, \textbf{PCR strongly dominates all monotone filters that are separated from step functions} (\Cref{thm:pcr-sd}), which includes both GD and ridge. 
Therefore, among monotone filters, PCR is optimal and hence admissible, while GD is inadmissible.
In comparison, previous work only showed that GD strongly dominates ridge \citep{wu2026risk}, i.e., ridge is inadmissible; and for linear regression with fixed design, PCR and GD both dominate ridge \citep{dhillon2013risk,ali2019continuous}.

In addition, we make the following contributions:
\begin{enumerate}

\item We provide \textbf{instance-wise tight upper and lower risk bounds for GD} (\Cref{thm:gd:finite-snr}). These bounds match, up to a constant factor, for all linear regression problems with bounded signal-to-noise ratio and all GD hyperparameters (stepsize and stopping time), and hold under weaker distributional assumptions \citep[\Cref{assum:lbb}, as used by][]{wu2026risk}. 
Thus, we fully characterize GD's implicit regularization and the impact of hyperparameters in linear regression. 
Previously, instance-wise tight bounds were only known for ridge regression \citep{tsigler2023benign} and online SGD \citep{wu2022last,zou2023benign}.

\item We provide \textbf{novel upper and lower risk bounds for PCR} (\Cref{thm:pcr-risk}) that, to our knowledge, improve on the best known results; see \Cref{sec:related} for discussion.
Our bounds further establish that PCR strongly dominates GD in the more general setting of \emph{well-specified} linear regression problems (\Cref{assum:lbb}), matching the setting of \citet{wu2026risk}. 

\item We introduce new techniques for analyzing general spectral filters. Our approach extends classical leave-one-out ideas in ridge analysis by controlling noncommutative matrix perturbations using tools from the theory of Schur multipliers and matrix divided differences.
While these tools were developed in other areas of mathematics \citep[see, e.g.,][]{aleksandrov2016operator}, we demonstrate their utility in the context of statistical learning (see \Cref{sec:general,sec:master-ub}).

\end{enumerate}

The rest of the paper is structured as follows. The setting and main dominance results are presented in \Cref{sec:main}. The bounds for GD and PCR are presented in \Cref{sec:gd,sec:pcr}; these in turn are proved using the general methods developed in \Cref{sec:general}. An overview of related works is given in \Cref{sec:related}. All missing proofs can be found in the appendix.

\section{Main Results}\label{sec:main}

\subsection{Preliminaries}

\paragraph{Linear regression.}
Let $\Hbb$ be a separable Hilbert space of either finite or countably infinite dimension.
Let $\xB \in \Hbb$ and $y \in \Rbb$ be a pair of covariates and response, with population distribution $\mu(\xB, y)$. In linear regression, we seek to minimize the population risk, defined as
\begin{align*}
\risk(\wB) := \Ebb(\xB^\top \wB -y)^2,\quad \wB\in \Hbb,
\end{align*}
where the expectation is over~$\mu(\xB, y)$. Denote the optimal parameter as
\begin{align*}
\wB^* \in \arg\min\risk(\cdot).
\end{align*}
If the optimal parameter is not unique, let~$\wB^*$ be the one with minimum $\ell_2$-norm. The \emph{excess risk} is defined as
\begin{equation*}%
    \excessRisk(\wB):= \risk(\wB) - \risk(\wB^*) = \|\wB-\wB^*\|_{\SigmaB}^2,\quad \wB \in \Hbb.
\end{equation*}
We refer to a linear regression problem by its population probability measure~$\mu(\xB, y)$. 
When necessary, we also write $\excessRisk_\mu$ to emphasize the dependence on the population distribution. Let~$n\ge 1$ be the sample size and let $(\xB_i, y_i)_{i=1}^n$ be~$n$ independent copies of~$(\xB, y)$. We also write 
\begin{align*}
\XB:= \begin{bmatrix}
    \xB_1^\top \\
    \vdots \\ 
    \xB_n^\top 
\end{bmatrix} \in %
\Hbb^n, \quad 
\yB := \begin{bmatrix}
    y_1 \\ 
    \vdots \\
    y_n
\end{bmatrix}\in\Rbb^n .
\end{align*}
The population and empirical covariance of the covariates are denoted as
\begin{align*}
\SigmaB := \Ebb[\xB\xB^\top] \in \Hbb^{\otimes 2}, \qquad \hat\SigmaB := \frac1n \XB^\top\XB \in \Hbb^{\otimes 2}.
\end{align*}
We assume $\tr(\SigmaB)<\infty$ in order for the problem to be learnable. Let the eigendecomposition of~$\SigmaB$ be
\begin{align*}
\SigmaB = \sum_{i\ge 1} \lambda_i\uB_i\uB_i^\top,\quad \lambda_1\ge \lambda_2\ge \dots,
\end{align*}
where $(\lambda_i, \uB_i)_{i\ge 1}$ are the eigenvalues, in non-increasing order, and their corresponding eigenvectors. For an index $k$, allowed to be zero or infinity, we define 
\begin{equation*}
    \SigmaB_{0:k} := \sum_{i\le k} \lambda_i \uB_i\uB_i^\top,\quad 
    \SigmaB_{k:\infty} := \sum_{i>k}\lambda_i\uB_i\uB_i^\top,
\end{equation*}
both of which are positive semidefinite (PSD) matrices in $\Hbb^{\otimes 2}$. We also denote
\begin{align}\label{eq:def-alpha}
\alpha_k &:= \bigg( \frac{\sum_{i>k}\lambda_i}{n} \bigg)^2 + \frac{\sum_{i>k}\lambda_i^2}{n}.
\end{align}

\paragraph{Algorithms.} Principal component regression (PCR) refers to OLS applied to the features selected by principal component analysis. Formally, let $(\hat\lambda_i, \hat\uB_i)_{i\ge 1}$ be the eigendecomposition of~$\hat\SigmaB$, then PCR is defined as 
\begin{equation}\label{eq:pcr}\tag{PCR}
    \hat\wB^{\pcr}_\rho := \frac{1}{n}\Bigg(\sum_{i: \hat\lambda_i > \rho} \hat\lambda_i \hat\uB_i \hat\uB_i^\top \Bigg)^{-1}\XB^\top \yB,
\end{equation}
where the matrix inverse is understood as the Moore-Penrose pseudoinverse throughout the paper. 
Here, the spectral threshold~$\rho\ge 0$ is a hyperparameter controlling the number of features used by PCR.

Gradient descent (GD) refers to the $t$-th GD iterate with fixed stepsize~$\eta>0$, i.e., 
\begin{equation}\label{eq:gd}\tag{GD}
 \hat \wB_t^{\gd}:= \wB_t,\quad \text{where} \ \ \wB_0 = 0, \ \  \wB_{s} = \wB_{s-1}- \frac{\eta}{n}\XB^\top (\XB\wB_{s-1} - \yB) \ \ (s\ge 1).
\end{equation}
We assume $\eta\le (c(\lam_1+\tr(\SigmaB)/n))^{-1}$ for a sufficiently large constant $c>1$ throughout the paper, which ensures GD is in the stable regime, and treat the \emph{stopping time}~$t\ge 0$ as the main hyperparameter.

More generally, we consider the following class of estimators.

\begin{definition}[spectral filters]\label{def:filter}
A \emph{spectral filter} is a measurable function $g:\Rbb_{\ge 0}\to \Rbb$ satisfying $g(0)=0$. The corresponding estimator is defined as
\begin{equation}\label{eq:spec}\tag{Spec}
    \hat\wB_g := \frac1n \psi(\hat\SigmaB) \XB^\top \yB, \quad\text{where} \ \ \psi(z) := \begin{cases}
    g(z)/z & z>0, \\ 0 & z=0.
    \end{cases}
\end{equation}
If moreover $0\le g\le 1$,~$g$ is a \emph{shrinkage} filter; if~$g$ is nondecreasing,~it is a \emph{monotone} filter. The class of all monotone shrinkage filters is denoted by $\mathcal G$. 
\end{definition}

Spectral and shrinkage filters, also called spectral regularization methods, have a long history in statistics and inverse problems, see for instance \citep{kneip1994ordered,engl1996regularization,devito2005learning,devito2005spectral,bauer2007regularization,gerfo2008spectral}. Shrinkage filters can be interpreted as OLS with each feature $(\hat\lam_i,\hat\uB_i)$ weighted (shrunk) by a factor of~$g(\hat\lam_i)$. Monotonicity is natural since lower variance directions should intuitively be given less weight; this property has also been studied in the context of ordered linear smoothers \citep{kneip1994ordered,Golubev_2010}. Clearly, any useful monotone filter must be shrinkage, as otherwise it will incur constant risk. The following classical estimators are all examples of monotone shrinkage filters.

\begin{itemize}[itemsep=0pt]
\item OLS: $g(z) = \mathbf{1}\{z>0\}$
\item PCR: $g(z) = \mathbf{1}\{z>\rho\}$ where $\rho\ge 0$
\item Gradient descent: $g(z) = 1-(1-\eta z)^t$ where $t\in\Nbb,\ \eta>0$
\item Ridge regression: $g(z)=z/(z+\lam)$ where $\lam>0$
\item Iterated Tikhonov (iterated ridge): $g(z) = 1-(\lam/(z+\lam))^p$ where $\lam>0,\ p\ge 1$ \citep{riley1955solving}
\item Power-exponential: $g(z)=1-\exp(-(tz)^p)$ where $t>0,\ p\ge 1$; includes gradient flow ($p=1$) and Gaussian type filter ($p=2$) \citep{calvetti1999iterative}
\item Ridge PCR: $g(z) = \min\{1,z/\rho\}$ where $\rho>0$ \citep{carrasco2007linear}
\end{itemize}
On the other hand, some algorithms such as momentum or accelerated GD are spectral filters but generally not monotone or shrinkage; while SGD or adaptive methods such as Adam are not of the form~\eqref{eq:spec}.

\begin{definition}[dominance and admissibility]\label{def:dominance}
Let $\Pbb$ be a set of linear regression problems. Let $\hat\wB_{\rho}$ and $\hat\wB_{\tau}'$ be two classes of estimators constructed on $n$ samples $(\xB_i,y_i)_{i=1}^n$, indexed by hyperparameters $\rho$ and $\tau$, respectively.
\begin{itemize}
\item We say $\hat\wB_{\rho}$ \emph{dominates} $\hat\wB'_{\tau}$ over $\Pbb$, if there are constants $c_0, n_0\ge 1$ such that for every $\mu \in \Pbb$ and $n\ge n_0$,
\begin{align*}
\text{with probability at least $0.99$},\quad 
   \inf_{\rho} \Ebb [ \excessRisk_{\mu}(\hat \wB_{\rho}) | \XB ] \le c_0 \inf_{\tau} \Ebb \excessRisk_{\mu}(\hat \wB'_{\tau}).
\end{align*}
\item We say $\hat\wB_{\rho}$ \emph{strongly dominates} $\hat\wB'_{\tau}$ over $\Pbb$, if in addition to the above, there exists $d_0>0$ such that for every $n \ge n_0$, there exists $\mu_n \in \Pbb$ satisfying
\begin{align*}
\text{with probability at least $0.99$},\quad 
   \inf_{\rho} \Ebb[ \excessRisk_{\mu_n}(\hat \wB_{\rho} ) | \XB ] \le \frac{c_0}{n^{d_0}} \inf_{\tau}\Ebb \excessRisk_{\mu_n}(\hat \wB'_{\tau}).
\end{align*}
\end{itemize}
If another class $\hat\wB_{\rho}$ strongly dominates~$\hat\wB'_{\tau}$ over $\Pbb$, we say that $\hat\wB'_{\tau}$ is \emph{inadmissible} for $\Pbb$; otherwise, it is \emph{admissible} for $\Pbb$.
\end{definition}

Note that this definition requires comparing a ``high probability'' upper bound for $\hat\wB_\rho$ against an ``in expectation'' lower bound for $\hat\wB'_\tau$. This is a mostly technical issue and can be circumvented by taking a Bayesian perspective and assuming a symmetric prior on the optimal parameter~$\wB^*$; see the discussion in \citep{tsigler2023benign,wu2026risk}. 

\subsection{PCR dominates all monotone filters}

In this subsection, we consider the class~$\Gbb_b$ of all linear regression problems with Gaussian random design and bounded signal-to-noise ratio, i.e.,
\begin{equation}\label{eq:gbb}
    \Gbb_{b} := \big\{\mu(\xB, y): \xB\sim \Ncal(0,\SigmaB),\ \Ebb[\yB\, | \, \xB ] = \xB^\top\wB^*,\
     \sigma^2 / c \le \Ebb[(\yB-\xB^\top\wB^*)^2\, | \, \xB ] \le c \sigma^2\ \text{ a.s., } \|\wB^*\|_{\SigmaB}^2 \le b\sigma^2
     \big\}
\end{equation}
for constants $b,c>0$ (we omit the dependence on~$c$ for brevity). Each problem instance is essentially specified by the triple $(\SigmaB, \wB^*, \sigma^2)$.

Our main result shows that PCR with a well-tuned threshold dominates the class~$\mathcal G$ of all monotone spectral filters over $\Gbb_b$; in other words, PCR is instance-wise optimal among~$\mathcal G$ (up to a constant factor), and in particular, it is admissible.

\begin{theorem}[PCR dominates monotone filters]\label{thm:general-comparison}
Let $\hat\wB^{\pcr}_{\rho}$ be given by \Cref{eq:pcr} with regularization $\rho\ge 0$, and $\hat \wB_g$ be given by \Cref{eq:spec} with any monotone filter~$g\in\mathcal G$. Then $\hat\wB^{\pcr}_{\rho}$ dominates $\{\hat \wB_g:g\in\mathcal G\}$ over $\Gbb_b$. Specifically, there exist constants $c_0,n_0\ge 1$ only depending on~$b$ such that for every $g\in\mathcal G$, $\mu\in \Gbb_b$ and $n\ge n_0$,
\begin{align*}
\text{with probability at least $0.99$}, \quad \inf_{\rho\ge 0}\Ebb[\excessRisk_{\mu}(\hat \wB^{\pcr}_{\rho}) | \XB ] \le c_0 \Ebb  \excessRisk_{\mu}(\hat \wB_g).
\end{align*}
\end{theorem}

\begin{proof}[Proof sketch of \Cref{thm:general-comparison}]
While we obtain risk upper bounds for PCR (\Cref{thm:pcr-risk}) and lower bounds for general monotone filters~$g$ (\Cref{thm:master-lb}), dominance cannot be established by simply comparing these directly, as these bounds can be loose near the spectral threshold even when~$g$ itself is PCR (see \Cref{thm:pcr-risk}). Instead, the proof is divided into two branches. When the scale of the tail of the spectrum is large compared to the transition region of~$g$, we show a stronger lower bound for the bias of~$g$ based on a Cram\'{e}r--Rao inequality, which can be compared against the PCR upper bound. When the tail is small, we instead invoke a `gluing' lemma for linear sums of filters to reduce to the case where~$g$ vanishes on some interval containing the tail. This enables directly comparing the risk of~$g$ against that of PCR with a threshold chosen within the transition region, since the two filters only differ at scales that are much larger than the tail and so exhibit good concentration. 
\end{proof}

Clearly, it is not possible for PCR to \emph{strongly} dominate~$\mathcal G$, since the latter includes PCR itself. Nonetheless, our next example shows that PCR will strongly dominate any subclass which is separated from step functions in the following sense.

\begin{theorem}[PCR strongly dominates non-step monotone filters]\label{thm:pcr-sd}
Fix any $\eps>0$ and $\delta\in(0,1/2)$. The class of \emph{$(\eps,\delta)$-non-step} monotone filters~$\mathcal G_{\eps,\delta}$ is defined as
\begin{align}\label{eq:non-step}
\mathcal G_{\eps,\delta}:= \bigg\{g\in\mathcal G:\ \frac{\inf\{z: g(z)\ge 1-\delta\}}{\sup\{z: g(z)\le \delta\}} \ge 1+\eps \bigg\}.
\end{align}
Then $\hat\wB^{\pcr}_{\rho}$ strongly dominates $\{\hat \wB_g: g\in\mathcal G_{\eps,\delta}\}$ over $\Gbb_b$. Specifically, there exist $c_0,n_0\ge 1$, depending at most polynomially on $\eps^{-1},\delta^{-1}$, and~$b$, for which the following holds: for every $n \ge n_0$, there exists $\mu_n \in \Pbb$ such that
    \begin{align*}\text{with probability at least $0.99$},\quad 
   \inf_{\rho\ge 0} \Ebb[ \excessRisk_{\mu_n}(\hat \wB_{\rho}^{\pcr} ) | \XB ] \le \frac{c_0}{\sqrt{n}} \inf_{g\in\mathcal G_{\eps,\delta}}\Ebb \excessRisk_{\mu_n}(\hat \wB_g).
    \end{align*}
\end{theorem}

\begin{wrapfigure}{r}{0.33\textwidth}
    \centering
    \vspace{-0.8\baselineskip}
    \begin{tikzpicture}[x=0.75cm,y=2.2cm,font=\small]
        \draw[->] (0,0) -- (4.8,0) node[right] {$z$};
        \draw[->] (0,0) -- (0,1.08);
        \def\xa{2.00}
        \def\xb{2.5}
        \def\yd{0.2}      
        \def\yud{0.8}     
        
        \draw[thick]
            plot[smooth] coordinates {
                (0,0)
                (0.55,0.01)
                (1.10,0.05)
                (1.55,0.1)
                (\xa,\yd)
                (2.22,0.40)
                (\xb,\yud)
                (3.1,0.92)
                (3.70,0.96)
                (4.20,0.98)
                (4.55,0.99)
            };

        \draw[dashed] (\xa,0) -- (\xa,\yd);
        \draw[dashed] (\xb,0) -- (\xb,\yud);
        \draw[dashed] (0,\yd) -- (\xa,\yd);
        \draw[dashed] (0,\yud) -- (\xb,\yud);
        \draw[dashed] (0,1) -- (4.65,1);

        \node[below] at (\xa,0) {$1$};
        \node[below] at (\xb,0) {\quad\;$1+\eps$};
        \node[left] at (0,\yd) {$\delta$};
        \node[left] at (0,\yud) {$1-\delta$};
        \node[left] at (0,0) {$0$};
        \node[left] at (0,1) {$1$};
    \end{tikzpicture}
\vspace{-0.2cm}
\captionsetup{labelformat=empty,font=small}
\caption{An $(\eps,\delta)$-non-step monotone filter.}
\label{fig:non-step-filter}
\vspace{-\baselineskip}
\end{wrapfigure}
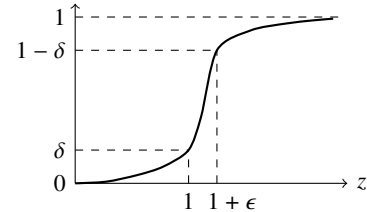

The condition \Cref{eq:non-step} implies that $g$ must not increase from $\delta$ to $1-\delta$ within a short interval of the form $(z,(1+\eps)z)$, disallowing step filters very similar to PCR. Clearly, any monotone shrinkage filter $g\in\mathcal{G}$ that is not contained in $\mathcal{G}_{\eps,\delta}$ for any $\eps,\delta>0$ must be equal to a PCR filter a.e. Moreover, since \Cref{eq:non-step} is invariant under rescaling, for any $g\in\mathcal{G}_{\eps,\delta}$, the family of filters $\{g^{(\tau)}(z) := g(z/\tau):\tau>0\}$ obtained by varying a scale parameter~$\tau$ will also be contained in $\mathcal{G}_{\eps,\delta}$. In particular, the families given in \Cref{lem:non-step-examples} are all non-step monotone filters, and therefore strongly dominated by PCR.

\begin{example}[non-step filters]\label{lem:non-step-examples}
GD (for any~$t\in\Nbb$), ridge (for any~$\lambda>0$), iterated Tikhonov (for any~$p\ge 1$), and ridge PCR are all contained in $\mathcal G_{1,1/3}$. Power-exponential filters with $p\ge 1$ are contained in $\mathcal G_{1,1/(2^p+1)}$.
\end{example}

\subsection{PCR strongly dominates GD}\label{sec:pcr-gd}

Specializing to GD, we now show that PCR also strongly dominates GD over a significantly more general set of linear regression problems given by the below assumptions. Previously, \citet{wu2026risk} showed that GD strongly dominates ridge in the same problem set.

\begin{assumption}[conditions for risk bounds]\label{assum:lbb}
We assume for constants $c_x > 0$ and $c_y \ge 1$ that:
\begin{assumpenum}
    \item \label{assum:item:x} the entries of $\SigmaB^{-1/2}\xB$ are independent, centered and $c_x^2$-subgaussian;
    \item\label{assum:item:noise} the conditional noise is zero mean, and its variance is bounded from above (for upper bounds) and from below (for lower bounds),
    \begin{align*}
    \Ebb[\yB\, | \, \xB] = \xB^\top \wB^*,\quad  \sigma^2 / c_y \le \Ebb[(\yB-\xB^\top\wB^*)^2\, | \, \xB ] \le c_y\sigma^2 \text{ a.s.};
    \end{align*}
    \item\label{assum:item:symmetry} for lower bounds only, the distribution of each component of $\SigmaB^{-1/2}\xB$ is symmetric, i.e., $\la\eB_i, \SigmaB^{-1/2}\xB\ra \stackrel{d}{=} -\la \eB_i, \SigmaB^{-1/2}\xB\ra$ for all $i$.
\end{assumpenum}
\end{assumption}
We remark that the bound from above in \Cref{assum:item:noise} is only required for our risk upper bounds, while the bound from below and \Cref{assum:item:symmetry} are only required for our lower bounds.

For each signal-to-noise ratio $b>0$, we define the set of \emph{well-specified linear regression problems} as
\begin{equation}\label{eq:lbb}
\Lbb_{b} := \big\{\mu(\xB, y) \ \text{satisfying \Cref{assum:lbb} with} \  \|\wB^*\|_{\SigmaB}^2 \le b\sigma^2 \big\}.
\end{equation}
These assumptions are directly borrowed from \citet{wu2026risk}, and we refer the reader to their paper for discussions on the coverage and limitations of the assumptions. As the main example, Gaussian linear regression problems $\mu\in\Gbb_b$ satisfy \Cref{assum:lbb} with $c_x=1$ and $c_y=c$, hence $\Gbb_b \subset \Lbb_b$.

The following result shows that PCR strongly dominates GD over $\Lbb_b$ for each $b>0$. Thus, we establish that GD is also inadmissible for linear regression over the wider class $\Lbb_b$.

\begin{theorem}[PCR strongly dominates GD]\label{thm:pcr-gd}
Let $\hat\wB^{\pcr}_{\rho}$ be given by \Cref{eq:pcr} with regularization $\rho\ge 0$, and~$\hat \wB^{\gd}_{t}$ be given by \Cref{eq:gd} with any fixed stepsize $\eta\le (c(\lam_1+\tr(\SigmaB)/n))^{-1}$ and stopping time $t\ge 0$. Then $\hat\wB^{\pcr}_{\rho}$ strongly dominates $\hat \wB^{\gd}_{t}$ over $\Lbb_b$. 
Specifically,
\begin{itemize}
    \item for every $\mu\in \Lbb_b$ and stopping time $t \ge 0$, there exists a PCR threshold $\rho \ge 0$ such that
\begin{align*}
\text{with probability at least $0.99$}, \quad     \Ebb[ \excessRisk_{\mu}(\hat \wB^{\pcr}_{\rho}) | \XB ] \le c \Ebb  \excessRisk_{\mu}(\hat \wB^{\gd}_{t}) ,
\end{align*}
where $c\ge 1$ only depends on $c_x,c_y$, and $b$.
\item For every $n\ge n_0$, there exist $\mu_n\in\Lbb_b$ and $\rho_n>0$, depending only on $n,b$, satisfying
\begin{align*}
\text{with probability at least $0.99$},\quad 
\Ebb[ \excessRisk_{\mu_n}(\hat \wB_{\rho_n}^{\pcr} ) | \XB ] \le \frac{c}{\sqrt n} \inf_{t}\Ebb \excessRisk_{\mu_n}(\hat \wB_t^{\gd}),
\end{align*}
where $c\ge 1$ is a constant.
\end{itemize}
\end{theorem}

The fact that PCR can be polynomially better than even well-tuned gradient descent is quite surprising. Although a polylogarithmic improvement over GD is somewhat expected due to the bias from optimization, the polynomial advantage actually comes from the variance. While PCR has the ability to zero out smaller spectral components that are irrelevant to learning the signal, GD is forced to partially fit all directions, leading to potentially much larger variance. The same intuition is behind the polynomial separation in \Cref{thm:pcr-sd} for general non-step filters. We believe this is an important insight for designing admissible and computationally efficient algorithms for linear regression and broader noisy learning problems.

The proof of \Cref{thm:pcr-gd} requires an instance-wise risk comparison between PCR and GD for all well-specified linear regression problems, and is deferred to \Cref{sec:pcr}. To this end, we present new risk bounds for GD and PCR in the next two sections.

\section{Tight Risk Bounds for Gradient Descent}\label{sec:gd}

In this section, we provide matching upper and lower bounds on the excess risk of GD for all well-specified linear regression problems with bounded signal-to-noise (SNR) ratio, i.e., over~$\Lbb_b$. This provides, for the first time, a complete picture of the interplay between the optimization error, implicit (tail) regularization, and variance of GD. A more general version without the SNR assumption is given in \Cref{thm:gd:bound} in the appendix. Both the upper and lower bounds improve upon the strongest previously known instance-wise bounds given by \citet{wu2026risk}.

To state the bounds, recall that $(\lambda_i)_{i\ge1}$ denote the ordered eigenvalues of~$\SigmaB$, while $\SigmaB_{0:k}, \SigmaB_{k:\infty}$ denote the truncated covariance matrices.

\begin{corollary}[GD risk under bounded SNR]\label{thm:gd:finite-snr}
There exist constants $c_0, \ldots,c_4 >1$ that only depend on $c_x,c_y$ for which the following hold over~$\Lbb_b$. Let $\hat\wB^{\gd}_t$ be given by \Cref{eq:gd} with stepsize $\eta\le(c_4(\lam_1+\tr(\SigmaB)/n))^{-1}$ and steps $t\ge 1$. Let
\begin{align*}
    k^* := \min\left\{k:  \frac{1}{\eta t} + \frac{\sum_{i>k}\lambda_i}{n} \ge c_2 \lambda_{k+1} \right\},\quad \tilde\lambda:= \frac{1}{\eta t} + \frac{\sum_{i>k^*}\lambda_i}{n},\quad  D := k^* + \frac{1}{\tilde\lambda^2} \sum_{i>k^*}\lambda_i^2,
\end{align*}
and additionally assume $k^*\le n/c_3$.
\begin{itemize}
\item \textbf{Upper bound.} With probability at least $1-\exp(-k^*/c_0)$ over sampling of~$\XB$, it holds that
\newcommand{\overstrut}{%
    \vphantom{
        \bigg( \frac{\sum_{i>k^*}\lambda_i}{n}\bigg)^2
        \|\wB^*\|^2_{\SigmaB_{0:k^*}^{-1}}
    }%
}
\newcommand{\understrut}{%
    \vphantom{
        \overbrace{
            \overstrut
            \bigg( \frac{\sum_{i>k^*}\lambda_i}{n}\bigg)^2
            \|\wB^*\|^2_{\SigmaB_{0:k^*}^{-1}}
        }^{\textup{implicit (tail) regularization}}
    }%
}
\begin{align*}
   \frac{1}{c_1}\Ebb[ \excessRisk(\hat \wB^{\gd}_t) | \XB ] \le 
    \underbrace{
        \understrut
        \overbrace{
            \overstrut
            \big\|(\IB- \eta \SigmaB)^t \wB^* \big\|^2_{\SigmaB_{0:k^*}}
        }^{\textup{optimization error}}
        +
        \overbrace{
            \overstrut
            \bigg( \frac{\sum_{i>k^*}\lambda_i}{n}\bigg)^2
            \|\wB^*\|^2_{\SigmaB_{0:k^*}^{-1}}
        }^{\textup{implicit (tail) regularization}}
    }_{\textup{head bias}}
    +
    \underbrace{
        \understrut
        \|\wB^*\|^2_{\SigmaB_{k^*:\infty}}
    }_{\textup{high-dimensional tail}}
    +
    \underbrace{
        \understrut
        (1+b)\sigma^2 \frac{D}{n}
    }_{\textup{variance}}.
\end{align*}
\item \textbf{Lower bound.} In expectation, 
\begin{align*}
    c_1\Ebb \excessRisk(\hat \wB^{\gd}_t) \ge \big\|(\IB- \eta \SigmaB)^t \wB^* \big\|^2_{\SigmaB_{0:k^*}}+ \bigg( \frac{\sum_{i>k^*}\lambda_i}{n}\bigg)^2  \|\wB^*\|^2_{\SigmaB_{0:k^*}^{-1}}+ \|\wB^*\|^2_{\SigmaB_{k^*:\infty}}
    + \sigma^2\frac{D}{n}.
\end{align*}
\end{itemize}
\end{corollary}

The bounds can be interpreted as follows. The bias is decomposed into the essentially low-dimensional ``head'' and high-dimensional ``tail'' parts according to the critical index~$k^*$ \citep{bartlett2020benign,tsigler2023benign}. The signal from the tail cannot be reliably estimated, so it remains in the risk. The head bias is further characterized as the sum of the GD optimization error, which is time-dependent, and the implicit regularization effect of the tail. The latter arises because the tail collectively acts as an additional ridge-like regularizer with strength $n^{-1}\sum_{i>k^*} \lam_i$. This can be improved with \emph{debiasing} or negative~$\ell_2$ regularization \citep{tsigler2023benign}, however this may lead to an increase in the variance. Finally, the variance is obtained in terms of the effective dimension~$D$, which is the sum of the head rank and contribution of tail fluctuations. The proof of the variance bound is given in \citet{wu2026risk} and is due to \citet{bartlett2020benign}.

\paragraph{Comparison to \citet{wu2026risk}.} Both our upper and lower bounds improve upon the bounds for GD previously shown in \citet{wu2026risk}. Their Theorem~3.1 bounds the head bias with the ridge-like bias $\tilde\lam^2 \|\wB^*\|^2_{\SigmaB_{0:k^*}^{-1}}$, which obscures the optimization error by replacing it with the crude estimate
\begin{align*}
(1-\eta\lam_i)^t \le \frac{1}{\eta\lam_i t} \quad\implies\quad \big\|(\IB- \eta \SigmaB)^t \wB^* \big\|^2_{\SigmaB_{0:k^*}} \le \frac{1}{(\eta t)^2} \|\wB^*\|^2_{\SigmaB_{0:k^*}^{-1}}.
\end{align*}
While this suffices to show that GD is no worse than ridge as in their Theorem~3.2, it is loose for components above the effective regularization threshold~$1/(\eta t)$, since GD optimization error decays exponentially in~$t$. Our result makes the separation of optimization and regularization explicit for every instance.

Their Theorem~4.3 gives a more detailed characterization in terms of \emph{effective bias} and \emph{effective variance}, the latter of which is equivalent to the head term involving $\alpha_{k^*}$ in \Cref{thm:gd:bound}. Here, their order-$1$ effective dimension is the tail regularization, while the $D/n$ term is due to the empirical fluctuation of the tail and has been absorbed into the variance in \Cref{thm:gd:finite-snr}. Nonetheless, their result is restricted to moderate stopping times $t\le cn$ and still obtains a suboptimal bound for the optimization error in the effective bias, with a decay factor of $(1-\eta\SigmaB)^{t/2}$ instead of $(1-\eta\SigmaB)^t$.

Their lower bound Theorem~4.1 also uses the suboptimal OLS critical index $\ell^* := \min\{k: n^{-1}\sum_{i>k}\lam_i \ge c_2\lam_{k+1}\}$, resulting in a crude lower bound which is simply the sum of OLS bias and ridge variance \citep{bartlett2020benign,tsigler2023benign}. On the other hand, our lower bound is tight up to constant factors.

\section{Risk Bounds for Principal Component Regression}\label{sec:pcr}

For PCR at spectral threshold~$\rho$, recalling the definition of~$\alpha_k$ from~\eqref{eq:def-alpha}, we obtain the following upper and lower bounds. The proofs essentially follow from the general upper and lower bounds developed in \Cref{sec:general} and are deferred to \Cref{sec:pcr-proof}.

\begin{theorem}[PCR risk]\label{thm:pcr-risk}
Let $\hat\wB_{\rho}^{\pcr}$ be given by \Cref{eq:pcr} with threshold $\rho \ge 0$. Under \Cref{assum:lbb}, there exist constants $c_0, \ldots ,c_4 >1$ that only depend on $c_x,c_y$ for which the following holds. Let
\begin{align*}
    k^*:= \min\left\{k: 1.1\rho + \frac{1}{c_2} \frac{\sum_{i>k}\lambda_i}{n} \ge \lambda_{k+1} \right\}.
\end{align*}
\begin{itemize}
    \item \textbf{Upper bound.} If additionally $k^*\le n/c_3$, then with probability at least $1-\delta-\exp(-n/c_0)$ over sampling of~$\XB$, it holds that
\begin{align*}
  \frac{1}{c_1} \Ebb[ \excessRisk(\hat\wB_\rho^{\pcr}) | \XB]\le \left(\alpha_{k^*} + \frac{\lam_{k^*+1}^2\log(1/\delta)}{n}\right) \|\wB^*\|^2_{\SigmaB_{0:k^*}^{-1}}  + \|\wB^*\|^2_{\SigmaB_{k^*:\infty}} + \sigma^2 \frac{\min\{D, \hat k\}}{n},
\end{align*}
where
\begin{align*}
D := k^* + \bigg(\rho + \frac{\sum_{i>k^*}\lambda_i}{n}\bigg)^{-2} \sum_{i>k^*}\lambda_i^2,\quad
\hat k := \#\{i: \hat\lambda_i > \rho\}.
\end{align*}
Also, with probability at least $1-\exp(-n / c_0)$ it holds that
\begin{align*}
    \hat k \le \#\left\{i: \lambda_i + \frac{\sum_{j>n}\lambda_j}{n} > \frac{\rho}{c_2} \right\}.
\end{align*}

\item \textbf{Lower bound.} 
Suppose additionally that $\xB\sim\Ncal(0,\SigmaB)$. Let
\begin{align*}
r^* := \min\left\{r:\lam_{r+1} + \frac{\sum_{i>r}\lam_i}{n} \le 0.9\rho\right\}, \quad \tilde\rho := \max\left\{1.1\rho, c_3\frac{\sum_{i>k^*}\lam_i}{n} \right\}
\end{align*}
and additionally assume $r^* \le n/c_3$. Then in expectation,
\begin{align*}
c_1 \Ebb \excessRisk(\hat\wB_\rho^{\pcr})
&\ge \alpha_{r^*} \|\wB^*\|_{\SigmaB_{0:k^*}^{-1}}^2 + \|\wB^*\|_{\SigmaB_{r^*:\infty}}^2 + \frac{\sigma^2}{n}
\#\{i:\lam_i>\tilde\rho\}.
\end{align*}
\end{itemize}
\end{theorem}

Notice that the structure of the PCR upper bound is similar to the bounds for GD (\Cref{thm:gd:bound}) or ridge \citep{tsigler2023benign}; indeed, we will soon see that all of these bounds follow from the same general principle. The gap between the bias upper and lower bounds is due to the mismatch between the critical indices~$k^*,r^*$. Ignoring the tail, the defining conditions of the two indices are essentially $\lam_{r^*}>0.9\rho$ and $\lam_{k^*}>1.1\rho$. Both constants can be made arbitrarily close to~$1$ by adjusting the other constants. Nonetheless, the existence of some gap is unavoidable in our approach, as demonstrated by the following example.

\begin{lemma}\label{lem:two-index-example}
Fix $\rho>0$, $\eps\in(0,0.05)$ and set $m_n:=\lfloor \eps^2n / c_3\rfloor$. Consider the noiseless Gaussian instances
\begin{align*}
\SigmaB_n^\pm
:=
\diag\left(2\rho,(1\pm\eps)\rho\IB_{m_n}\right),
\quad
\wB_n^*:=\eB_1, \quad \sigma_n :=0.
\end{align*}
Then $k^*=1$ and $r^*=m_n+1$ in \Cref{thm:pcr-risk} for both instances. Moreover the risk upper bound is $\Theta(\eps^2\rho)$, which is attained by~$\SigmaB_n^-$; while the risk lower bound is zero, which is attained by~$\SigmaB_n^+$.
\end{lemma}

We now have all the necessary ingredients to prove \Cref{thm:pcr-gd}.

\begin{proof}[Proof of \Cref{thm:pcr-gd}]
To prove the first claim, we directly compare the PCR upper bound and GD lower bound. Let $t\ge 1$ and let $k^*,\tilde\lambda,D$ be as in \Cref{thm:gd:finite-snr}. If $k^*>n/c_3$ or $D>n$, then $D\ge k^*$ and the variance lower bound (which holds regardless of~$k^*$, see \Cref{thm:gd:bound}) gives
\begin{align*}
\Ebb\excessRisk_\mu(\hat\wB_t^{\gd})\ge c\sigma^2.
\end{align*}
Since the zero estimator (obtained by taking $\rho\to\infty$) achieves risk $\|\wB^*\|_{\SigmaB}^2 \le b\sigma^2$, the claim follows. Hence we may suppose $k^*\le n/c_3$, $D\le n$. Set $\rho:=(1.1c_2\eta t)^{-1}$, so that the critical index in \Cref{thm:pcr-risk} is equal to~$k^*$ and the corresponding effective dimension is at most~$cD$. Then with probability at least $0.99$,
\begin{align*}
  \frac{1}{c_1} \Ebb[ \excessRisk(\hat\wB_\rho^{\pcr}) | \XB]\le \alpha_{k^*} \|\wB^*\|^2_{\SigmaB_{0:k^*}^{-1}}  + \|\wB^*\|^2_{\SigmaB_{k^*:\infty}} + \sigma^2 \frac{D}{n} \le c_1\Ebb\excessRisk(\hat\wB_t^{\gd}).
\end{align*}
The second claim follows from \Cref{thm:pcr-sd} and \Cref{lem:non-step-examples}.
\end{proof}

\paragraph{Comparison to \citet{hucker2023note}.} The closest existing bounds on PCR are the upper bounds in Theorems~1 and~2 of \citet{hucker2023note}. Our upper bound improves upon both results, as we now describe. Since their PCR algorithm is defined by the number of selected components~$k$, we equivalently choose $\rho = \lam_{k+1}$ and set the failure probability as $\delta = e^{-n/c_0}$.

Let $c,C>0$ denote constants. When $n^{-1} \sum_{i>k} \lam_i \le c\lam_{k+1}$, their Theorem~1 upper bounds the bias and variance by $\lam_{k+1}\|\wB^*\|^2$ and $\sigma^2 k/n$, respectively. Noting that $k^*\le k$ by definition, some algebra gives
\begin{align*}
\frac{\sum_{i>k^*}\lam_i}{n} \le \frac{\sum_{i>k}\lam_i + k\lam_{k^*+1}}{n} \le  c\lam_{k+1} + \frac{k}{n} \left(1.1\lam_{k+1} + \frac{1}{c_2} \frac{\sum_{i>k^*}\lam_i}{n}\right) \quad\implies\quad \frac{\sum_{i>k^*}\lam_i}{n}, \ \lam_{k^*+1} \lesssim \lam_{k+1}.
\end{align*}
Using that $n^{-1}\sum_{i>k^*}\lam_i < c_2\lam_{k^*}$, each head component in our upper bound has coefficient bounded as
\begin{align*}
\frac{\alpha_{k^*} + \lam_{k^*+1}^2}{\lam_{k^*}} \le \frac{1}{\lam_{k^*}} \bigg(\frac{\sum_{i>k^*}\lam_i}{n}\bigg)^2 + \frac{\lam_{k^*+1}}{\lam_{k^*}} \frac{\sum_{i>k^*}\lam_i}{n} + \lam_{k^*+1} \lesssim \lam_{k+1},
\end{align*}
and each tail component has coefficient at most $\lam_{k^*+1} \lesssim \lam_{k+1}$. Thus, our bounds for the bias and variance are both tighter in the setting of their Theorem~1.

When $n^{-1} \sum_{i>k} \lam_i \ge C\lam_{k+1}$, their Theorem~2 gives a bias upper bound of $(n^{-1} \sum_{i>j^*} \lam_i) \|\wB^*\|^2$, where $j^* := \min\{ j: n^{-1} \sum_{i>j} \lam_i \ge C\lam_{j+1}\}$. Noting that $j^*\le k$,
\begin{align*}
\frac{\sum_{i>k^*}\lam_i}{n} \le \frac{\sum_{i>j^*}\lam_i}{n} + \frac{j^*}{n} \left(1.1\lam_{k+1} + \frac{1}{c_2} \frac{\sum_{i>k^*}\lam_i}{n}\right) \quad\implies\quad \frac{\sum_{i>k^*}\lam_i}{n}, \ \lam_{k^*+1} \lesssim \frac{\sum_{i>j^*}\lam_i}{n},
\end{align*}
and the same argument shows that our bound for the bias is tighter. We remark that their variance bound is incomparable to ours in the setting of their Theorem~2.

\section{Risk Bounds for General Spectral Filters}\label{sec:general}

In this section, we present general techniques to bound the excess risk of arbitrary spectral filters. For a given filter~$g: \Rbb_{\ge 0}\to\Rbb$, the bias-variance decomposition is (omitting the dependence on~$\mu,\XB$)
\begin{align*}
\Ebb [\excessRisk_{\mu}(\hat \wB_g)|\XB]
&= \underbrace{\vphantom{
\Theta\left(\frac{\sigma^2}{n}\right)} \|(\IB-g(\hat\SigmaB))\wB^*\|_{\SigmaB}^2}_{=: \bias(\hat\wB_g)} + \underbrace{\Theta\left(\frac{\sigma^2}{n}\right) \tr\big(\SigmaB \hat\SigmaB^{-1}g(\hat\SigmaB)^2\big)}_{=:\variance(\hat\wB_g)}.
\end{align*}
The variance is typically straightforward to control. By monotonicity of trace, it holds that
\begin{align*}
g_1 \ge g_2\ge 0 \quad\implies\quad \variance(\hat\wB_{g_1}) \gtrsim \variance(\hat\wB_{g_2}).
\end{align*}
Hence one can obtain variance bounds by directly comparing against ridge filters with suitable $\ell_2$~penalty, for which tight variance bounds are known~\citep{tsigler2023benign}. The main difficulty lies in controlling the bias, due to the noncommutativity of $\SigmaB,\hat\SigmaB$.

In \Cref{sec:general-ub}, we present an upper bound for general spectral filters, building on the theory of Schur multipliers to control noncommutative matrix perturbations. This result is used to obtain our upper bounds for PCR and GD. Complementing this result, in \Cref{sec:general-lb}, we present a useful bias lower bound for monotone filters. Finally, as an application, we derive uniformly tight bounds for iterative Tikhonov regularization.

\subsection{General upper bound: Schur multipliers and matrix divided difference}\label{sec:general-ub}

Before stating our results, we give a brief background on Schur multipliers, norm, and factorization; more details are given in \Cref{sec:master-ub}. For simplicity, we focus on the finite-dimensional case, however, the subsequent discussions can be extended to infinite-dimensional Hilbert space in the usual manner.

\paragraph{Schur multiplier.}
Fix two symmetric matrices $\AB\in\Rbb^{n\times n}$ and $\BB\in\Rbb^{m\times m}$ with eigendecompositions 
\begin{align*}
    \AB := \UB\LambdaB \UB^\top := \sum_{i}\lambda_i \uB_i\uB_i^\top ,\quad 
    \BB := \VB\GammaB \VB^\top := \sum_{j}\gamma_j \vB_j\vB_j^\top.
\end{align*}
A kernel function $k: \Rbb\otimes \Rbb \to \Rbb$ induces a matrix operator from $\Rbb^{n\times m}$ to $\Rbb^{n\times m}$ by
\begin{align*}
    k(\AB,\BB) := \sum_{i, j} k(\lambda_i, \gamma_j) \big(\vB_j\vB_j^\top \big)\otimes \big(\uB_i\uB_i^\top \big).
\end{align*}
This operator is referred to as the \emph{Schur multiplier}, also known as a \emph{double operator integral}.

\paragraph{Divided difference for matrix functions.}
For a scalar function $f:\Rbb\to\Rbb$, its (first) divided difference is a kernel function $f^{[1]}:\Rbb\otimes \Rbb\to \Rbb$ such that
\begin{align*}
f(x) - f(y) = f^{[1]}(x,y) (x-y)\ \  \text{for all $x,y\in\Rbb$}.
\end{align*}
Note that the existence of $f^{[1]}$ does not require $f$ to be differentiable.
Clearly, $f^{[1]}$ is symmetric in its two arguments; its off-diagonal value is given by 
\begin{align*}
    f^{[1]}(x,y) := \frac{f(x) - f(y)}{x-y},\quad x\ne y,
\end{align*}
while its diagonal value can be chosen as is convenient. The Schur multiplier provides a canonical way to extend divided difference for scalar functions to symmetric matrix functions, as made precise by the following proposition.

\begin{proposition}[matrix divided difference]\label{lemma:divided-difference}
The Schur multiplier induced by the kernel function~$f^{[1]}$ of a scalar function~$f:\Rbb\to\Rbb$ gives a (symmetric) matrix divided difference, i.e.,
\begin{align*}
f(\AB) - f(\BB) = f^{[1]}(\AB, \BB) \circ (\AB - \BB)\quad \text{for all symmetric matrices $\AB,\BB$}.
\end{align*}
\end{proposition}

\paragraph{Schur norm and factorization.}
Let $I, J \subset \Rbb$ be two real index sets. 
A scalar kernel $k: I\otimes J \to\Rbb$ can be viewed as an infinite-dimensional generalized matrix $k(I, J):= [k(x,y)]_{x\in I, y\in J}$. For any $I'\subset I$ and $J'\subset J$ such that $|I'|, |J'|<\infty$, 
\begin{align*}
k(I', J'):= [k(x,y)]_{x\in I', y\in J'} \in \Rbb^{|I'|\times |J'|}
\end{align*}
is a finite submatrix of $k(I, J)$, and therefore induces a Schur multiplier via
\begin{align*}
\ZB\mapsto k(I', J') \odot \ZB,\quad \ZB\in\Rbb^{|I'|\times |J'|}.
\end{align*}
Recall that matrix operator norm is denoted by $\|\cdot\|$.
The Schur norm of a matrix, denoted by $\|\cdot\|_{\mathfrak{S}}$, is the induced operator norm of the corresponding Schur multiplier:
\begin{align*}
\| k(I', J') \|_{\mathfrak{S}} := 
\sup_{\|\ZB\|\le 1} \| k(I', J') \odot \ZB\|.
\end{align*}
The Schur norm of a kernel $k:I\otimes J\to \Rbb$, equivalently denoted as $k(I, J)$, is the supremum of the Schur norm of its finite submatrices:
\begin{align*}
\|k(I, J)\|_{\mathfrak{S}} := \sup_{\substack{I' \subset I, J'\subset J,\\ |I'|, |J'|<\infty}} \| k(I', J')\|_{\mathfrak{S}}.
\end{align*}
A \emph{Schur factorization} of a kernel $k(I, J)$ is a Hilbert space $\Hbb$ (with inner product denoted by $\la\cdot, \cdot\ra$ and norm denoted by $\|\cdot\|$)
and two feature maps $\fB : I \to \Hbb$ and $\gB: J \to \Hbb$ for which 
\begin{align*}
k(x,y) = \la \fB(x), \gB(y)\ra \quad \text{for all $x\in I $ and $y\in J$.}
\end{align*}
Specifically, any finite submatrix admits a decomposition as
\begin{align*}
k(I', J') = \FB \GB^\top 
\ \ \text{where}\ \  
\FB := \begin{bmatrix}
\vdots \\ 
    \fB(x)^\top\\
    \vdots 
\end{bmatrix}_{x\in I'} \in \Rbb^{|I'|}\otimes \Hbb,\ \  
\GB := \begin{bmatrix}
    \vdots \\
    \gB(y)^\top \\ 
    \vdots
\end{bmatrix}_{y\in J'} \in \Rbb^{|J'|}\otimes \Hbb.
\end{align*}

\begin{proposition}[\citet{aleksandrov2016operator}, Theorems~2.2.1--2.2.3]\label{prop:kernel-schur}
The Schur norm of a kernel $k: I\otimes J\to \Rbb$ is the smallest $\gamma_2$-factorization norm, i.e.,
\begin{align*}
    \|k(I, J)\|_{\mathfrak{S}} = \inf_{\Hbb} \inf_{\fB, \gB}\bigg\{ \sup_{x\in I}\|\fB(x)\| \cdot \sup_{y\in J}\|\gB(y)\| :\ k(x, y) = \la \fB (x), \gB (y)\ra \bigg\},
\end{align*}
where the infimum is taken over all possible Schur factorizations of $k(I, J)$.
\end{proposition}

Fix the ridge filter with regularization~$\lam$ as the reference filter  
\begin{align*}
    g^*(z) = \frac{z}{z+\lambda}.
\end{align*}
For a shrinkage filter $g$, define the \emph{relative kernel}~$r: \Rbb_{\ge 0}\otimes \Rbb_{\ge 0} \to \Rbb$ by
\begin{align*}
r(x,y) := \frac{g(x) - g(y)}{g^*(x) - g^*(y)}, \quad x\ne y.
\end{align*}
The relative kernel can be seen as the divided difference of~$g$ under the metric induced by~$g^*$. The diagonal entries of $r$ do not affect the proof but could affect its Schur norm; they can be chosen at discretion for convenience. We use the convention $r(0,0)=0$.


With this background in place, we are now in a position to state our main upper bound for general spectral filters. Recall that $\hat\SigmaB := \frac1n \XB^\top\XB$ is the empirical covariance and let $\hat\SigmaB_{\le k} := \frac{1}{n}\XB_{\le k}^\top \XB_{\le k}$ denote the empirical covariance of the first~$k$ coordinates.

\begin{theorem}
\label{thm:master-ub}
Under \Cref{assum:lbb}, there exist constants $c_0,\ldots,c_4$ that depend only on $c_x,c_y$ for which the following holds.
Let~$g:\Rbb_{\ge 0}\to\Rbb$ be a spectral filter and fix $\lambda>0$.
Let $k$ be an index satisfying
\begin{align*}
    k\le \frac{n}{c_3},
    \quad
    \lambda+\frac{\sum_{i>k}\lam_i}n\ge c_2\lam_{k+1}.
\end{align*}
Let~$r$ be the relative kernel associated with~$g$, and let $H,I\subseteq\R_{\ge 0}$ be fixed measurable sets such that the following Schur norms are finite:
\begin{align}\label{eq:schur-norm}
\|r(I, H)\|_{\mathfrak{S}},\, \|r(H, H)\|_{\mathfrak{S}},\, \|r(I, \{0\})\|_{\mathfrak{S}} <\infty.
\end{align}
Then with probability at least $1-\delta-\exp(-n/c_0)$, if
\begin{align}\label{eq:spec-cond}
\sigma(\hat\SigmaB)\subseteq I \quad\text{and}\quad \sigma(\SigmaB_{0:k}) \cup \sigma(\hat\SigmaB_{\le k}) \subseteq H\cup\{0\},
\end{align}
it holds that
\begin{align*}
\frac{1}{c_1} \bias(\hat\wB_g)
&\le 
    \left( \big(1+\|r(I, H)\|_{\mathfrak{S}}^2 + \|r(I, \{0\})\|_{\mathfrak{S}}^2\big) \left(\alpha_k + \frac{\lam_{k+1}^2 \log(1/\delta)}{n}\right) + \|r(H, H)\|_{\mathfrak{S}}^2  \frac{k+\log(1/\delta)}n\lam^2\right) \|\wB^*\|_{\SigmaB_{0:k}^{-1}}^2 \\
    &\qquad +\big\|(\IB-g(\SigmaB))\wB^*\big\|_{\SigmaB_{0:k}}^2 + \Big(1+\|r(I, \{0\})\|_{\mathfrak{S}}^2\Big) \|\wB^*\|_{\SigmaB_{k:\infty}}^2.
\end{align*}
Moreover, for the ridge estimator $\hat\wB_\lam^{\ridge}$ corresponding to~$g^*$,
\begin{align*}
\frac{1}{c_1} \variance(\hat\wB_g) \le \|r(I, \{0\})\|_{\mathfrak{S}}^2  \cdot\min\left\{\variance(\hat\wB_\lam^{\ridge}), \frac{\sigma^2}{n} \rank (g(\hat\SigmaB))\right\}.
\end{align*}
\end{theorem}

The Schur norms~\eqref{eq:schur-norm} are invariant under
the simultaneous rescaling
$g(z)\mapsto g(z/\tau)$,
$\lambda\mapsto\tau\lambda$, $I\mapsto\tau I$, and
$H\mapsto\tau H$, where $\tau>0$. Hence, it suffices to evaluate one representative filter for any class obtained by tuning a scale parameter, e.g., the stepsize~$\eta$ for GD or regularization strength~$\lam$ for (iterated) ridge. Our upper bounds for GD and PCR are obtained directly from the examples below, demonstrating the utility of our approach.

\begin{example}\label{lemma:schur:examples}
The norms~\eqref{eq:schur-norm} for ridge, GD and PCR are all bounded as follows.
\begin{itemize}
\item Ridge $g(z) = z/(z+\lambda)$: setting $I=H=\Rbb_{\ge 0}$, we have $\|r(I,I)\|_{\Sfrak} = 1$. In particular, this holds when specializing to OLS by taking $\lam\to 0$.
\item GD $g(z) = 1-(1-\eta z)^t$: setting $\lambda=(\eta t)^{-1}$ and $I=H=[0,1/\eta]$, we have $\|r(I,I)\|_{\Sfrak} \le 33$.
\item PCR $g(z) = \mathbf 1\{z>\rho\}$: setting $\lam=\rho$ and $I = \Rbb_{\ge 0}$, $H = [1.1\rho, \infty)$, we have
    \begin{align*}
        \|r(I,H )\|_{\Sfrak}\le 42, 
        \quad \|r(H, H)\|_{\Sfrak} =0,\quad \|r(I,\{0\})\|_{\Sfrak}\le 2 .
    \end{align*}
\end{itemize}
\end{example}

For sufficiently regular filters, the computation of the norms~\eqref{eq:schur-norm} can be further reduced to certain H\"{o}lder smoothness conditions. Writing~$\|\cdot\|_{C^{1,\beta}}$ for the usual~$C^{1,\beta}$ H\"{o}lder norm on~$\Rbb_{\ge 0}$, we obtain the following corollary.

\begin{corollary}
\label{thm:holder}
There exist constants $c_0,\ldots,c_4>0$, depending only on $c_x$, for which the following holds. Let $g:\R_{\ge 0}\to\R$ be a spectral filter and fix $\lam>0$ and $\beta\in (0,1]$. Suppose the following quantities are finite: 
\begin{align*}
C_\beta &:= \max_{j\in\{0,1,2\}}\frac{1}{\beta}\|z\mapsto z^j(1-g(\lam z))\|_{C^{1,\beta}},\quad
C_\psi := \sup_{z\ge 0}(z+\lam)|\psi(z)|.
\end{align*}
Let
\begin{align*}
k^* := \min\left\{k: \lam + \frac{\sum_{i>k}\lambda_i}{n} \ge c_2 \lambda_{k+1} \right\}.
\end{align*}
Then under \Cref{assum:lbb}, if additionally $k^*\le n/c_3$, then with probability at least $1-\delta-\exp(-n/c_0)$,
\begin{align*}
\frac{1}{c_1} \bias(\hat\wB_g)
    &\le \max\{C_\beta,C_\psi,1\}^2 \left(\alpha_{k^*} + \frac{k^* +\log(1/\delta)}n\lam^2\right) \|\wB^*\|_{\SigmaB_{0:k^*}^{-1}}^2
    + \max\{C_\psi,1\}^2 \|\wB^*\|_{\SigmaB_{k^*:\infty}}^2 \\&\qquad + \big\|(\IB-g(\SigmaB))\wB^*\big\|_{\SigmaB_{0:k^*}}^2.
\end{align*}
\end{corollary}

\paragraph{Proof roadmap.}
We provide a high-level overview of how the theory of Schur multipliers is applied. Denote the full and truncated Gram matrices by $\AB := \XB\XB^\top$ and $\AB_{\le k} := \XB_{\le k}\XB_{\le k}^\top$. The main challenge in the proof is to study perturbations in the form of a \emph{leave-tail-out} difference,
\begin{align}\label{eq:ff}
f(\AB) - f(\AB_{\le k}),
\end{align}
which is difficult as $\AB$ and $\AB_{\le k}$ generally do not commute. 
However, the introduced machinery makes this possible:
\begin{itemize}
\item The matrix divided difference (\Cref{lemma:divided-difference}) converts~\eqref{eq:ff} to studying the action of the Schur multiplier $k(\AB, \AB_{\le k})$ on the tail matrix $\AB_{>k} := \AB-\AB_{\le k}$ for $k= f^{[1]}$.
\item The Schur multiplier $k(\AB, \AB_{\le k})$ is equivariant to a standard Schur multiplier in the rotated space,
\begin{align*}
\ZB \mapsto k(I, J)\odot \ZB,
\end{align*}
where~$I$ and~$J$ consist of the spectra of $\AB$ and $\AB_{\le k}$, respectively.
So studying the action of $k(\AB, \AB_{\le k})$ on~$\AB_{>k}$ reduces to controlling the scalar kernel $k(I, J)$, with one complication that the rotation correlates with $\AB_{>k}$.
\item This final issue is addressed by Schur factorization, which allows us to fix a factorization of the whole kernel $k(\cdot,\cdot)$ independent of $\AB,\AB_{\le k}$ first, then apply the factorization to $k(I,J)$ by taking appropriate submatrices indexed by eigenvalues.
This operation decouples the correlation.
\end{itemize}
A similar argument is used to bound head perturbations involving $\SigmaB_{0:k}$ and $\hat\SigmaB_{\le k}$. This two-scale analysis allows us to cover even discontinuous filters such as PCR, and ultimately obtain sharp instance-wise bounds.

\subsection{General lower bound and applications}\label{sec:general-lb}

As discussed in \Cref{sec:pcr}, the PCR lower bound is not always tight due to the presence of two distinct critical indices. However, we can obtain a sharper `single index' lower bound for~$\hat\wB_g$ as long as the filter~$g$ does not saturate too quickly; the statement is presented below. Both results are a corollary of a more general lower bound for monotone filters, which we present in \Cref{thm:master-lb} in the appendix.

\begin{corollary}\label{cor:same-index}
Under \Cref{assum:lbb} with Gaussian design $\xB\sim\Ncal(0,\SigmaB)$, there exist constants $c_0,\ldots,c_3>0$ such that the following holds. Let $g:\R_{\ge 0}\to[0,1]$ be a monotone shrinkage filter, fix $\lam>0$ and let sample size $n\ge c_0$. Suppose that
\begin{align*}
C_g := \inf_{z\ge 0}\max\left\{1-g(1.1z), g(0.9z)-g(\lam)\right\}
\end{align*}
is positive. Let
\begin{align*}
k^* := \min\left\{k:\lam + \frac{\sum_{i>k}\lambda_i}{n} \ge c_2 \lambda_{k+1} \right\}.
\end{align*}
If additionally $k^* \le n/c_3$, then
\begin{align*}
c_1\Ebb\bias(\hat\wB_g)
&\ge C_g^2 \alpha_{k^*} \|\wB^*\|_{\SigmaB_{0:k^*}^{-1}}^2 + (1-g(\lam))^2 \|\wB^*\|_{\SigmaB_{k^*:\infty}}^2
+ \big\|(\IB-g(1.1\SigmaB))\wB^*\big\|_{\SigmaB}^2.
\end{align*}
If $g$ is concave, the last term may be replaced by $\big\|(\IB-g(\SigmaB))\wB^*\big\|_{\SigmaB}^2$.
\end{corollary}
Unlike \Cref{thm:master-ub} which only requires \Cref{assum:lbb}, the lower bounds require both Gaussian design and monotone shrinkage~$g$. Note that the tail coefficient $(1-g(\lam))^2$ is always at least as large as~$C_g^2$. Comparing \Cref{thm:master-ub} and \Cref{cor:same-index}, we see that whenever the Schur norms \eqref{eq:schur-norm} (or $C_\beta,C_\psi$ in \Cref{thm:holder}) are $\bigO(1)$ and $C_g = \Omega(1)$ for the same threshold~$\lam$ with $g(\lam)<1$, the upper and lower bias bounds match up to the following differences.
\begin{itemize}
\item The $(k^*/n)\lam^2$ factor in the head term; this is due to head concentration, and can typically be absorbed into variance under bounded SNR (see the proof of \Cref{thm:tikhonov}).
\item The~$1.1$ spectral shift in the residual term; this constant can be made arbitrarily close to~$1$, and disappears when~$g$ is concave.
\end{itemize}
For GD and ridge ($r=1$), both conditions are satisfied and we recover the known tight bounds (\Cref{thm:gd:finite-snr} and \citet{wu2026risk}, Proposition~2.1), albeit under the more restrictive Gaussian design. We can also obtain analogous tight bounds for gradient flow, as well as bounds for power-exponential filters such as Gaussian type filters \citep{calvetti1999iterative,calvetti2002lanczos}; we omit these for brevity.

\medskip
We conclude our study with an application.

\paragraph{Application: iterated Tikhonov.} The iterated Tikhonov (iterated ridge) estimator~$\hat\wB_{\lam,p}^{\itik}$ \citep{riley1955solving,king1979approximation} is given by the filter
\begin{align*}
g_{\lam,p}(z):=1-\left(\frac{\lam}{z+\lam}\right)^p,
\quad \lam>0,\quad p\ge 1.
\end{align*}
When~$p$ is an integer, it is equivalently obtained by the regularized iterative update
\begin{align*}
\hat\wB_{\lam,p}^{\itik} := \wB_p,\quad \text{where} \ \  \wB_{s} = 
\begin{dcases}
      \arg\min_{\wB} \frac1n \|\yB-\XB\wB\|^2 + \lam\|\wB - \wB_{s-1}\|^2 & s\ge 1, \\
        0 & s=0.
 \end{dcases}
\end{align*}

We prove matching upper and lower risk bounds for iterated Tikhonov, generalizing the known bounds for ridge \citep{wu2026risk}. In particular, the bounds are uniform with respect to both~$\lam$ and~$p$.

\begin{corollary}[Iterated Tikhonov risk under bounded SNR]\label{thm:tikhonov}
There exist constants $c_0,\ldots,c_3>1$ that depend only on
$c_x,c_y$ for which the following hold over $\Lbb_b$. Let $\hat\wB_{\lam,p}^{\itik}$ be given by the filter~$g_{\lam,p}$ with $\lam>0$, $p\ge 1$. Let
\begin{align*}
k^*
&:=\min\left\{
k:\frac{\lam}{p}+\frac{\sum_{i>k}\lam_i}{n}
\ge c_2\lam_{k+1}
\right\},\quad
\tilde\lam
:=\frac{\lam}{p}+\frac{\sum_{i>k^*}\lam_i}{n},
\quad
D:=k^*+\frac{1}{\tilde\lam^2} \sum_{i>k^*}\lam_i^2,
\end{align*}
and additionally assume $n\ge c_0$ and $k^*\le n/c_3$.
\begin{itemize}
\item \textbf{Upper bound.} With probability at least $1-\exp(-k^*/c_0)$ over sampling of~$\XB$, it holds that
\begin{align*}
\frac1{c_1}\Ebb[\excessRisk(\hat\wB_{\lam,p}^{\itik})|\XB]
&\le
\sum_{i\le k^*}\lam_i
\left(\frac{\lam}{\lam_i+\lam}\right)^{2p}\wB_i^{*2}
+\left(\frac{\sum_{i>k^*}\lam_i}{n}\right)^2
\|\wB^*\|_{\SigmaB_{0:k^*}^{-1}}^2 +\|\wB^*\|_{\SigmaB_{k^*:\infty}}^2
+(1+b)\sigma^2\frac{D}{n}.
\end{align*}

\item \textbf{Lower bound.} If additionally $\xB\sim\Ncal(0,\SigmaB)$, then in expectation,
\begin{align*}
c_1\Ebb\excessRisk(\hat\wB_{\lam,p}^{\itik})
&\ge
\sum_{i\le k^*}\lam_i
\left(\frac{\lam}{\lam_i+\lam}\right)^{2p}\wB_i^{*2}
+\left(\frac{\sum_{i>k^*}\lam_i}{n}\right)^2
\|\wB^*\|_{\SigmaB_{0:k^*}^{-1}}^2 +\|\wB^*\|_{\SigmaB_{k^*:\infty}}^2
+\sigma^2\frac{D}{n}.
\end{align*}
\end{itemize}
\end{corollary}

\section{Related Works}\label{sec:related}

Our analyses build upon the foundational works of \citet{bartlett2020benign,tsigler2023benign} on benign overfitting of OLS and ridge, as well as analyses of spectral regularization \citep[e.g.,][]{engl1996regularization,devito2005learning,bauer2007regularization}; we refer the reader to these works for more classical references. Below, we discuss works most relevant to our setting.

\paragraph{Dominance in linear regression.} For linear regression with fixed design, \citet{dhillon2013risk} showed that the risk of PCR with spectral threshold~$\lam$ is at most~$4$ times that of ridge regression with the same regularization strength~$\lam$, while it can be arbitrarily smaller. Thus, PCR strongly dominates ridge in our ratewise sense. Similarly, \citet{ali2019continuous} showed that the risk of early-stopped gradient flow at time $t=1/\lam$ is at most~$1.69$ times that of ridge, assuming an isotropic prior on the optimal parameter~$\wB^*$. In the random design setting, under mild regularity assumptions (\Cref{assum:lbb}), \citet{wu2026risk} proved that early-stopped GD with fixed stepsize at time $\eta t=1/\lam$ strongly dominates ridge. On the other hand, GD is incomparable to online SGD, while GD strongly dominates the latter when restricted to problems with fast and continuously decaying spectra, such as power law decay. This was shown by proving upper bounds for the excess risk of GD and comparing with the known essentially tight bounds for ridge \citep{tsigler2023benign} and SGD \citep{wu2022power}.

\paragraph{Risk bounds for GD.} The previous strongest known instance-wise bounds on the risk of GD were given in \citet{wu2026risk}. Their upper bound sharply characterizes the tail regularization, tail and variance terms, mirroring the risk analysis of ridge \citep{tsigler2023benign}. Compared to \Cref{thm:gd:bound}, only the optimization error term is missing (Theorem~3.1) or suboptimal (Theorem~4.3); see \Cref{sec:gd} for a detailed comparison. A similar upper bound was previously obtained in \citet{zou2022risk} by additionally assuming an isotropic prior on~$\wB^*$. Other works obtained bounds with coarser signal dependence \citep{raskutti2014early,kuzborskij2021role,xu2023towards}, or analyzed GD under various source and capacity conditions \citep{yao2007early,lin2017optimal,dicker2017kernel,blanchard2018optimal,lin2025improved}; see also Section 5 of \citet{wu2026risk}. Risk bounds for analytic spectral filters assuming power law decay, including GD and iterated ridge, were given in \citet{li2026generalization}.

\paragraph{Risk bounds for PCR.} The work most relevant to ours is \citet{hucker2023note}. \citet{hucker2023note} explain the implicit regularization effect of PCR by a comparison to an ``oracle'' PCR, which replaces the empirical principal components with the population versions, and obtain instance-wise upper bounds for the PCR risk in the same setting as \Cref{thm:pcr-risk}. However, their bounds are looser compared to ours, and they do not provide lower bounds. See Section~\ref{sec:pcr} for a detailed comparison. Risk bounds for a debiased variant of PCR were also recently given in \citet{zhao2026defloored}, however they require very strong assumptions on the signal and spectrum at the selected index~$k$, such as zero tail signal, bounded head condition number, and eigengap and small tail mass conditions; see the conditions of their Theorem~4.4. In contrast, our analysis applies to every instance ($\SigmaB,\wB^*$).

\citet{bing2021prediction} obtained coarser risk bounds in terms of $\|\SigmaB\|,\|\wB^*\|$ and studied data-adaptive threshold selection. \citet{teresa2022} proved upper and lower bounds for PCR in terms of the largest, smallest and threshold eigenvalues using general concentration arguments, and provided tighter bounds under spectral gap assumptions. \citet{lu2013regularization,dicker2017kernel,blanchard2018optimal} studied PCR under source and capacity conditions and established minimax rates, see also \citet{hall2007methodology,brunel2016nonasymptotic} for the functional perspective. Asymptotic analyses based on random matrix theory techniques were given in e.g., \citet{pmlr-v235-gedon24a,green2025high} for PCR and \citet{advani2020high} for gradient flow; see the works cited in their introduction for related approaches. Nonetheless, these results do not yield the fine-grained component-wise sample complexity bounds required to study dominance.

\paragraph{Comparison with \citet{li2026risk}.} In a concurrent work, \citet{li2026risk} studied the prediction risk of a class of spectral estimators and the corresponding Gaussian sequence estimators. They established conditions under which the two are asymptotically equivalent, that is,
\begin{align*}
\Ebb\excessRisk(\hat\wB_g) \quad \overset{?}{=\joinrel=} \quad (1+o_{\Pbb}(1)) \left(\big\|(\IB-g(\SigmaB))\wB^*\big\|_{\SigmaB}^2 + \frac{\sigma^2}{n} \sum_i g(\lam_i)^2\right).
\end{align*}
In particular, their Appendix~C utilizes Schur multipliers (double operator integrals) on Schatten classes to control matrix differences, which is closely related to our Schur multiplier approach developed in \Cref{sec:general-ub}. However, they apply this technique to directly control differences of $\SigmaB,\hat\SigmaB$, which additionally requires various restrictive conditions on the signal, spectrum, sample size and filter (their Assumption~3) to conclude even a constant-order version of the above equivalence. They also require the filter to satisfy a rescaled Lipschitz property, which includes GD and ridge but excludes PCR. 

In contrast, we apply this technique to $\AB,\AB_{\le k}$ to analyze the `leave-tail-out' perturbation in sample space, as well as to $\SigmaB_{\le k},\hat\SigmaB_{\le k}$ to separately control the head perturbation. In the process, we develop a decoupling argument via Schur factorization to handle the dependence between the perturbation and full Gram matrix. This allows us to study even discontinuous filters such as PCR by isolating the head, and to retain a precise characterization of finite-sample effects such as tail regularization for all problem instances.

We remark that contour integration can also be used to control noncommutative perturbations, which has been used for instance by \citet{li2026generalization} to study spectral regularization methods under power-law decay. However, this requires the filter to be analytic.

\section{Conclusion}

In this paper, we establish that suitably tuned PCR dominates all monotone filters, and therefore is instance-wise optimal for linear regression with Gaussian random design. In particular, GD is inadmissible, and PCR can achieve polynomially smaller risk than GD. We also prove tight risk bounds for GD, characterizing the effects of optimization error, implicit bias, and variance. These results demonstrate that instance-wise comparisons reveal substantial differences between methods that worst-case analyses do not capture. From an algorithmic perspective, this separation also highlights the variance benefit of discarding components with small eigenvalues, and suggests investigating whether similar control of weak directions can improve the statistical performance of algorithms for wider noisy estimation problems. Limitations of our work include the assumption of Gaussian design in the dominance results and general lower bounds; these may be removable with more careful analysis, as they are not required for the GD lower bound. Also, the variance bounds for PCR can likely be improved.

\section*{Acknowledgements}

JK, HF and JDL acknowledge support of NSF IIS 2107304, NSF CCF 2212262, NSF CAREER Award 2540142, NSF 2546544, NSF CCF 2019844 and ONR N00014-24-1-2639. This material is based upon work supported by the U.S. National Science Foundation under Cooperative Agreement No. 2433450.

\paragraph{Disclosure of AI usage.}
Parts of key technical ingredients were obtained from iterative conversations with GPT~5.5 and~5.6 Pro, including a proof of the PCR upper bound, an analytic version of the general upper bound, and a leave-one-out lemma used in the lower bound. 
AI was not used in the writing of the final draft. 
The authors rederived and verified
all proofs and take full responsibility for the content of this paper.

\bibliography{ref}

@article{aleksandrov2016operator,
  author  = {Aleksandrov, A. B. and Peller, V. V.},
  title   = {{Operator Lipschitz functions}},
  journal = {Russian Mathematical Surveys},
  volume  = {71},
  number  = {4},
  pages   = {605--702},
  year    = {2016}
}

@inproceedings{ali2019continuous,
  author    = {Ali, Alnur and Kolter, J. Zico and Tibshirani, Ryan J.},
  title     = {{A continuous-time view of early stopping for least squares regression}},
  booktitle = {Proceedings of the Twenty-Second International Conference on Artificial Intelligence and Statistics},
  editor    = {Chaudhuri, Kamalika and Sugiyama, Masashi},
  series    = {Proceedings of Machine Learning Research},
  volume    = {89},
  pages     = {1370--1378},
  publisher = {PMLR},
  year      = {2019}
}

@article{blackwell1947conditional,
  title={Conditional expectation and unbiased sequential estimation},
  author={Blackwell, David},
  journal={The Annals of Mathematical Statistics},
  pages={105--110},
  year={1947},
  publisher={JSTOR}
}

@article{rao1945information,
  title={Information and the accuracy attainable in the estimation of statistical parameters},
  author={Rao, C Radhakrishna and others},
  journal={Bull. Calcutta Math. Soc},
  volume={37},
  number={3},
  pages={81--91},
  year={1945}
}

@book{wainwright2019high,
  title={High-dimensional statistics: A non-asymptotic viewpoint},
  author={Wainwright, Martin J},
  volume={48},
  year={2019},
  publisher={Cambridge university press}
}

@article{bartlett2020benign,
  author  = {Bartlett, Peter L. and Long, Philip M. and Lugosi, G{\'a}bor and Tsigler, Alexander},
  title   = {{Benign overfitting in linear regression}},
  journal = {Proceedings of the National Academy of Sciences},
  volume  = {117},
  number  = {48},
  pages   = {30063--30070},
  year    = {2020}
}

@article{bauer2007regularization,
  author  = {Bauer, Frank and Pereverzev, Sergei and Rosasco, Lorenzo},
  title   = {{On regularization algorithms in learning theory}},
  journal = {Journal of Complexity},
  volume  = {23},
  number  = {1},
  pages   = {52--72},
  year    = {2007}
}

@book{berger1985statistical,
  author    = {Berger, James O.},
  title     = {{Statistical decision theory and Bayesian analysis}},
  edition   = {Second},
  series    = {Springer Series in Statistics},
  publisher = {Springer-Verlag},
  address   = {New York},
  year      = {1985}
}

@article{blanchard2018optimal,
  author  = {Blanchard, Gilles and M{\"u}cke, Nicole},
  title   = {{Optimal rates for regularization of statistical inverse learning problems}},
  journal = {Foundations of Computational Mathematics},
  volume  = {18},
  number  = {4},
  pages   = {971--1013},
  year    = {2018}
}

@article{calvetti1999iterative,
  author  = {Calvetti, Daniela and Reichel, Lothar and Zhang, Qin},
  title   = {{Iterative exponential filtering for large discrete ill-posed problems}},
  journal = {Numerische Mathematik},
  volume  = {83},
  number  = {4},
  pages   = {535--556},
  year    = {1999}
}

@article{calvetti2002lanczos,
  author  = {Calvetti, Daniela and Reichel, Lothar},
  title   = {{Lanczos-based exponential filtering for discrete ill-posed problems}},
  journal = {Numerical Algorithms},
  volume  = {29},
  number  = {1--3},
  pages   = {45--65},
  year    = {2002}
}

@incollection{carrasco2007linear,
  author    = {Carrasco, Marine and Florens, Jean-Pierre and Renault, Eric},
  title     = {{Linear inverse problems in structural econometrics: estimation based on spectral decomposition and regularization}},
  booktitle = {Handbook of Econometrics},
  editor    = {Heckman, James J. and Leamer, Edward E.},
  volume    = {6B},
  chapter   = {77},
  pages     = {5633--5751},
  publisher = {Elsevier},
  year      = {2007}
}

@article{devito2005learning,
  author  = {De Vito, Ernesto and Rosasco, Lorenzo and Caponnetto, Andrea and De Giovannini, Umberto and Odone, Francesca},
  title   = {{Learning from examples as an inverse problem}},
  journal = {Journal of Machine Learning Research},
  volume  = {6},
  number  = {30},
  pages   = {883--904},
  year    = {2005}
}

@techreport{devito2005spectral,
  author      = {De Vito, Ernesto and Rosasco, Lorenzo and Verri, Alessandro},
  title       = {{Spectral methods for regularization in learning theory}},
  institution = {Dipartimento di Informatica e Scienze dell'Informazione (DISI), Universit{\`a} di Genova},
  number      = {DISI-TR-05-18},
  address     = {Genova, Italy},
  year        = {2005}
}

@article{dhillon2013risk,
  author  = {Dhillon, Paramveer S. and Foster, Dean P. and Kakade, Sham M. and Ungar, Lyle H.},
  title   = {{A risk comparison of ordinary least squares vs ridge regression}},
  journal = {Journal of Machine Learning Research},
  volume  = {14},
  number  = {46},
  pages   = {1505--1511},
  year    = {2013}
}

@article{dicker2017kernel,
  author  = {Dicker, Lee H. and Foster, Dean P. and Hsu, Daniel},
  title   = {{Kernel ridge vs. principal component regression: minimax bounds and the qualification of regularization operators}},
  journal = {Electronic Journal of Statistics},
  volume  = {11},
  number  = {1},
  pages   = {1022--1047},
  year    = {2017}
}

@book{engl1996regularization,
  author    = {Engl, Heinz W. and Hanke, Martin and Neubauer, Andreas},
  title     = {{Regularization of inverse problems}},
  series    = {Mathematics and Its Applications},
  volume    = {375},
  publisher = {Kluwer Academic Publishers},
  address   = {Dordrecht},
  year      = {1996}
}

@article{Golubev_2010,
  author  = {Golubev, Yuri},
  title   = {{On universal oracle inequalities related to high-dimensional linear models}},
  journal = {The Annals of Statistics},
  volume  = {38},
  number  = {5},
  pages   = {2751--2780},
  year    = {2010}
}

@article{green2025high,
  author  = {Green, Alden and Romanov, Elad},
  title   = {{The high-dimensional asymptotics of principal component regression}},
  journal = {The Annals of Statistics},
  volume  = {53},
  number  = {4},
  pages   = {1697--1727},
  year    = {2025}
}

@article{hucker2023note,
  author  = {Hucker, Laura and Wahl, Martin},
  title   = {{A note on the prediction error of principal component regression in high dimensions}},
  journal = {Theory of Probability and Mathematical Statistics},
  volume  = {109},
  pages   = {37--53},
  year    = {2023}
}

@inproceedings{james1961estimation,
  author    = {James, W. and Stein, Charles},
  title     = {{Estimation with quadratic loss}},
  booktitle = {Proceedings of the Fourth Berkeley Symposium on Mathematical Statistics and Probability},
  editor    = {Neyman, Jerzy},
  volume    = {1},
  pages     = {361--379},
  publisher = {University of California Press},
  year      = {1961}
}

@article{king1979approximation,
  author  = {King, J. Thomas and Chillingworth, D.},
  title   = {{Approximation of generalized inverses by iterated regularization}},
  journal = {Numerical Functional Analysis and Optimization},
  volume  = {1},
  number  = {5},
  pages   = {499--513},
  year    = {1979}
}

@article{kneip1994ordered,
  author  = {Kneip, Alois},
  title   = {{Ordered linear smoothers}},
  journal = {The Annals of Statistics},
  volume  = {22},
  number  = {2},
  pages   = {835--866},
  year    = {1994}
}

@article{koltchinskii2017concentration,
  author  = {Koltchinskii, Vladimir and Lounici, Karim},
  title   = {{Concentration inequalities and moment bounds for sample covariance operators}},
  journal = {Bernoulli},
  volume  = {23},
  number  = {1},
  pages   = {110--133},
  year    = {2017}
}

@misc{li2026generalization,
  author       = {Li, Yicheng and Gan, Weiye and Shi, Zuoqiang and Lin, Qian},
  title        = {{Generalization error curves for analytic spectral algorithms under power-law decay}},
  howpublished = {arXiv preprint arXiv:2401.01599v4},
  year         = {2026}
}

@article{teresa2022,
  author  = {Huang, Ningyuan (Teresa) and Hogg, David W. and Villar, Soledad},
  title   = {{Dimensionality reduction, regularization, and generalization in overparameterized regressions}},
  journal = {SIAM Journal on Mathematics of Data Science},
  volume  = {4},
  number  = {1},
  pages   = {126--152},
  year    = {2022}
}

@book{triebel1983theory,
  author    = {Triebel, Hans},
  title     = {{Theory of function spaces}},
  series    = {Monographs in Mathematics},
  volume    = {78},
  publisher = {Birkh{\"a}user},
  address   = {Basel},
  year      = {1983}
}

@article{tsigler2023benign,
  author  = {Tsigler, Alexander and Bartlett, Peter L.},
  title   = {{Benign overfitting in ridge regression}},
  journal = {Journal of Machine Learning Research},
  volume  = {24},
  number  = {123},
  pages   = {1--76},
  year    = {2023}
}

@book{Vershynin2026high,
  author    = {Vershynin, Roman},
  title     = {{High-dimensional probability: an introduction with applications in data science}},
  edition   = {Second},
  series    = {Cambridge Series in Statistical and Probabilistic Mathematics},
  publisher = {Cambridge University Press},
  address   = {Cambridge},
  year      = {2026}
}

@book{wald1950statistical,
  author    = {Wald, Abraham},
  title     = {{Statistical decision functions}},
  publisher = {John Wiley \& Sons},
  address   = {New York},
  year      = {1950}
}

@inproceedings{wu2022last,
  author    = {Wu, Jingfeng and Zou, Difan and Braverman, Vladimir and Gu, Quanquan and Kakade, Sham M.},
  title     = {{Last iterate risk bounds of SGD with decaying stepsize for overparameterized linear regression}},
  booktitle = {Proceedings of the 39th International Conference on Machine Learning},
  year      = {2022}
}

@inproceedings{wu2022power,
  author    = {Wu, Jingfeng and Zou, Difan and Braverman, Vladimir
               and Gu, Quanquan and Kakade, Sham M.},
  title     = {{The power and limitation of pretraining-finetuning for linear regression under covariate shift}},
  booktitle = {Advances in Neural Information Processing Systems},
  year      = {2022}
}

@inproceedings{wu2026risk,
  title = {{Risk comparisons in linear regression: implicit regularization dominates explicit regularization}},
  author =       {Wu, Jingfeng and Bartlett, Peter L. and Kakade, Sham M. and Lee, Jason D. and Yu, Bin},
  booktitle = 	 {Proceedings of Thirty Ninth Conference on Learning Theory},
  year = 	 {2026}
}

@article{zou2023benign,
  author  = {Zou, Difan and Wu, Jingfeng and Braverman, Vladimir and Gu, Quanquan and Kakade, Sham M.},
  title   = {{Benign overfitting of constant-stepsize SGD for linear regression}},
  journal = {Journal of Machine Learning Research},
  volume  = {24},
  number  = {326},
  pages   = {1--58},
  year    = {2023}
}

@inproceedings{xu2023towards,
  author    = {Xu, Jing and Teng, Jiaye and Yuan, Yang and Yao, Andrew Chi-Chih},
  title     = {{Towards data-algorithm dependent generalization: a case study on overparameterized linear regression}},
  booktitle = {Advances in Neural Information Processing Systems},
  year      = {2023}
}

@inproceedings{kuzborskij2021role,
  author    = {Kuzborskij, Ilja and Szepesv{\'a}ri, Csaba and Rivasplata, Omar
               and Rannen-Triki, Amal and Pascanu, Razvan},
  title     = {{On the role of optimization in double descent: a least squares study}},
  booktitle = {Advances in Neural Information Processing Systems},
  year      = {2021}
}

@article{raskutti2014early,
  author  = {Raskutti, Garvesh and Wainwright, Martin J. and Yu, Bin},
  title   = {{Early stopping and non-parametric regression: an optimal data-dependent stopping rule}},
  journal = {Journal of Machine Learning Research},
  volume  = {15},
  number  = {11},
  pages   = {335--366},
  year    = {2014}
}

@inproceedings{zou2022risk,
  author    = {Zou, Difan and Wu, Jingfeng and Braverman, Vladimir
               and Gu, Quanquan and Kakade, Sham M.},
  title     = {{Risk bounds of multi-pass SGD for least squares in the interpolation regime}},
  booktitle = {Advances in Neural Information Processing Systems},
  year      = {2022}
}

@inproceedings{lin2025improved,
  author    = {Lin, Licong and Wu, Jingfeng and Bartlett, Peter L.},
  title     = {{Improved scaling laws in linear regression via data reuse}},
  booktitle = {Advances in Neural Information Processing Systems},
  year      = {2025}
}

@article{yao2007early,
  author  = {Yao, Yuan and Rosasco, Lorenzo and Caponnetto, Andrea},
  title   = {{On early stopping in gradient descent learning}},
  journal = {Constructive Approximation},
  volume  = {26},
  number  = {2},
  pages   = {289--315},
  year    = {2007}
}

@article{lin2017optimal,
  author  = {Lin, Junhong and Rosasco, Lorenzo},
  title   = {{Optimal rates for multi-pass stochastic gradient methods}},
  journal = {Journal of Machine Learning Research},
  volume  = {18},
  number  = {97},
  pages   = {1--47},
  year    = {2017}
}

@article{bing2021prediction,
  author  = {Bing, Xin and Bunea, Florentina and Strimas-Mackey, Seth
             and Wegkamp, Marten},
  title   = {{Prediction under latent factor regression: adaptive PCR, interpolating predictors and beyond}},
  journal = {Journal of Machine Learning Research},
  volume  = {22},
  number  = {177},
  pages   = {1--50},
  year    = {2021}
}

@article{zhao2026defloored,
  author       = {Zhao, Peng},
  title        = {{De-floored principal component regression: when rank selection alone is insufficient for prediction}},
  journal = {arXiv preprint arXiv:2607.16638},
  year         = {2026}
}

@book{lu2013regularization,
  author    = {Lu, Shuai and Pereverzev, Sergei V.},
  title     = {{Regularization theory for ill-posed problems: selected topics}},
  series    = {Inverse and Ill-Posed Problems Series},
  volume    = {58},
  publisher = {De Gruyter},
  address   = {Berlin and Boston},
  year      = {2013}
}

@article{brunel2016nonasymptotic,
  author  = {Brunel, {\'E}lodie and Mas, Andr{\'e} and Roche, Angelina},
  title   = {{Non-asymptotic adaptive prediction in functional linear models}},
  journal = {Journal of Multivariate Analysis},
  volume  = {143},
  pages   = {208--232},
  year    = {2016}
}

@article{hall2007methodology,
  author  = {Hall, Peter and Horowitz, Joel L.},
  title   = {{Methodology and convergence rates for functional linear regression}},
  journal = {The Annals of Statistics},
  volume  = {35},
  number  = {1},
  pages   = {70--91},
  year    = {2007}
}

@inproceedings{pmlr-v235-gedon24a,
  author    = {Gedon, Daniel and Ribeiro, Antonio H. and Sch{\"o}n, Thomas B.},
  title     = {{No double descent in principal component regression: a high-dimensional analysis}},
  booktitle = {Proceedings of the 41st International Conference on Machine Learning},
  year      = {2024}
}

@article{advani2020high,
  author  = {Advani, Madhu S. and Saxe, Andrew M. and Sompolinsky, Haim},
  title   = {{High-dimensional dynamics of generalization error in neural networks}},
  journal = {Neural Networks},
  volume  = {132},
  pages   = {428--446},
  year    = {2020}
}

@article{riley1955solving,
  author  = {Riley, James D.},
  title   = {{Solving systems of linear equations with a positive definite, symmetric, but possibly ill-conditioned matrix}},
  journal = {Mathematical Tables and Other Aids to Computation},
  volume  = {9},
  number  = {51},
  pages   = {96--101},
  year    = {1955}
}

@article{gerfo2008spectral,
author = {Gerfo, L. and Rosasco, Lorenzo and Odone, Francesca and De Vito, Ernesto and Verri, Alessandro},
year = {2008},
month = {07},
pages = {1873--1897},
title = {{Spectral algorithms for supervised learning}},
volume = {20},
journal = {Neural Computation}
}

@article{li2026risk,
      title={{Risk equivalence between RKHS regression and sequence models for Lipschitz spectral algorithms}}, 
      author={Yicheng Li and Yuqian Cheng and Zhuo Chen and Qian Lin},
      year={2026},
      journal = {arXiv preprint arXiv:2609.08817} 
}

\appendix

\clearpage
\section{Notation and basic concentration lemmas}\label{sec:proof:preliminaries}
Without loss of generality, assume $\SigmaB$ is diagonal throughout the appendices. For an index $k$, we define the following notation, following or adapted from the convention from \citet{wu2026risk,tsigler2023benign,zou2023benign}:
\begin{align*}
    \SigmaB = \begin{bmatrix}
        \SigmaB_{\le k} & \\
        & \SigmaB_{>k}
    \end{bmatrix},\quad 
    \wB^* = \begin{bmatrix}
        \wB^*_{\le k} \\
        \wB^*_{>k}
    \end{bmatrix},\quad 
    \XB = \begin{bmatrix}
        \XB_{\le k} & \XB_{>k}
    \end{bmatrix}, \\
    \AB := \XB\XB^\top,\quad 
    \AB_{\le k} := \XB_{\le k}\XB_{\le k}^\top, \quad \AB_{>k} := \XB_{>k}\XB_{>k}^\top.
\end{align*}
The empirical covariance and its head block are defined as 
\begin{align*}
    \hat \SigmaB := \frac{1}{n}\XB^\top \XB,\quad \hat\SigmaB_{\le k} := \frac{1}{n}\XB_{\le k}^\top \XB_{\le k}.
\end{align*}
Additionally, we use the notation
\begin{align}\label{eq:tilde-alpha}
\tilde\alpha_k &:= \bigg( \frac{\sum_{i>k}\lambda_i}{n} \bigg)^2 + \frac{\sum_{i>k}\lambda_i^2}{n} + \frac{\lam_{k+1}^2\log(1/\delta)}{n}.
\end{align}

\begin{lemma}[typical events]\label{lemma:basic:concentration}
Under \Cref{assum:item:x}, there are constants $c_0, c_1, c_3, c_4>1$ that depend on $c_x$ for which the following holds.
With probability at least $1-\exp(-n/c_0)$, we have
\begin{gather*}
 \frac{1}{2} \SigmaB_{\le k} \preceq \hat\SigmaB_{\le k} \preceq 2 \SigmaB_{\le k} \quad \text{for $k\le n/c_3$},\\
    \big\| \XB_{>k} \wB^*_{>k}\big\|^2 \le c_1 n \big\|\wB^*_{>k}\big\|^2_{\SigmaB_{>k}}, \\
\frac{2}{3}\sum_{i>k}\lambda_i  - c_1 n\lambda_{k+1} \le\lambda_{\min}(\AB_{>k}) \le \lambda_{\max}(\AB_{>k}) \le \frac{4}{3}\sum_{i>k}\lambda_i  + c_1 n\lambda_{k+1},\\
\|\hat\SigmaB\|\le   \frac{c_4}{2}\bigg(\lambda_1 + \frac{\tr(\SigmaB)}{n}\bigg).
\end{gather*}
In particular, the first claim can be tightened to: with probability at least $1-\delta$,
\begin{align*}
   \big\|\SigmaB_{\le k}^{-1/2}(\hat\SigmaB_{\le k}-\SigmaB_{\le k})\SigmaB^{-1/2}_{\le k} \big\| \le c_1 \left(\sqrt{\frac{k+\log(1/\delta)}{n}} + \frac{k+\log(1/\delta)}{n}\right).
\end{align*}
\end{lemma}

\begin{proof}[Proof of \Cref{lemma:basic:concentration}]
These are applications of standard matrix concentration bounds \citep[see, e.g.,][]{Vershynin2026high}: the first two claims appear in \citet[Lemma B.4]{wu2026risk} and the third claim appears in \citet[Lemma B.9]{bartlett2020benign}. 
For the last two claims, \citet{koltchinskii2017concentration} showed that with probability at least $1-\exp(-n/c_0)$, we have 
\begin{align*}
    \|\hat\SigmaB - \SigmaB\|\le c \|\SigmaB\|\bigg( \sqrt{\frac{\tr(\SigmaB)}{\|\SigmaB\|n}} + \frac{\tr(\SigmaB)}{\|\SigmaB\| n } + 1 \bigg) \le c' \bigg(\|\SigmaB\| + \frac{\tr(\SigmaB)}{ n }\bigg)
\end{align*}
for some constants $c_0, c, c'\ge 1$;
then the fourth claim follows, and the last claim follows by applying the above to $\SigmaB^{-1/2}\xB$.
\end{proof}

\begin{lemma}[tail regularization]\label{lemma:basic:tail-regularization}
Under \Cref{assum:item:x}, there are constants $c_0, c_2 >1$ that only depend on $c_x^2$ for which the following holds. 
For each $\lambda\ge 0$ and $k$ such that 
\begin{align*}
\lambda + \frac{ \sum_{i>k}\lambda_i}{n} \ge c_2 \lambda_{k+1},
\end{align*}
with probability at least $1-\exp(- n / c_0)$ it holds that
\begin{align*}
\frac{1}{2}\bigg(\lambda + \frac{\sum_{i>k}\lambda_i}{n}\bigg) \IB \preceq \frac{1}{n}\XB_{>k}\XB_{>k}^\top + \lam \IB \preceq 2 \bigg(\lambda + \frac{\sum_{i>k}\lambda_i}{n}\bigg) \IB.
\end{align*}
\end{lemma}
\begin{proof}[Proof of \Cref{lemma:basic:tail-regularization}]
This is \citet[Lemma 9]{bartlett2020benign} or \citet[Lemma 3]{tsigler2023benign}; from their proof, it is clear that, by choosing a large enough $c_2$, the constant factors in the upper and lower bounds can be made $2$ and $1/2$, respectively.
\end{proof}

\begin{lemma}[tail deviation]\label{lemma:basic:tail}
Under \Cref{assum:item:x}, there are constants $c_0, c_1 >1$ that only depend on $c_x$ for which the following holds.
Let $\uB$ be a fixed unit vector of an appropriate dimension.
\begin{itemize}
    \item with probability at least $1 - \delta$ for any $\delta>0$, it holds that
\begin{gather*}
    \big\|\SigmaB_{>k}^{1/2} \XB_{>k}^\top \uB\big\| \le c_1 \sqrt{ \lambda_{k+1}^2\log(1/\delta) + \sum_{i>k} \lambda_i^2}.
\end{gather*}
\item The $L^p$-norm of $\AB_{>k} \uB$ is bounded by
\begin{align}\label{eq:lp}
    \|\AB_{>k} \uB \|_{L^p} := \big(\Ebb\|\AB_{>k}\uB\|^p\big)^{1/p}  \le c_1 \Bigg( p \lambda_{k+1} + \sum_{i>k}\lambda_i + \sqrt{ \max\{n, p\} \bigg(p\lambda_{k+1}^2 + \sum_{i>k}\lambda_i^2 \bigg)}\Bigg),\quad p\ge 4.
\end{align}
\end{itemize}
\end{lemma}
\begin{proof}[Proof of \Cref{lemma:basic:tail}]
The first claim is a standard application of the Bernstein inequality \citep[see, e.g.,][proof of Lemma B.9]{bartlett2020benign}; the second claim is from \citet[Lemma E.2]{wu2026risk}.
\end{proof}

\section{GD Risk Bounds}\label{sec:proof:gd}

In this section, we prove the following upper and lower bounds on GD risk under \Cref{assum:lbb}. The bounds match except for the additional concentration coefficient in the tail regularization term, which is bounded as $\bigO((\eta t)^{-2})$. This term is included in the ``effective variance'' characterized in \citet{wu2026risk}, and can be absorbed in the variance when SNR is bounded (\Cref{thm:gd:finite-snr}).

\begin{theorem}[GD risk]\label{thm:gd:bound}
Under \Cref{assum:lbb}, there exist constants $c_0, \ldots,c_4 >1$ that only depend on $c_x,c_y$ for which the following holds. Let $\hat\wB^{\gd}_t$ be given by \Cref{eq:gd} with stepsize $\eta\le(c_4(\lam_1+\tr(\SigmaB)/n))^{-1}$ and steps $t\ge 1$. Let
\begin{align*}
    k^* := \min\left\{k:  \frac{1}{\eta t} + \frac{\sum_{i>k}\lambda_i}{n} \ge c_2 \lambda_{k+1} \right\},\quad \tilde\lambda:= \frac{1}{\eta t} + \frac{\sum_{i>k^*}\lambda_i}{n},\quad  D := k^* + \frac{1}{\tilde\lambda^2} \sum_{i>k^*}\lambda_i^2.
\end{align*}
\begin{itemize}
\item \textbf{Upper bound.}
If additionally $k^* \le n/c_3$, 
then with probability at least $1-\delta-\exp(-n/c_0)$ over sampling of~$\XB$, it holds that
\begin{align*}
&\frac{1}{c_1} \Ebb[ \excessRisk(\hat\wB^{\gd}_t) | \XB ]\le \big\|(\IB- \eta \SigmaB)^t \wB^*\big\|^2_{\SigmaB_{0:k^*}} + \left( \alpha_{k^*} + \frac{1}{(\eta t)^2}\frac{k^* + \log(1/\delta)}{n} \right) \|\wB^*\|^2_{\SigmaB_{0:k^*}^{-1}}  + \|\wB^*\|^2_{\SigmaB_{k^*:\infty}} + \sigma^2 \frac{D}{n}.
\end{align*}

\item \textbf{Lower bound.}
In expectation,
\begin{align*}
   c_1  \Ebb \excessRisk(\hat\wB^{\gd}_t) \ge \big\|(\IB- \eta \SigmaB)^t \wB^*\big\|^2_{\SigmaB_{0:k^*}} + \alpha_{k^*} \|\wB^*\|^2_{\SigmaB_{0:k^*}^{-1}}  + \|\wB^*\|^2_{\SigmaB_{k^*:\infty}} + \sigma^2 \min\bigg\{\frac{D }{n}\, , 1 \bigg\}.
\end{align*}
\end{itemize}
\end{theorem}

\begin{proof}[Proof of \Cref{thm:gd:finite-snr}]
For the upper bound, we use \Cref{thm:gd:bound} with $\delta= \exp(-k^*/c_0)$. Then 
\begin{align*}
\frac{\sum_{i>k^*}\lambda_i^2}{n} + \frac{1}{(\eta t)^2} \frac{k^*+\log(1/\delta)}{n} \le \frac{\sum_{i>k^*}\lambda_i^2}{n} + \frac{1}{(\eta t)^2} \frac{2k^*}{n} \le 3\tilde\lam^2\frac{D}{n}.
\end{align*}
Since $\tilde\lambda < c_2 \lambda_{k^*}$ by definition and $\|\wB^*\|^2_{\SigmaB} \le b\sigma^2$, we must have
\begin{align*}
\left(\frac{\sum_{i>k^*} \lam_i^2}{n} + \frac{1}{(\eta t)^2} \frac{k^*+\log(1/\delta)}{n} \right) \|\wB^*\|^2_{\SigmaB_{0:k^*}^{-1}} \le 3c_2^2 \|\wB^*\|^2_{\SigmaB_{0:k^*}} \frac{D}{n} \le 3c_2^2 b\sigma^2 \frac{D}{n}.
\end{align*}
Bringing this into \Cref{thm:gd:bound} and rescaling constants give the desired bound.

For the lower bound, it suffices to note that when $k^*\le n/c_3$,
\begin{align*}
D \le k^* + \frac{\lam_{k^*+1}}{\tilde\lam} \cdot \frac{1}{\tilde\lam} \sum_{i>k^*} \lam_i \le k^* + \frac{n}{c_2} = \bigO(n),
\end{align*}
thus the min thresholding in the variance can be removed.
\end{proof}

\begin{proof}[Proof of \Cref{thm:gd:bound}, upper bound]
The upper bound follows from an application of \Cref{thm:master-ub} to the GD example in \Cref{lemma:schur:examples}. By \Cref{lemma:basic:concentration}, it holds that $\|\hat\SigmaB_{\le k}\| \le \|\hat\SigmaB\| \le 1/\eta$, so \eqref{eq:spec-cond} holds with probability at least $1-\exp(-n/c_0)$. Also, $c_2\lambda_{k^*+1}\le \tilde\lam$ and we may assume $\log(1/\delta)\le n$, so that
\begin{align*}
\alpha_{k^*}+\frac{\lambda_{k^*+1}^2\log(1/\delta)}n
\le c\left( \alpha_{k^*} +\frac{1}{(\eta t)^2}\frac{\log(1/\delta)}n \right).
\end{align*}
Bringing this into \Cref{thm:master-ub} and combining with the ridge variance for $\lam=(\eta t)^{-1}$ \citep{tsigler2023benign} proves the upper bound.
\end{proof}

The rest of the section is devoted to proving the lower bound.

\subsection{GD risk decomposition}

Recall that the GD estimator defined in \Cref{eq:gd} is
\begin{align*}
    \hat\wB_t^{\gd} = \XB^\top \PsiB^{(t)} \yB,\quad \PsiB^{(t)} := \AB^{-1} \big(\IB - (\IB- \gamma \AB)^t\big),\quad \gamma:=\frac{\eta}{n}.
\end{align*}
Assume
\begin{equation}\label{eq:gd:stepsize}
    \eta \le \frac{1}{c_4( \lam_1 + \tr(\SigmaB)/n)}.
\end{equation}
Then under \Cref{assum:item:noise}, standard linear algebra \citep[see, e.g.][Appendix C]{wu2026risk} gives 
\begin{align*}
    \Ebb[\excessRisk(\hat\wB^{\gd}_t) | \XB]\ge \la \BB, \wB^{*\otimes 2}\ra + \frac{\sigma^2}{c_y} \la \CB, \SigmaB\ra,
\end{align*}
where
\begin{equation*}
\BB:= \big(\IB - \XB^\top\PsiB^{(t)} \XB\big) \SigmaB \big(\IB - \XB^\top \PsiB^{(t)} \XB\big),\quad 
\CB:= \XB^\top \big(\PsiB^{(t)}\big)^2 \XB.
\end{equation*}
Notice that \Cref{assum:item:symmetry} suggests that $\Ebb \BB_{ij} = 0$ for $i\ne j$ \citep[see, e.g.][proof of Lemma C.2]{wu2026risk}, which implies 
\begin{align*}
\Ebb \la \BB, \wB^{*\otimes 2}\ra = \Ebb \sum_i\BB_{ii}\wB^{*2}_i.
\end{align*}
Following the convention of \citet{wu2026risk,tsigler2023benign}, we write
\begin{align*}
\SigmaB\eB_i = \lambda_i \eB_i, \quad \XB \eB_i  = \lambda_i^{1/2}\zB_i,\quad \zB_i \in \Rbb^n,
\end{align*}
where entries of $\zB_i$'s are independent, mean zero, unit variance, and $c_x^2$-subgaussian by \Cref{assum:item:x}.
Then the diagonal entries of $\BB$ decompose as \citep{wu2026risk,tsigler2023benign},
\begin{align*}
    \BB_{ii} &= \eB_i^\top \BB\eB_i \\ 
    &= \eB_i^\top \big(\IB - \XB^\top \PsiB^{(t)} \XB\big)\SigmaB \big(\IB - \XB^\top \PsiB^{(t)} \XB\big) \eB_i \\ 
    &= \lambda_i - 2\lambda_i^{2}\zB_i^\top \PsiB^{(t)}\zB_i + \lambda_i \zB_i^\top \PsiB^{(t)} \XB\SigmaB\XB^\top \PsiB^{(t)} \zB_i \\ 
    &= \lambda_i - 2\lambda_i^{2}\zB_i^\top \PsiB^{(t)}\zB_i + \lambda_i \zB_i^\top \PsiB^{(t)} \bigg(\sum_{j} \lambda_j^2 \zB_j \zB_j^\top \bigg) \PsiB^{(t)} \zB_i \\ 
    &= \lambda_i - 2\lambda_i^{2}\zB_i^\top \PsiB^{(t)}\zB_i + \lambda_i^3 (\zB_i^\top \PsiB^{(t)} \zB_i)^2 + \lambda_i \zB_i^\top\PsiB^{(t)} \bigg(\sum_{j\neq i} \lambda_j^2 \zB_j \zB_j^\top \bigg) \PsiB^{(t)} \zB_i\\
    &= \lambda_i \big( 1- \lambda_i \zB_i^\top\PsiB^{(t)}\zB_i \big)^2  + \lambda_i \zB_i^\top\PsiB^{(t)} \bigg(\sum_{j\neq i} \lambda_j^2 \zB_j \zB_j^\top \bigg) \PsiB^{(t)} \zB_i.
\end{align*}
This leads to the following risk decomposition,
\begin{equation*}
    \Ebb \excessRisk(\hat\wB^{\gd}_t) 
    \ge  \Ebb \underbrace{\sum_i\lambda_i \wB_i^{*2} \big( 1- \lambda_i \zB_i^\top\PsiB^{(t)}\zB_i \big)^2}_{\effBias} + 
    \Ebb \underbrace{\sum_i \lambda_i\wB_i^{*2}\zB_i^\top\PsiB^{(t)} \bigg( \sum_{j\ne i} \lambda_j^2 \zB_j\zB_j^\top \bigg) \PsiB^{(t)} \zB_i}_{\effVar} + \frac{\sigma^2}{c_y} \Ebb \underbrace{\la  \CB, \SigmaB\ra}_{\variance}.
\end{equation*}
It remains to bound effective bias, effective variance, and variance errors. The last is done in \citet[Lemma C.1]{wu2026risk}, which we will restate for completeness. The main effort of this part is to provide new tight lower bounds on the effective bias and effective variance errors.

Throughout this part, let
\begin{equation}\label{eq:gd:lb:k-star}
    k^* := \min\left\{k: \frac{1}{\eta t}+ \frac{\sum_{i>k}\lambda_i}{n} \ge  c_2 \lambda_{k+1}\right\},\quad \tilde\lambda := \frac{1}{\eta t}+ \frac{\sum_{i>k^*}\lambda_i}{n} .
\end{equation}

\begin{lemma}[reduction to ridge]\label{lemma:gd:reduction-ridge}
The event of \Cref{lemma:basic:concentration} and stepsize condition \Cref{eq:gd:stepsize} imply that
\begin{align*}
\frac{1}{2} \bigg(\AB + \frac{n}{\eta t} \bigg)^{-1}\preceq    \PsiB^{(t)} \preceq 2 \bigg(\AB + \frac{n}{\eta t} \bigg)^{-1}.
\end{align*}
\end{lemma}
\begin{proof}[Proof of \Cref{lemma:gd:reduction-ridge}]
    This is \citet[Lemma B.2]{wu2026risk}. The lower inequality is understood on $\ran\AB$.
\end{proof}

\subsection{Variance error}
The following variance error bound is a restatement of \citep[Lemma C.1]{wu2026risk}, which is ultimately a reduction to the variance error of ridge \citep{bartlett2020benign,tsigler2023benign}.

\begin{lemma}[variance error]\label[lemma]{lemma:gd:lower-bound:variance}
Under \Cref{assum:item:x}, there exist $c_0, c_1, c_2>1$ that only depend on $c_x$ for which the following holds. 
For all $\eta$ satisfying \Cref{eq:gd:stepsize}, with probability at least $1-\exp(-n/c_0)$,
we have 
\begin{align*}
 \la \CB,\, \SigmaB\ra  \ge \frac{1}{c_1} \min\bigg\{\frac{k^* + {1}/{\tilde\lambda^2}\sum_{i>k^*}\lambda_i^2}{n},\, 1\bigg\},
\end{align*}
where $k^*$ and $\tilde\lambda$ are defined in \Cref{eq:gd:lb:k-star}.
\end{lemma}
\begin{proof}[Proof of \Cref{lemma:gd:lower-bound:variance}]
This is by \Cref{lemma:gd:reduction-ridge} and \citet[Theorem 2]{tsigler2023benign}.
\end{proof}

\subsection{Effective bias error}

\begin{lemma}[effective bias, tail bound]\label{lemma:gd:lb:effbias:tail}
Under \Cref{assum:item:x}, for some constant $c_0> 1$ that only depends on $c_x$, with probability at least $1-\exp(-n/c_0)$ it holds that 
\begin{align*}
  \effBias \ge\frac{1}{8} \| \wB^*_{>k^*} \|^2_{\SigmaB_{>k^*}}.
\end{align*}
\end{lemma}
\begin{proof}[Proof of \Cref{lemma:gd:lb:effbias:tail}]
Fix an index $j>k^*$. By the definition of $k^*$ and $\tilde\lambda$, we have 
\begin{align*}
    \tilde\lambda \ge c_2 \lambda_j.
\end{align*}
Without loss of generality, we assume $c_2 \ge 16$.
By Hoeffding's inequality, \Cref{lemma:gd:reduction-ridge,lemma:basic:tail-regularization}, 
with probability at least $1-\exp(-n/c_0)$ for some constant $c_0>1$, we have
\begin{align*}
    \|\zB_j\|^2\le 2 n,\quad \|\PsiB^{(t)}\|\le 2 \bigg\|\bigg(\AB_{>k^*}+\frac{n}{\eta t}\bigg)^{-1} \bigg\| \le \frac{4}{n\tilde\lambda}.
\end{align*}
Under this event, we have
\begin{align*}
    \lambda_j \zB_j^\top \PsiB^{(t)}\zB_j 
    \le \lambda_j \|\PsiB^{(t)}\| \cdot \|\zB_j\|^2  
    \le \lambda_j \frac{8}{\tilde\lambda}  
 \le \frac{8}{c_2} \le \frac{1}{2}
\end{align*}
since $c_2\ge 16$.
Thus, for each $j>k^*$, with probability at least $1-\exp(-n/c_0)$ it holds that 
\begin{align*}
\big(1- \lambda_j \zB_j^\top \PsiB^{(t)} \zB_j \big)^2 \ge \frac{1}{4}.
\end{align*}
By \citet[Lemma 15]{bartlett2020benign}, with probability at least $1-2\exp(-n/c_0)$ it holds that 
\begin{align*}
 \sum_{j>k^*} \lambda_j \wB_j^{*2} \big( 1- \lambda_j \zB_j^\top \PsiB^{(t)}\zB_j \big)^2 \ge \frac{1}{2} \sum_{j>k^*} \lambda_j\wB_j^{*2} \cdot\frac{1}{4} 
    = \frac{1}{8} \| \wB^*_{>k^*} \|^2_{\SigmaB_{>k^*}}.
\end{align*}
The left-hand side is a lower bound on $\effBias$, so we complete the proof by rescaling the constant.
\end{proof}

\begin{lemma}[effective bias, exponential bound]\label{lemma:gd:lb:effbias:exp} 
Under \Cref{assum:item:x}, there exist constants $c_0,c_1>1$ that only depend on $c_x$ for which the following hold: 
\begin{itemize}
    \item
    with probability at least $1-\exp(-n/c_0)$,
    \begin{align*}
        \effBias \ge \frac{1}{2} \big\|(\IB- 1.1 \eta\SigmaB)^t \wB^*\big\|^2_{\SigmaB};
    \end{align*}
    \item in expectation,
    \begin{align*}
    \Ebb\effBias\ge\frac{5}{18} \big\|(\IB-\eta\SigmaB)^t\wB^*\big\|_{\SigmaB}^2.
    \end{align*}
\end{itemize}
\end{lemma}

\begin{proof}[Proof of \Cref{lemma:gd:lb:effbias:exp}]
Notice that
\begin{align*}
    \IB - \XB^\top \PsiB^{(t)}\XB
    &= \IB - \XB^\top \AB^{-1}\big( \IB-(\IB-\gamma\AB)^t\big)\XB
    && \explain{definition of $\PsiB^{(t)}$} \\
    &= \IB - \XB^\top \AB^{-1}\XB
    \big( \IB-(\IB-\gamma\XB^\top\XB)^t\big)
    && \explain{$\AB = \XB\XB^\top$}\\
    &= (\IB-\gamma\XB^\top\XB)^t, \\
  \implies\quad   1- \lambda_i\zB_i^\top \PsiB^{(t)} \zB_i
    &= \eB_i^\top
    \big( \IB - \XB^\top \PsiB^{(t)}\XB\big)\eB_i
    && \explain{$\XB\eB_i=\lambda_i^{1/2}\zB_i$} \\
    &= \eB_i^\top (\IB - \gamma \XB^\top\XB)^t \eB_i.
\end{align*}
For each $i$, let the leave-one-out matrix be
\begin{align*}
\AB_{-i} := \sum_{j\ne i} \lambda_j \zB_j \zB_j^\top.
\end{align*}
The stepsize condition \Cref{eq:gd:stepsize} and  \Cref{lemma:basic:concentration} show that
\begin{align}
\text{with probability at least $1-\exp(-n/c_0)$:}\quad 
& \gamma \|\AB\|\le \frac{1}{2}, \label{eq:gd:lb:stable-regime} \\ 
\implies\quad &  \gamma\|\AB_{-i}\|,\ \gamma\lam_i\|\zB_i\|^2\le \frac12. && \explain{$\AB =\lam_i\zB_i\zB_i^\top + \AB_{-i}$}. \notag 
\end{align}
On this event, we have
\begin{align*}
    1- \lambda_i\zB_i^\top \PsiB^{(t)} \zB_i
    &= \eB_i^\top (\IB - \gamma \XB^\top\XB)^t \eB_i\\
    &\ge \big(1 - \gamma \eB_i^\top \XB^\top \XB \eB_i\big)^t
    && \explain{Jensen's inequality} \\
    &= \big(1 - \lambda_i\gamma\|\zB_i\|^2\big)^t
    && \explain{$\XB\eB_i=\lambda_i^{1/2}\zB_i$} \\
    &\ge 0.
\end{align*}

\emph{High probability lower bound.}~
Fix an index $i$.
By \Cref{assum:item:x} and Hoeffding's inequality, with probability at least $1-\exp(-n/c_0)$ for sufficiently large $c_0$, it holds that $\|\zB_i\|^2 \le 1.1 n$. On the intersection of this event with \Cref{eq:gd:lb:stable-regime}, we have 
\begin{align*}
    1- \lambda_i\zB_i^\top \PsiB^{(t)} \zB_i
    \ge \big(1 - \lambda_i\gamma \|\zB_i\|^2\big)^t 
    \ge (1-1.1\eta \lambda_i)^t \ge 0. && \explain{$\gamma = \eta / n,\ \eta \le 1/(2\lambda_1)$}
\end{align*}
Similarly to the proof of \Cref{lemma:gd:lb:effbias:tail}, by
\citet[Lemma 15]{bartlett2020benign}, we have with probability at least
$1-4\exp(-n/c_0)$,
\begin{align*}
    \effBias
    := \sum_i\lambda_i \wB_i^{*2}
    \big(1-\lambda_i\zB_i^\top\PsiB^{(t)}\zB_i\big)^2 
    \ge \frac{1}{2}\sum_i\lambda_i\wB_i^{*2}
    (1-1.1\eta\lambda_i)^{2t} 
    = \frac{1}{2}
    \big\|(\IB-1.1\eta\SigmaB)^t\wB^*\big\|_{\SigmaB}^2,
\end{align*}
which gives the lower bound after rescaling the constants.

\emph{Expectation lower bound.}~
Since $\AB_{-i}$ and $\zB_i$ are independent, it follows that
\begin{align*}
\Ebb \big(1- \lam_i\zB_i^\top \PsiB^{(t)} \zB_i\big)^2 &\ge \Pr\left(\gamma\|\AB_{-i}\| \le \frac12\right) \Ebb\left(\ind{\gamma\lam_i\|\zB_i\|^2 \le \frac12} (1-\gamma\lam_i\|\zB_i\|^2)^{2t}\right) \\
&\ge \frac12 \left(\Ebb (1-\gamma\lam_i\|\zB_i\|^2)_+^{2t} - 2^{-2t}\right) && \explain{by \Cref{eq:gd:lb:stable-regime}} \\
&\ge \frac12 \left((1-\Ebb\gamma \lam_i\|\zB_i\|^2)_+^{2t}- 2^{-2t}\right) && \explain{Jensen's inequality} \\ 
&= \frac12 \left((1-\eta \lam_i)^{2t}- 2^{-2t}\right).
\end{align*}
By increasing~$c_4$, we may assume
\begin{align*}
\eta\lam_i \le \frac{\lam_i}{c_4(\lam_1 + \tr(\SigmaB)/n)} \le \frac14 \quad\implies\quad & (1-\eta \lam_i)^{2t} \ge \left(\frac34\right)^{2t} > \frac94 \cdot 2^{-2t} \\
\quad\implies\quad & \Ebb \big(1- \lam_i\zB_i^\top \PsiB^{(t)} \zB_i\big)^2 \ge \frac{5}{18} (1-\eta \lam_i)^{2t}.
\end{align*}
We thus have
\begin{align*}
    \Ebb\effBias
     = \Ebb \sum_i\lambda_i \wB_i^{*2}
    \big(1-\lambda_i\zB_i^\top\PsiB^{(t)}\zB_i\big)^2 
    \ge \frac{5}{18} \sum_i\lambda_i (1-\eta\lam_i)^{2t} \wB_i^{*2} 
    = \frac{5}{18} \big\|(\IB-\eta\SigmaB)^t\wB^*\big\|_{\SigmaB}^2,
\end{align*}
which completes our proof.
\end{proof}

\begin{lemma}[effective bias, OLS-type lower bound]\label{lemma:gd:lb:effbias:ols}
Under \Cref{assum:item:x}, there exist constants $c_0, c_1, c_5 > 1$ that depend only on $c_x$ such that the following holds.
If $k^*$ defined in \Cref{eq:gd:lb:k-star} satisfies
    \begin{align*}
\frac{\sum_{i>k^*}\lambda_i}{n}  \ge c_5 \sqrt{\frac{\sum_{i>k^*}\lambda_i^2}{n}},
    \end{align*}
    then it holds with probability $1-\exp(-n/c_0)$ that
    \begin{align*}
        \effBias \ge \frac{1}{c_1} \bigg(\frac{\sum_{i>k^*}\lambda_i}{n} \bigg)^2\|\wB^*_{\le k^*}\|^2_{\SigmaB_{\le k^*}^{-1}}.
    \end{align*}
\end{lemma}
\begin{proof}[Proof of \Cref{lemma:gd:lb:effbias:ols}]
By the stepsize condition \Cref{eq:gd:stepsize} and \Cref{lemma:basic:concentration}, we have 
$0\preceq \PsiB^{(t)} \preceq \AB^{-1}$ with probability at least $1-\exp(-n/c_0)$.
Under this event, we have
\begin{align*}
    \effBias :=  \sum_i\lambda_i \wB_i^{*2} \big( 1- \lambda_i \zB_i^\top\PsiB^{(t)}\zB_i \big)^2 
    \ge  \sum_i\lambda_i \wB_i^{*2} \big( 1- \lambda_i \zB_i^\top\AB^{-1}\zB_i \big)^2,
\end{align*}
the right-hand side of which is exactly the effective bias error of OLS \citep[this idea has been used in, e.g.,][Lemma C.2]{wu2026risk}.
Furthermore, the assumption on $k^*$ and the standard small-ball concentration bound imply that for some constants $c_0,c_5> 1$, with probability at least $1-\exp(-n/c_0)$,
\begin{align*}
& \sum_{i>k^*}\lambda_i \zB_i\zB_i^\top \succeq \bigg( \sum_{i>k^*}\lambda_i - \frac{c_5}{2} \sqrt{n\sum_{i>k^*}\lambda_i^2} \bigg)\IB, && \explain{small-ball concentration} \\
\implies\quad &  \sum_{i>k^*}\lambda_i \zB_i\zB_i^\top \succeq \frac{1}{2} \sum_{i>k^*}\lambda_i \IB. && \explain{assumption on $k^*$}
\end{align*}
This enables the lower bound analysis for the effective bias of OLS by  \citet{tsigler2023benign} (see the proof of their Lemma 8 with $\lambda=0$), which states, with rescaling of constants, that with probability at least $1-\exp(-n/c_0)$,
\begin{align*}
\sum_i\lambda_i \wB_i^{*2} \big( 1- \lambda_i \zB_i^\top\AB^{-1}\zB_i \big)^2 \ge \frac{1}{c_1} \bigg(\frac{\sum_{i>k^*}\lambda_i}{n} \bigg)^2\|\wB^*_{\le k^*}\|^2_{\SigmaB_{\le k^*}^{-1}}.
    \end{align*}
We complete the proof by chaining the two inequalities with a union bound and rescaling the constant.
\end{proof}

\subsection{Effective variance error}
Obtaining a high-probability lower bound for effective variance error is hard for general spectra.
However, it is possible to obtain a lower bound on the expectation, as shown below.

\begin{lemma}[GD leave-one-out]\label{lemma:gd:lb:leave-one-out}
Under \Cref{assum:item:x}, there exists a constant $c_0 > 1$ that depends only on $c_x$ such that the following holds.
For a fixed index~$j$ with $\lam_j>0$, let 
\begin{align*}
\AB_{-j} := \sum_{i\ne j} \lambda_i \zB_i \zB_i^\top,\quad \PsiB_{-j}^{(t)}:= \gamma\sum_{s=0}^{t-1}(\IB - \gamma\AB_{-j})^{s}. 
\end{align*}
If $j$ satisfies
\begin{align*}
    \gamma\|\AB_{-j}\|\le \frac12, \quad 16c_x^4 \lambda_j \tr \big(\PsiB_{-j}^{(t)} \big) \le \frac{1}{2},
\end{align*}
then for any $n\ge c_0$, in expectation over $\zB_j$, we have 
\begin{align*}
\Ebb_{\zB_j}\PsiB^{(t)} \zB_j \zB_j^\top \PsiB^{(t)}  \succeq \frac{1}{16}  \big(\PsiB^{(t)}_{-j}\big)^2.
\end{align*}
\end{lemma}
\begin{proof}[Proof of \Cref{lemma:gd:lb:leave-one-out}]
In this proof, we only consider the randomness of $\zB_j$; so
$\Ebb$ refers to expectation over $\zB_j$.
It suffices to show 
\begin{align*}\text{for any fixed vector $\zB$},\quad 
\Ebb \big( \zB_j^\top \PsiB^{(t)} \zB \big)^2 \ge \frac{1}{c_1} \big\| \PsiB^{(t)}_{-j} \zB\big\|^2.
\end{align*}
Let
\begin{align*}
\mathcal F:=\{\gamma\lambda_j\|\zB_j\|^2\le1/2\}, \quad \mB:= \Ebb \mathbf 1_{\mathcal F} \zB_j \zB_j^\top \PsiB^{(t)} \zB,
\end{align*}
then $\Pr(\mathcal F) \ge 1-\exp(-n/c_0)$ by \Cref{assum:item:x} and the stepsize condition \Cref{eq:gd:stepsize}, and under $\Fcal$, we have 
\begin{align*}
    \gamma\|\AB\|\le 1 \quad \text{since}\ \ \AB = \AB_{-j} + \lambda_j \zB_j\zB_j^\top \ \ \text{and}\ \ \gamma\|\AB_{-j}\|\le \frac{1}{2}. 
\end{align*}
Notice that
\begin{align*}
\|\mB\|^2
&= \Ebb \mathbf 1_{\mathcal F}\mB^\top \zB_j \zB_j^\top \PsiB^{(t)} \zB \\ 
&\le \sqrt{\Ebb \mathbf 1_{\mathcal F}(\mB^\top \zB_j )^2 \Ebb \big(\zB_j^\top \PsiB^{(t)} \zB\big)^2} && \explain{Cauchy--Schwarz inequality} \\
&\le \sqrt{\|\mB\|^2 \Ebb \big(\zB_j^\top \PsiB^{(t)} \zB\big)^2}, && \explain{$\oneB_{\Fcal}\le 1$ and $\Ebb \zB_j \zB_j^\top = \IB$} 
\end{align*}
which implies that 
\begin{align*}
    \Ebb \big(\zB_j^\top \PsiB^{(t)} \zB\big)^2 \ge \|\mB\|^2 = \big\| \Ebb \mathbf 1_{\mathcal F}\zB_j \zB_j^\top \PsiB^{(t)} \zB\big\|^2.
\end{align*}
It remains to bound the right-hand side.
Let 
\begin{align*}\QB_{-j}^{(t)}:= \PsiB^{(t)}_{-j} - \PsiB^{(t)}
\quad\implies\quad 
\big\| \Ebb \mathbf 1_{\mathcal F}\zB_j \zB_j^\top \PsiB^{(t)} \zB\big\|
    \ge \big\| \Ebb \mathbf 1_{\mathcal F}\zB_j \zB_j^\top \PsiB^{(t)}_{-j} \zB\big\| - \big\| \Ebb \mathbf 1_{\mathcal F}\zB_j \zB_j^\top \QB_{-j}^{(t)} \zB\big\| .
\end{align*}
For the first term on the right-hand side, we have  
\begin{align*}
    \big\| \Ebb \mathbf 1_{\mathcal F}\zB_j \zB_j^\top \PsiB^{(t)}_{-j} \zB\big\|
    &\ge \big\| \Ebb \zB_j \zB_j^\top \PsiB^{(t)}_{-j} \zB\big\| - \big\| \Ebb \mathbf 1_{\mathcal F^c}\zB_j \zB_j^\top \PsiB^{(t)}_{-j} \zB\big\| \\
    &\ge \left(1 - \big\| \Ebb \mathbf 1_{\mathcal F^c}\zB_j \zB_j^\top \big\|  \right) \big\|  \PsiB^{(t)}_{-j} \zB\big\| && \explain{$\Ebb \zB_j \zB_j^\top = \IB$}\\
    &\ge  \bigg(1 -\sqrt{\Ebb \mathbf 1_{\mathcal F^c} \Ebb \|\zB_j\|^4 }  \bigg) \big\|  \PsiB^{(t)}_{-j} \zB\big\| && \explain{Cauchy--Schwarz inequality} \\
    &\ge  \left(1 -\exp(-n/c_0) 4c_x^2\Ebb\|\zB_j\|^2  \right) \big\|  \PsiB^{(t)}_{-j} \zB\big\| && \explain{subgaussian hypercontractivity} \\
    &=  \left(1 -\exp(-n/c_0) 4c_x^2n  \right) \big\|  \PsiB^{(t)}_{-j} \zB\big\| && \explain{$\Ebb \zB_j \zB_j^\top = \IB$} \\
    &\ge \frac{3}{4} \big\|  \PsiB^{(t)}_{-j} \zB\big\|. && \explain{taking $n$ sufficiently large}
\end{align*}
In sum, we have shown that 
\begin{align*}
\sqrt{\Ebb \big(\zB_j^\top \PsiB^{(t)} \zB\big)^2} &\ge 
\big\| \Ebb \mathbf 1_{\mathcal F}\zB_j \zB_j^\top \PsiB^{(t)} \zB\big\| \\
    &\ge \big\| \Ebb \mathbf 1_{\mathcal F}\zB_j \zB_j^\top \PsiB^{(t)}_{-j} \zB\big\| - \big\| \Ebb \mathbf 1_{\mathcal F}\zB_j \zB_j^\top \QB_{-j}^{(t)} \zB\big\| \\  
    &\ge \frac34 \big\| \PsiB^{(t)}_{-j} \zB\big\| - \big\| \Ebb \mathbf 1_{\mathcal F}\zB_j \zB_j^\top \QB_{-j}^{(t)} \zB\big\|,
\end{align*}
so to complete our proof it suffices to show that 
\begin{align*}
\big\| \Ebb \oneB_{\Fcal}\zB_j \zB_j^\top \QB_{-j}^{(t)} \zB\big\| \le \frac{1}{2}\big\| \PsiB^{(t)}_{-j} \zB\big\|.
\end{align*}
Since $\lambda_j>0$, we have $\zB_j\in\ran\AB$, and telescoping gives
\begin{align*}
\zB_j^\top\PsiB^{(t)} = \gamma \sum_{s=0}^{t-1} \zB_j^\top (\IB-\gamma\AB)^s \quad\implies\quad \zB_j^\top\QB_{-j}^{(t)}
&=\gamma\sum_{r=0}^{t-2}
\zB_j^\top(\IB-\gamma\AB)^r
\lambda_j\zB_j\zB_j^\top\PsiB_{-j}^{(t-1-r)}.
\end{align*}
Under $\Fcal$, we have $\gamma\|\AB\|\le 1$, 
so we have 
\begin{align*}
   \big\| \Ebb \mathbf 1_{\mathcal F}\zB_j \zB_j^\top \QB_{-j}^{(t)} \zB \big\|
    &\le \lambda_j  \gamma \sum_{r=0}^{t-2} \big\| \Ebb \mathbf 1_{\mathcal F}\zB_j \zB_j^\top  (\IB-\gamma\AB)^r \zB_j \zB_j^\top \PsiB^{(t-1-r)}_{-j} \zB\big\| \\
    &\le \lambda_j  \gamma \sum_{r=0}^{t-2} \big\| \Ebb \mathbf 1_{\mathcal F}\zB_j \zB_j^\top  (\IB-\gamma\AB)^r\zB_j \zB_j^\top\big\| \cdot \big\| \PsiB^{(t-1-r)}_{-j} \zB\big\| && \explain{$\PsiB^{(t-1-r)}_{-j} \zB$ is independent of $\zB_j$} \\
     &\le \lambda_j  \gamma \sum_{r=0}^{t-2} \big\| \Ebb \mathbf 1_{\mathcal F}\zB_j \zB_j^\top  (\IB-\gamma\AB)^r  \zB_j \zB_j^\top\big\| \cdot \big\| \PsiB^{(t)}_{-j} \zB\big\|.  && \explain{$\PsiB^{(t-1) 2}_{-j}\preceq \PsiB^{(t) 2}_{-j}$}
\end{align*}
To proceed, notice that under $\Fcal$,
\begin{align*}
&(\IB-\gamma\AB_{-j})^t -  (\IB-\gamma\AB)^t = \gamma\sum_{s=0}^{t-1}(\IB-\gamma\AB)^{s}\lambda_j \zB_j\zB_j^\top (\IB-\gamma\AB_{-j})^{t-1-s} \\
&\implies\quad  
\zB_j^\top  (\IB-\gamma\AB)^r  \zB_j\le \zB_j^\top  (\IB-\gamma\AB_{-j})^r  \zB_j.
\end{align*}
So we have 
\begin{align*}
\Ebb \mathbf 1_{\mathcal F}\zB_j \zB_j^\top  (\IB-\gamma\AB)^r  \zB_j \zB_j^\top
    &\preceq \Ebb \mathbf 1_{\mathcal F} \zB_j \zB_j^\top  (\IB-\gamma\AB_{-j})^r \zB_j \zB_j^\top \\
       &\preceq \Ebb \zB_j \zB_j^\top  (\IB-\gamma\AB_{-j})^r \zB_j \zB_j^\top && \explain{$\mathbf 1_{\mathcal F}\le 1$} \\
    &\preceq 16c_x^4 \tr\left( (\IB-\gamma\AB_{-j})^r \right)\IB. && \explain{subgaussian hypercontractivity}
\end{align*}
Bringing this back, we have 
\begin{align*}
    \big\| \Ebb \mathbf 1_{\mathcal F} \zB_j \zB_j^\top \QB_{-j}^{(t)} \zB \big\|
    &\le  16c_x^4 \lambda_j  \gamma \sum_{r=0}^{t-2} \tr\left( (\IB-\gamma\AB_{-j})^r \right) \cdot \big\| \PsiB^{(t)}_{-j} \zB\big\| \\
    &=  16c_x^4 \lambda_j  \tr \big(\PsiB_{-j}^{(t-1)}\big) \cdot \big\| \PsiB^{(t)}_{-j} \zB\big\| && \explain{$\PsiB_{-j}^{(t-1)}:= \gamma \sum_{r=0}^{t-2}  (\IB-\gamma\AB_{-j})^r $}\\
     &\le 16c_x^4 \lambda_j  \tr \big(\PsiB_{-j}^{(t)}\big) \cdot \big\| \PsiB^{(t)}_{-j} \zB\big\| && \explain{$\PsiB_{-j}^{(t-1)}\preceq \PsiB_{-j}^{(t)}$}.
\end{align*}
The choice of $j$ guarantees that 
\begin{align*}
 16c_x^4 \lambda_j  \tr\big(\PsiB_{-j}^{(t)}\big) \le \frac{1}{2}
 \quad \implies\quad 
 \big\| \Ebb \mathbf 1_{\mathcal F}\zB_j \zB_j^\top \QB_{-j}^{(t)} \zB \big\|\le \frac{1}{2}\big\| \PsiB^{(t)}_{-j} \zB\big\|,
\end{align*}
which completes our proof.
\end{proof}

\begin{lemma}[effective variance]\label{lemma:gd:lb:effvar}
Under \Cref{assum:item:x}, there exist constants $c_0,c_1 > 1$ that depend only on $c_x$ such that the following holds.
For any $n\ge c_0$, it holds in expectation that 
    \begin{align*}
        \Ebb\effBias+\Ebb\effVar \ge \frac{1}{c_1} \frac{\sum_{i>k^*}\lambda_i^2}{n} \|\wB^*_{\le k^*}\|^2_{\SigmaB_{\le k^*}^{-1}}.
    \end{align*}
\end{lemma}
\begin{proof}[Proof of \Cref{lemma:gd:lb:effvar}]
For each $j>k^*$, the definition of $k^*$ implies that 
\begin{align*}
    \lambda_j \le \frac{1}{c_2} \bigg( \frac{1}{\eta t} + \frac{\sum_{i>k^*}\lambda_i}{n} \bigg).
\end{align*}
Additionally, \Cref{lemma:basic:concentration,lemma:gd:reduction-ridge,lemma:basic:tail-regularization} implies that with probability at least $1-\exp(-n/c_0)$,
\begin{align*}
\gamma\|\AB_{-j}\|\le \frac12, \quad \PsiB^{(t)}_{-j} \preceq 2\bigg(\AB_{-j} + \frac{n}{\eta t}\bigg)^{-1}\preceq \frac{4}{n/(\eta t) + \sum_{i>k^*}\lambda_i} \IB \preceq \frac{4}{c_2 n \lambda_j}\IB.
\end{align*}
Then for a sufficiently large $c_2$, 
\begin{align*}
   16c_x^4  \lambda_j  \tr\big(\PsiB^{(t)}_{-j} \big) \le 16c_x^4\lambda_j \cdot \frac{n}{c_2 n \lambda_j} \le \frac{1}{2}.
\end{align*}
Under this event of $(\zB_i)_{i\ne j}$ which we denote as $\mathcal G_{-j}$, \Cref{lemma:gd:lb:leave-one-out} applies, hence 
\begin{align*}
    \Ebb \PsiB^{(t)} \zB_j \zB_j^\top \PsiB^{(t)} \succeq \frac{1}{c_1} \Ebb \mathbf 1_{\mathcal G_{-j}} \big(\PsiB^{(t)}_{-j}\big)^2,\quad j>k^*. 
\end{align*}
Then we have
\begin{align*}
\Ebb \effVar &:= \Ebb \sum_i \lambda_i\wB_i^{*2} \zB_i^\top\PsiB^{(t)} \bigg( \sum_{j\ne i} \lambda_j^2 \zB_j\zB_j^\top \bigg) \PsiB^{(t)} \zB_i  \\
&\ge \Ebb \sum_{i\le k^*} \lambda_i\wB_i^{*2} \zB_i^\top\PsiB^{(t)} \bigg( \sum_{j> k^*} \lambda_j^2 \zB_j\zB_j^\top \bigg) \PsiB^{(t)} \zB_i  \\
    &\ge \frac{1}{c_1} \sum_{i\le k^*} \lambda_i\wB_i^{*2}  \sum_{j>k^*} \lambda_j^2 \Ebb \mathbf 1_{\mathcal G_{-j}} \zB_i^\top \big(\PsiB^{(t)}_{-j}\big)^2 \zB_i.
\end{align*}
The term $\zB_i^\top \big(\PsiB^{(t)}_{-j}\big)^2 \zB_i$ appears in the lower bound of GD variance, which was analyzed by ridge reduction \citep{tsigler2023benign,wu2026risk}.
We use a similar argument.
First, we have
\begin{align*}
    \Ebb \mathbf 1_{\mathcal G_{-j}} \zB_i^\top \big(\PsiB^{(t)}_{-j}\big)^2 \zB_i
    &\ge \Ebb \mathbf 1_{\mathcal G_{-j} } \frac{ \big( \zB_i^\top \PsiB^{(t)}_{-j} \zB_i  \big)^2  }{\|\zB_i\|^2} && \explain{Cauchy--Schwarz inequality} \\
    &\ge \frac{1}{2n} \Ebb \mathbf 1_{\mathcal G_{-j} \cap \{\|\zB_i\|^2\le 2n \} }  \big( \zB_i^\top \PsiB^{(t)}_{-j} \zB_i \big)^2 \\
    &\ge \frac{1}{8n } \Ebb \mathbf 1_{\mathcal G_{-j}\cap\{\|\zB_i\|^2 \le 2n \}}\Bigg(  \zB_i^\top \bigg(\AB_{-j} + \frac{n}{\eta t} \IB \bigg)^{-1} \zB_i \Bigg)^2.  && \explain{\Cref{lemma:gd:reduction-ridge}}
\end{align*}
For $i\le k^*$, write $b_i:=\Ebb(1-\lambda_i\zB_i^\top\PsiB^{(t)}\zB_i)^2$.
On the event $\gamma\|\AB\|\le1/2$, \Cref{lemma:gd:reduction-ridge} gives
\begin{align*}
\zB_i^\top\bigg(\AB_{-j}+\frac{n}{\eta t}\IB\bigg)^{-1}\zB_i
&\ge \zB_i^\top\bigg(\AB+\frac{n}{\eta t}\IB\bigg)^{-1}\zB_i
\ge \frac12\zB_i^\top\PsiB^{(t)}\zB_i.
\end{align*}
Using $x^2\ge 1/2-(1-x)^2$, we obtain, for sufficiently large $n$,
\begin{align*}
\Ebb\mathbf1_{\mathcal G_{-j}}\zB_i^\top\big(\PsiB_{-j}^{(t)}\big)^2\zB_i
&\ge \frac{1}{32n\lambda_i^2}\left(\frac14-b_i\right).
\end{align*}
By the definition of $k^*$, for $i\le k^*$,
\begin{align*}
\frac{\sum_{j>k^*}\lambda_j^2}{n\lambda_i^2}
\le \frac{\tilde\lambda^2}{c_2\lambda_i^2}
\le c_2.
\end{align*}
Substituting these bounds into the earlier lower bound for $\Ebb\effVar$ yields
\begin{align*}
\Ebb\effVar
&\ge \frac{1}{128c_1}\frac{\sum_{j>k^*}\lambda_j^2}{n}
\|\wB^*_{\le k^*}\|_{\SigmaB_{\le k^*}^{-1}}^2
-\frac{c_2}{32c_1}\Ebb\effBias.
\end{align*}
Rescaling constants completes the proof.
\end{proof}

\subsection{Proof of GD lower bound}
\begin{proof}[Proof of \Cref{thm:gd:bound}, lower bound]
Following the notation introduced, the risk decomposes into 
\begin{align*}
    \Ebb \excessRisk(\hat\wB_t^{\gd}) &\ge \Ebb \effBias+\Ebb \effVar + \frac{\sigma^2}{c_y} \Ebb \variance,
\end{align*}
where all the random variables are nonnegative. 
Note that a high-probability lower bound of a nonnegative random variable implies an expectation lower bound with rescaled constant factors. 
We then apply \Cref{lemma:gd:lower-bound:variance,lemma:gd:lb:effbias:exp,lemma:gd:lb:effbias:ols,lemma:gd:lb:effbias:tail,lemma:gd:lb:effvar} in their expectation version.
For bias error, there is a constant $c_1> 1$ such that
\begin{align*}
\Ebb \effBias &\ge \frac{1}{c_1}\max\bigg\{ \big\|(\IB-\eta\SigmaB)^t\wB^*\big\|_{\SigmaB}^2 ,\, \| \wB^*_{>k^*} \|^2_{\SigmaB_{>k^*}} \bigg\}, && \explain{\Cref{lemma:gd:lb:effbias:exp,lemma:gd:lb:effbias:tail}} \\ 
\Ebb \effBias &\ge \frac{1}{c_1} \bigg(\frac{\sum_{i>k^*}\lambda_i}{n} \bigg)^2\|\wB^*_{\le k^*}\|^2_{\SigmaB_{\le k^*}^{-1}} \ \ \text{if} \ \  \frac{\sum_{i>k^*}\lambda_i}{n}  \ge c_5 \sqrt{\frac{\sum_{i>k^*}\lambda_i^2}{n}}. && \explain{\Cref{lemma:gd:lb:effbias:ols}}
\end{align*}
Combining with \Cref{lemma:gd:lb:effvar}, for some sufficiently large constant $c_1'> 1$, we have 
\begin{align*}
&\ \Ebb \effBias + \Ebb \effVar  \\
   &\ge \frac{1}{c_1'} \Bigg( \big\|(\IB-\eta\SigmaB)^t\wB^* \big\|_{\SigmaB}^2 + \| \wB^*_{>k^*} \|^2_{\SigmaB_{>k^*}}+ \bigg(\frac{\sum_{i>k^*}\lambda_i}{n} \bigg)^2\|\wB^*_{\le k^*}\|^2_{\SigmaB_{\le k^*}^{-1}}+ \frac{\sum_{i>k^*}\lambda_i^2}{n} \|\wB^*_{\le k^*}\|^2_{\SigmaB_{\le k^*}^{-1}} \Bigg).
\end{align*}
For variance, there is a constant $c_1> 1$ such that 
\begin{align*}
\Ebb\variance\ge \frac{1}{c_1}\min\bigg\{ \frac{k^*+(1/\tilde\lambda^2)\sum_{i>k^*}\lambda_i^2}{n} ,\, 1\bigg\}.
\end{align*}
Putting them together, we obtain the lower bound in \Cref{thm:gd:bound} with a rescaling of the constants.
\end{proof}

\section{PCR Risk Bounds}\label{sec:pcr-proof}

\subsection{Proof of Theorem \ref{thm:pcr-risk}}

\begin{proof}[Proof of \Cref{thm:pcr-risk}, upper bound]
We show a slightly strengthened upper bound, which introduces a slackness parameter $\eps\in (0,1]$. Let $n\ge c_0\eps^{-2}$ and
\begin{align*}
    k^*:= \min\left\{k: (1+\eps)\rho + \frac{1}{c_2} \frac{\sum_{i>k}\lambda_i}{n} \ge \lambda_{k+1} \right\}.
\end{align*}
We prove the following statement: under \Cref{assum:lbb}, if additionally $k^*\le \eps^2 n/c_3$, then with probability at least $1-\delta-\exp(-\eps^2 n/c_0)$ over the randomness of sampling $\XB$ it holds that
\begin{align*}
  \frac{1}{c_1} \Ebb[ \excessRisk(\hat\wB_\rho^{\pcr}) | \XB]\le \frac{1}{\eps^2} \left(\alpha_{k^*} + \frac{\lam_{k^*+1}^2\log(1/\delta)}{n}\right) \|\wB^*\|^2_{\SigmaB_{0:k^*}^{-1}}  + \|\wB^*\|^2_{\SigmaB_{k^*:\infty}} + \sigma^2 \frac{\min\{D, \hat k\}}{n}.
\end{align*}
First suppose $\rho=0$, then PCR coincides with minimum-norm OLS. Let $k=k^*$ and $s:=n^{-1}\sum_{i>k}\lambda_i$. If $s>0$, then $s\ge c_2\lambda_{k+1}$, so from
\citet[Proposition~2.1]{wu2026risk}, with
probability at least $1-\exp(-n/c_0)$,
\begin{align*}
\frac{1}{c_1}\Ebb[\excessRisk(\hat\wB_0^{\pcr})|\XB]
\le
s^2\|\wB^*\|_{\SigmaB_{0:k}^{-1}}^2
+\|\wB^*\|_{\SigmaB_{k:\infty}}^2
+\sigma^2 \frac{D}{n}.
\end{align*}
Moreover, \Cref{lemma:basic:concentration} gives
$\AB_{>k}\succeq ns\IB/2$, hence $\hat k=n$, while
$D=k+s^{-2}\sum_{i>k}\lambda_i^2
\le n (c_2^{-1}+c_3^{-1})$.
The desired bound follows.

If $s=0$, minimality of $k$ implies $\rank(\SigmaB)=k$. By
\Cref{lemma:basic:concentration}, $D=k=\hat k$ and the OLS bias vanishes, and the claim follows since
\begin{align*}
\Ebb[\excessRisk(\hat\wB_0^{\pcr})|\XB]
\le \frac{c_y\sigma^2}{n}
\tr(\SigmaB_{\le k}\hat\SigmaB_{\le k}^{-1})
\le \frac{2c_y\sigma^2k}{n}.
\end{align*}
Now assume $\rho>0$.
Set $k=k^*$,
$\lambda=c_2(1+\eps)\rho$ and take $I=\Rbb_{\ge 0}$, $H=[(1+\eps/2)\rho, \infty)$ in \Cref{thm:master-ub}. Since $\lambda_k>(1+\eps)\rho$ by definition of~$k$,
\Cref{lemma:basic:concentration} gives
\begin{align*}
\lambda_{\min}(\hat\SigmaB_{\le k})
\ge (1-\eps/4)\lambda_k
> (1+\eps/2)\rho
\end{align*}
with probability $1-\exp(-\eps^2 n/c_0)$ after adjusting constants, hence \eqref{eq:spec-cond} holds. If $k=0$, we take $H=\varnothing$ and omit the head terms. Then the proof of \Cref{lemma:schur:examples} gives
\begin{align*}
\|r(I,H)\|_{\Sfrak} = \bigO(\eps^{-1}),
\quad
\|r(H,H)\|_{\Sfrak}=0,
\quad
\|r(I,\{0\})\|_{\Sfrak} = \bigO(1).
\end{align*}
Here, note that changing the reference parameter from~$\rho$ to~$\lambda$ changes these bounds by at most constant factors by \Cref{lemma:schur:factorization-bound}.
Substituting into \Cref{thm:master-ub} gives the claimed bias bound.

Finally, the variance bound in \Cref{thm:master-ub} and the
ridge variance upper bound \citep{tsigler2023benign} bound the variance as
\begin{align*}
\variance(\hat\wB_\rho^{\pcr}) \le \frac{c\sigma^2}{n}
\min\left\{
k+\frac{\sum_{i>k}\lambda_i^2}
        {(\lambda+n^{-1}\sum_{i>k}\lambda_i)^2},
\,\hat k
\right\}
\le \frac{c\sigma^2}{n}\min\{D,\hat k\},
\end{align*}
where we used $\lambda>\rho$ and
$\rank (g(\hat\SigmaB))=\hat k$.
Combining the bounds gives the result.
\end{proof}

\begin{proof}[Proof of \Cref{thm:pcr-risk}, lower bound]
We apply \Cref{thm:master-lb} at index $r^*$ with the constant $0.9$ in the definition of $\Gamma_{ij}$ adjusted to $0.99$. For $i\le k^*$ and $j>r^*$, it holds that $0.99\lam_i>\rho$ and $1.1\lam_j\le\rho$, so that $\Gamma_{ij}=1$ and $\Lambda_i \le \lam_i + 0.9\rho \le 2\lam_i$. Thus
\begin{align*}
c_1 \Ebb\bias(\hat\wB_\rho^{\pcr})&\ge \sum_{i\le k^*}\frac{\lam_i}{\Lambda_i^2} \left(\frac1n\sum_{j>r^*}\lam_j^2\Gamma_{ij}^2\right) \wB_i^{*2} + \big\|(\IB-g(1.1\SigmaB))\wB^*\big\|_{\SigmaB}^2\\
&\ge \frac{\sum_{i>r^*}\lam_i^2}{4n} \|\wB^*\|_{\SigmaB_{0:k^*}^{-1}}^2 + \|\wB^*\|_{\SigmaB_{r^*:\infty}}^2.
\end{align*}
If moreover
\begin{align*}
\frac{\sum_{i>r^*}\lam_i^2}{n}> \frac{1}{c_2} \left(\frac{\sum_{i>r^*}\lam_i}{n}\right)^2,
\end{align*}
the claim immediately follows. Otherwise, by \Cref{thm:master-lb} and
\begin{align*}
\frac{\sum_{j>r^*}\lam_j}{n} \le 0.9\rho < \lam_i
\end{align*}
for all $i\le k^*$, we have that
\begin{align*}
c_1 \Ebb\bias(\hat\wB_\rho^{\pcr})
\ge \sum_{i\le k^*}\lam_i
\left(\frac{\sum_{j>r^*}\lam_j}{n\lam_i}\right)^2\wB_i^{*2} \ge  \left(\frac{\sum_{i>r^*}\lam_i}{n}\right)^2 \|\wB^*\|_{\SigmaB_{0:k^*}^{-1}}^2.
\end{align*}
In both cases, we have the lower bound
\begin{align*}
c_1 \Ebb\bias(\hat\wB_\rho^{\pcr})
\ge \alpha_{r^*} \|\wB^*\|_{\SigmaB_{0:k^*}^{-1}}^2.
\end{align*}
The final claim on variance error follows from the lemma below, which also holds more generally under \Cref{assum:lbb}.
\end{proof}

\begin{lemma}\label{lem:pcr-variance-lb}
There exist constants $c_0,c_1,c_3>1$, depending only on $c_x,c_y$, such that the following holds. Let~$\hat k, k^*,r^*$ be as in \Cref{thm:pcr-risk}. Suppose $\rho>0$, $r^*\le n/c_3$, and define
\begin{align*}
\tilde\rho
:=
\max\left\{1.1\rho, c_3\frac{\sum_{i>k^*}\lam_i}{n} \right\}.
\end{align*}
Then,
\begin{align*}
\text{with probability at least $1-\exp(-n/c_0)$,} \quad & c_1\variance(\hat\wB_\rho^{\pcr}) \ge
\frac{\sigma^2}{n}
\#\{j:\hat\lam_j>\tilde\rho\}; \\
\text{in expectation,} \quad & c_1\Ebb \variance(\hat\wB_\rho^{\pcr}) \ge
\frac{\sigma^2}{n}
\#\{j:\lam_j>\tilde\rho\}.
\end{align*}
\end{lemma}

\begin{proof}[Proof of \Cref{lem:pcr-variance-lb}]
By the definition of $\tilde\rho$,
\begin{align*}
\tr\left(\SigmaB(\SigmaB+\tilde\rho\IB)^{-1}\right)
&=
\sum_i\frac{\lam_i}{\lam_i+\tilde\rho} \le
k^*+\frac1{\tilde\rho}\sum_{i>k^*}\lam_i
\le
\frac{2n}{c_3}.
\end{align*}
Then after increasing $c_3$, with probability at least $1-\exp(-n/c_0)$, it holds that
\begin{align*}
   \SigmaB \preceq 2 (\hat\SigmaB + \tilde\rho \IB),\quad  \hat\SigmaB \preceq \frac{3}{2} \SigmaB + \frac{1}{2}\tilde\rho \IB.
\end{align*}
This appears in, e.g., \citet[Lemma 5]{hucker2023note}, which is a direct corollary of \citet{koltchinskii2017concentration} applied to the concentration of the empirical covariance matrix for subgaussian random vectors with covariance matrix $(\SigmaB+\tilde\rho\IB )^{-1}\SigmaB$. Therefore, for every pair $(\hat\lam_j,\hat\uB_j)$,
\begin{align*}
\frac{1}{\hat\lam_j} \hat\uB_j^\top\SigmaB\hat\uB_j
\ge
\frac13\left(2-\frac{\tilde\rho}{\hat\lam_j}\right)_+.
\end{align*}
It follows that
\begin{align*}
\variance(\hat\wB_\rho^{\pcr})
&\ge
\frac{\sigma^2}{c_y n}
\sum_{j:\hat\lam_j>\rho}
\frac{1}{\hat\lam_j} \hat\uB_j^\top\SigmaB\hat\uB_j\\
&\ge
\frac{\sigma^2}{3c_y n}
\sum_{j:\hat\lam_j>\rho}
\left(2-\frac{\tilde\rho}{\hat\lam_j}\right)_+\\
&\ge
\frac{\sigma^2}{3c_y n}
\#\{j:\hat\lam_j>\tilde\rho\}.
\end{align*}
Moreover for each~$j$ with $\lam_j > \tilde\rho$, we have $j\le r^* \le n/c_3$. Then with probability $1-\exp(-n/c_0)$ it holds that $\hat\lam_j \ge 0.99\lam_j > 0.99\tilde\rho$ by \Cref{lemma:basic:concentration} after increasing constants if necessary, and the second claim also follows by taking expectations.
\end{proof}

\subsection{Proof of Lemma \ref{lem:two-index-example}}

\begin{proof}[Proof of \Cref{lem:two-index-example}]
To see the thresholded index, note that
\begin{align*}
1.1\rho+\frac{1}{c_2n}\sum_{i>0}\lam_i
&=
1.1\rho
+\frac{2\rho+m_n(1\pm\eps)\rho}{c_2n}
< 2\rho =\lam_1,\\
1.1\rho+\frac{1}{c_2n}\sum_{i>1}\lam_i
&=
1.1\rho+\frac{m_n(1\pm\eps)\rho}{c_2n}
\ge
(1\pm\eps)\rho =\lam_2,
\end{align*}
hence $k^*=1$. Also, for every $r\le m_n$,
\begin{align*}
\lam_{r+1}+\frac1n\sum_{i>r}\lam_i
\ge
(1-\eps)\rho
>
0.9\rho,
\end{align*}
so $r^*$ is vacuously equal to $m_n+1$. Then for both instances, the bias lower bound from \Cref{thm:pcr-risk} vanishes, while the bias upper bound is
\begin{align*}
\alpha_1 \|\wB_n^*\|_{(\SigmaB_n^\pm)_{0:1}^{-1}}^2
&=
\left(\frac{m_n(1\pm\eps)^2\rho^2}{n}
+
\left(\frac{m_n(1\pm\eps)\rho}{n}\right)^2 \right) \frac{1}{2\rho} = \Theta(\eps^2\rho).
\end{align*}
We now derive the actual risk of both instances. First consider $\SigmaB_n^-$. Let
\begin{align*}
\XB_n=\YB_n(\SigmaB_n^-)^{1/2},
\quad
\YB_n=[\zB_1\ \ZB],
\end{align*}
where $\YB_n\in\Rbb^{n\times(m_n+1)}$ has independent standard Gaussian
entries. In the population eigenbasis,
\begin{align*}
\hat\SigmaB_n
=
\begin{bmatrix}
h&\GB^\top\\
\GB&\TB
\end{bmatrix},
\quad
\GB=\frac{\sqrt{2(1-\eps)} \rho}{n}\ZB^\top\zB_1.
\end{align*}
By standard concentration, with probability at least
$1-\exp(-\eps^2n/c_0)$,
\begin{align*}
(1-\eps)\SigmaB_n^-
\preceq
\hat\SigmaB_n
\preceq
(1+\eps)\SigmaB_n^-.
\end{align*}
On this event, $\|\TB\| \le (1-\eps^2)\rho<\rho$ and $h\ge 2(1-\eps)\rho \ge 1.9\rho$, hence exactly one
eigenvalue is selected by PCR. Let~$\lam$ be this eigenvalue and let $(a,\bB)$ be its unit eigenvector,
with $a\ge0$. Then
\begin{align*}
\bB=(\lam\IB-\TB)^{-1}\GB a.
\end{align*}
Since $\hat\SigmaB_n\succeq0$, it holds that
$\GB\GB^\top\preceq h\TB$ and $\lam \ge h \ge 1.9\rho$. Then
\begin{align*}
\frac{\|\bB\|}{a}
\le
\frac{\|\GB\|}{\lam-\rho}
\le
\frac{\sqrt{h\rho}}{h-\rho}< 2 \quad\implies\quad a^2\ge \frac15
\end{align*}
and
\begin{align*}
\|\bB\|^2
=
a^2\GB^\top(\lam\IB-\TB)^{-2}\GB
\ge \frac{\|\GB\|^2}{5\lam^2}.
\end{align*}
Moreover with probability at least $1-\exp(-\eps^2n/c_0)$,
\begin{align*}
\|\GB\|^2 = \frac{2(1-\eps)\rho^2}{n^2} \|\ZB^\top\zB_1\|^2 \ge \frac{2(1-\eps)\rho^2}{n^2} \frac{m_nn}{2} = \Theta(\eps^2\rho^2).
\end{align*}
Hence PCR satisfies
\begin{align*}
\bias(\hat\wB_\rho^{\pcr})
&=
\|(\IB-(a,\bB)(a,\bB)^\top)\eB_1\|_{\SigmaB_n^-}^2\\
&= 2\rho(1-a^2)^2+(1-\eps)\rho a^2\|\bB\|^2\\
&\ge c\rho \frac{\|\GB\|^2}{\lam^2} = \Theta(\eps^2\rho),
\end{align*}
where we also used
$\lam\le 2(1+\eps)\rho$. Thus the upper bound is tight in this case.

On the other hand, under $\SigmaB_n^+$, with probability at least $1-\exp(-\eps^2n/c_0)$,
\begin{align*}
\hat\SigmaB_n
\succeq
\left(1-\frac{\eps}{2}\right)\SigmaB_n^+
\succeq
\left(1-\frac{\eps}{2}\right)(1+\eps)\rho\IB
>
\rho\IB.
\end{align*}
Thus PCR is equivalent to OLS. Since $m_n<n$, $\XB^\top\XB$ has full rank almost surely, so $\bias(\hat\wB_\rho^{\pcr})=0$ and the lower bound is tight in this case.
\end{proof}

\section{Proofs for PCR Dominance}

\paragraph{Additional notation.} Let $g$ be a monotone shrinkage filter. We define
\begin{align*}
\GB:=g(\hat\SigmaB),\quad \BB:=(\IB-\GB)\SigmaB(\IB-\GB),\quad p_i:=\Ebb\eB_i^\top\GB\eB_i,
\end{align*}
and set $\GB_{ij}:=\eB_i^\top\GB\eB_j$, $\BB_{ij}:=\eB_i^\top\BB\eB_j$. Additionally denote
\begin{align*}
\PsiB:=\frac1n\psi(\AB/n),\quad \AB_{-i}:=\sum_{j\ne i}\lam_j\zB_j\zB_j^\top,\quad \VB_{-i}:=\sum_{j\ne i}\lam_j^2\zB_j\zB_j^\top.
\end{align*}
We assume $c=1$ for simplicity; otherwise, all variance upper and lower bounds are scaled by a factor of at most $c^2$ and $1/c^2$, respectively. The bias-variance decomposition can be written as
\begin{align*}
\Ebb [\excessRisk_{\mu}(\hat \wB_g)|\XB] &= \|(\IB - \psi(\hat\SigmaB)\hat\SigmaB)\wB^*\|_{\SigmaB}^2 + \frac{\sigma^2}{n^2} \tr(\XB\psi(\hat\SigmaB)\SigmaB \psi(\hat\SigmaB)\XB^\top) \\
&= \|(\IB-\GB)\wB^*\|_{\SigmaB}^2 + \frac{\sigma^2}{n}\tr(\SigmaB \hat\SigmaB^{-1}\GB^2).
\end{align*}
We thus denote, conditional on $\XB$,
\begin{align*}
\variance(\hat\wB_g) &:= \frac{\sigma^2}{n}\tr(\SigmaB \hat\SigmaB^{-1}\GB^2),\\
\bias(\hat\wB_g) &:= \|(\IB-\GB)\wB^*\|_{\SigmaB}^2.
\end{align*}
By sign symmetry, it holds in expectation that
\begin{align*}
\Ebb\bias(\hat\wB_g)=\sum_i\wB_i^{*2}\Ebb\BB_{ii}.
\end{align*}
We use the following equality multiple times in the proofs, which is obtained by direct expansion:
\begin{align*}
\BB_{ii}
&=\lam_i(1-\GB_{ii})^2+\lam_i\sum_{j\ne i}\lam_j^2(\zB_j^\top\PsiB\zB_i)^2.
\end{align*}
We also define the \emph{effective bias} for general filters as
\begin{align*}
\effBias(\hat\wB_g) := \sum_i \lam_i \wB_i^{*2} (1-\GB_{ii})^2.
\end{align*}
From the above, it is clear that in expectation, $\Ebb\bias(\hat\wB_g) \ge \Ebb\effBias(\hat\wB_g)$.

\paragraph{Fixed design.} For completeness, we first provide a short deterministic proof of PCR dominance in the fixed design setting, following the argument of \citet{dhillon2013risk}.

\begin{lemma}[PCR dominates monotone filters for fixed design]
\label{lem:fixed-design}
Let $\XB$ be fixed and suppose
$\Ebb[\yB\,|\,\XB]=\XB\wB^*$,
$\Ebb[\|\yB-\XB\wB^*\|^2\,|\,\XB]<\infty$. Denote the excess risk by $\excessRisk_{\XB}(\wB)
:=\|\wB-\wB^*\|_{\hat\SigmaB}^2$. For every monotone filter $g\in\mathcal G$, there exists
$\rho\ge0$, depending only on~$\XB,g$, such that
\begin{align*}
\Ebb[\excessRisk_{\XB}(\hat\wB_\rho^{\pcr})|\XB]
\le 4\Ebb[\excessRisk_{\XB}(\hat\wB_g)|\XB].
\end{align*}
\end{lemma}

\begin{proof}[Proof of \Cref{lem:fixed-design}]
Write $\hat\SigmaB=\SigmaB = \sum_i \lam_i \uB_i\uB_i^\top$ and $\yB-\XB\wB^* = \varepsilon$, $\vB_i = \XB\uB_i/\sqrt{n\lam_i}$. For all $g\in\mathcal G$, the bias-variance decomposition gives
\begin{align*}
\Ebb[\excessRisk_{\XB}(\hat\wB_g)|\XB]
&=\sum_{i:\lam_i>0}
\lam_i(1-g(\lam_i))^2
\wB_i^{*2} + \frac{1}{n} \sum_{i} g(\lam_i)^2 \Ebb[(\vB_i^\top\varepsilon)^2 |\XB].
\end{align*}
Let the PCR filter be $g_\rho(z):=\mathbf 1\{z>\rho\}$ and choose
\begin{align*}
\rho:=\max\left(
\{0\}\cup
\{\lam_i:\lam_i>0,\ g(\lam_i)\le1/2\}
\right).
\end{align*}
By monotonicity, it holds for all~$i$ that $g_\rho(\lam_i)
=\mathbf 1\{g(\lam_i)>1/2\}$, and so
\begin{align*}
g_\rho(\lam_i)^2
&\le4g(\lam_i)^2,\quad (1-g_\rho(\lam_i))^2
\le4(1-g(\lam_i))^2.
\end{align*}
Substituting into the above proves the claim.
\end{proof}

\subsection{Useful lemmas}

We first show a spectral result involving the leave-one-out matrix $\AB_{-i}$.

\begin{lemma}\label{lem:loo}
Let~$g$ be a monotone filter and let~$i$ be any index with $\lam_i>0$. Choose orthonormal eigenbases $(\beta_j,\vB_j)_{1\le j\le n}, (\gamma_j,\wB_j)_{1\le j\le n}$ such that
\begin{align*}
\frac{\AB}{n}=\sum_{j=1}^n\beta_j\vB_j\vB_j^\top, \quad \frac{\AB_{-i}}{n}= \sum_{j=1}^n\gamma_j\wB_j\wB_j^\top,
\end{align*}
and define
\begin{align*}
h_{g,j}(\zB_i):= n\frac{\zB_i^\top \PsiB\wB_j}{\zB_i^\top\wB_j} \mathbf{1}\{\zB_i^\top\wB_j\ne0\}.
\end{align*}
Then conditional on $\{\zB_j:j\ne i\}$, it holds that $h_{g,j}(\zB_i)\ge 0$ and $h_{g,j}$ is invariant under sign changes in the basis $(\wB_j)$. Moreover, for all $\lam>0$,
\begin{align*}
h_{g,j}(\zB_i) \ge \frac{(g(\gamma_j) - g(\lam))_+}{\gamma_j} \left(1 - \frac{\lam_i\|\zB_i\|^2}{\lam n}\right),
\end{align*}
where the right-hand side is interpreted as zero if $\gamma_j=0$.
\end{lemma}

\begin{proof}[Proof of \Cref{lem:loo}]
Note that $\zB_i^\top\wB_j\ne0$ for every~$j$ almost surely; we work under this event. From
\begin{align*}
(\beta_k-\gamma_j)\vB_k^\top \wB_j &= \vB_k^\top \frac{\AB}{n}\wB_j - \vB_k^\top \frac{\AB_{-i}}{n}\wB_j\\
&=\frac{\lam_i}{n}(\vB_k^\top \zB_i)(\zB_i^\top\wB_j),
\end{align*}
we see that $\beta_k=\gamma_j$ implies $\vB_k^\top\zB_i=0$. Since $\AB/n\succeq(\lam_i/n)\zB_i\zB_i^\top$, $\beta_k=0$ also implies $\vB_k^\top\zB_i=0$. Otherwise,
\begin{align*}
\frac{(\zB_i^\top\vB_k)(\vB_k^\top\wB_j)}{\zB_i^\top\wB_j}=\frac{\lam_i}{n}\frac{(\vB_k^\top\zB_i)^2}{\beta_k-\gamma_j},\quad \beta_k\notin\{\gamma_j,0\}.
\end{align*}
Hence summing over $k$ gives the identity
\begin{align*}
\frac{\lam_i}{n}\sum_{k:\beta_k>0,\beta_k\ne\gamma_j}\frac{(\vB_k^\top\zB_i)^2}{\beta_k-\gamma_j} = 1.
\end{align*}
From the layer-cake representation
\begin{align*}
\psi(\AB/n)=\int_0^1(\AB/n)^{-1}\mathbf 1\{g(\AB/n)>s\IB\}\dif s,
\end{align*}
we also have
\begin{align*}
h_{g,j}(\zB_i)&= \int_0^1 h_{g,j}(\zB_i; s) \dif s
\end{align*}
where
\begin{align*}
h_{g,j}(\zB_i; s) &:= \frac{\zB_i^\top(\AB/n)^{-1}\mathbf 1\{g(\AB/n)>s\IB\}\wB_j}{\zB_i^\top\wB_j} \\
&= \frac{\lam_i}{n}\sum_{\substack{k:g(\beta_k)>s\\ \beta_k>0,\beta_k\ne\gamma_j}}\frac{(\vB_k^\top\zB_i)^2}{\beta_k(\beta_k-\gamma_j)}.
\end{align*}
When $\gamma_j=0$, $h_{g,j}(\zB_i;s)\ge 0$, so $h_{g,j}(\zB_i)\ge 0$. Suppose $\gamma_j>0$ and fix $s\in(0,1)$. If $g(\gamma_j)\le s$, then every eigenvalue $\beta_k$ included in the above sum satisfies $g(\beta_k)>s\ge g(\gamma_j)$, so $\beta_k>\gamma_j$ and the sum is nonnegative. If instead $g(\gamma_j)>s$, every omitted $\beta_k$ (except possibly $\gamma_j,0$) satisfies $\beta_k<\gamma_j$ and contributes a negative summand, hence
\begin{align*}
h_{g,j}(\zB_i; s) &\ge \frac{\lam_i}{n}\sum_{\substack{k:\beta_k>0,\beta_k\ne\gamma_j}}\frac{(\vB_k^\top\zB_i)^2}{\beta_k(\beta_k-\gamma_j)} \\
&= \frac{\lam_i}{n}\sum_{\substack{k:\beta_k>0,\beta_k\ne\gamma_j}} \frac1{\gamma_j}\left(\frac1{\beta_k-\gamma_j}-\frac1{\beta_k}\right) (\vB_k^\top\zB_i)^2 \\
&=\frac1{\gamma_j}\big(1-\lam_i \zB_i^\top\AB^{-1}\zB_i\big) \ge 0.
\end{align*}
Therefore integrating over~$s$, it follows that $h_{g,j}(\zB_i)\ge 0$.

Moreover if $g(\lam) <s< g(\gamma_j)$, every included eigenvalue also satisfies $\beta_k > \lam$, hence
\begin{align*}
h_{g,j}(\zB_i; s) &= \frac{\lam_i}{n}\sum_{\substack{k:g(\beta_k)>s\\ \beta_k>0,\beta_k\ne\gamma_j}} \frac1{\gamma_j}\left(\frac1{\beta_k-\gamma_j}-\frac1{\beta_k}\right) (\vB_k^\top\zB_i)^2 \\
&\ge \frac1{\gamma_j} \left(1 - \frac{\lam_i\|\zB_i\|^2}{\lam n}\right).
\end{align*}
It follows that
\begin{align*}
h_{g,j}(\zB_i) \ge \frac{(g(\gamma_j) - g(\lam))_+}{\gamma_j} \left(1 - \frac{\lam_i\|\zB_i\|^2}{\lam n}\right).
\end{align*}
Finally, let $\SB$ be an arbitrary diagonal sign matrix in the eigenbasis $(\wB_j)$. The transformation $\zB_i\mapsto \SB\zB_i$ preserves the law of $\zB_i$ and maps $\AB$ to $\SB\AB\SB$, hence
\begin{align*}
h_{g,j}(\SB\zB_i)= n\frac{(\SB\zB_i)^\top \SB\PsiB\SB \wB_j}{(\SB\zB_i)^\top\wB_j} = h_{g,j}(\zB_i).
\end{align*}
This completes the proof.
\end{proof}

Using this lemma, we prove the following `gluing' inequality for spectral filters.

\begin{lemma}\label{lem:glue}
Suppose $g = \sum_{k=1}^m a_k g_k$ where $g_1,\dots,g_m:\Rbb_{\ge 0}\to [0,1]$ are nondecreasing and $a_k\ge 0$, $\sum_{k=1}^m a_k\le 1$. Then
\begin{align*}
\Ebb \excessRisk_{\mu}(\hat \wB_g) \ge \sum_{k=1}^m a_k^2 \Ebb  \excessRisk_{\mu}(\hat \wB_{g_k}).
\end{align*}
\end{lemma}

\begin{proof}[Proof of \Cref{lem:glue}]
Denote $\psi_k(z)=g_k(z)/z$, $\PsiB_k:=\frac1n\psi_k(\AB/n)$ and $\GB_k := g_k(\hat \SigmaB)$, so that
\begin{align*}
\GB_{k,ij} = \eB_i^\top\XB^\top \PsiB_k \XB\eB_j = \sqrt{\lam_i\lam_j}\zB_i^\top \PsiB_k \zB_j.
\end{align*}
We may additionally suppose $\sum_{k=1}^m a_k=1$ as otherwise we can add a dummy filter $g_{m+1}=0$.

We claim that
\begin{align}\label{eq:ege}
\Ebb \GB_{k,ij}\GB_{\ell,ij} \ge 0, \quad  i,j\ge 1, \quad k,\ell\le m.
\end{align}
We may suppose $i\ne j$ and $\lam_i,\lam_j>0$ and condition on $\{\zB_p:p\ne i\}$. From \Cref{lem:loo}, it holds that
\begin{align*}
h_{k,p}(\zB_i):= h_{g_k,p}(\zB_i):= n\frac{\zB_i^\top\PsiB_k\wB_p}{\zB_i^\top\wB_p} \mathbf{1}\{\zB_i^\top\wB_p\ne0\} \ge 0,
\end{align*}
and similarly $h_{\ell,p}(\zB_i)\ge0$. Moreover, let $\SB$ be an arbitrary diagonal sign matrix in the eigenbasis $(\wB_j)$, then both functions are invariant under the transformation $\zB_i\mapsto \SB\zB_i$. Averaging over all choices of $\SB$ gives (noting that $|\zB_i^\top \PsiB_k\zB_j| \le (\lam_i\lam_j)^{-1/2}$ is conditionally integrable)
\begin{align*}
\Ebb_{\zB_i} \GB_{k,ij}\GB_{\ell,ij}
&= \frac{\lam_i\lam_j}{n^2} \frac{1}{2^n} \sum_{\SB} \Ebb_{\zB_i}
(\SB\zB_i)^\top\psi_k(\SB\AB\SB/n)\zB_j
(\SB\zB_i)^\top\psi_\ell(\SB\AB\SB/n)\zB_j\\
&= \frac{\lam_i\lam_j}{n^2} \sum_p \Ebb_{\zB_i}
(\wB_p^\top\zB_i)^2 (\wB_p^\top\zB_j)^2
h_{k,p}(\zB_i)h_{\ell,p}(\zB_i) \ge0.
\end{align*}
Taking expectations over the remaining columns proves~\eqref{eq:ege}.

From \Cref{eq:ege}, it now follows that for all $i\ge 1$ and $k\ne\ell$,
\begin{align*}
\Ebb \eB_i^\top (\IB-\GB_k)\SigmaB(\IB-\GB_\ell) \eB_i &= \lam_i \Ebb (\eB_i^\top(\IB-\GB_k)\eB_i)(\eB_i^\top(\IB-\GB_\ell)\eB_i) + \sum_{j\ne i} \lam_j \Ebb (\eB_i^\top\GB_k\eB_j)(\eB_i^\top\GB_\ell\eB_j) \ge 0.
\end{align*}
Since $\IB-\GB = \sum_{k=1}^m a_k(\IB-\GB_k)$, we thus have
\begin{align*}
\Ebb\bias(\hat\wB_g) &= \Ebb \|(\IB-\GB)\wB^*\|_{\SigmaB}^2 \\
&= \sum_i \sum_{k,\ell=1}^m a_k a_\ell \wB_i^{*2} \Ebb \eB_i^\top (\IB-\GB_k) \SigmaB (\IB-\GB_\ell) \eB_i \\
&\ge \sum_i \sum_{k=1}^m a_k^2 \wB_i^{*2} \Ebb \eB_i^\top (\IB-\GB_k) \SigmaB (\IB-\GB_k) \eB_i\\
&= \sum_{k=1}^m a_k^2 \Ebb \|(\IB-\GB_k)\wB^*\|_{\SigmaB}^2 \\
&= \sum_{k=1}^m a_k^2 \Ebb\bias(\hat\wB_{g_k}).
\end{align*}
Moreover the expected variance satisfies
\begin{align*}
\Ebb\variance(\hat\wB_g) &= \frac{\sigma^2}{n} \Ebb \tr(\SigmaB \hat\SigmaB^{-1} \GB^2) \\
&= \frac{\sigma^2}{n} \sum_{k,\ell=1}^m a_ka_\ell \Ebb \tr(\SigmaB \hat\SigmaB^{-1} \GB_k \GB_\ell) \\
&\ge \frac{\sigma^2}{n} \sum_{k=1}^m a_k^2 \Ebb \tr(\SigmaB \hat\SigmaB^{-1} \GB_k^2) \\
&= \sum_{k=1}^m a_k^2 \Ebb\variance(\hat\wB_{g_k}).
\end{align*}
Here, we have used that the factors of $\hat\SigmaB^{-1}\GB_k\GB_\ell$ are commuting nonnegative functions of $\hat\SigmaB$, hence it is positive semidefinite. The proof is complete.
\end{proof}

We also show the following entrywise lower bound for the off-diagonals of~$\GB$, which is based on an application of the Cram\'{e}r--Rao inequality as discussed at the end of the proof.

\begin{lemma}\label{lem:general-filter-coordinates}
The sequence $(p_i)_{i\ge 1}$ is nonincreasing, with $p_i=p_j$ when $\lam_i=\lam_j$. If $\lam_i\ne\lam_j$, then
\begin{align}\label{eq:general-filter-rotation}
\Ebb \GB_{ij}^2\ge\frac{\lam_i\lam_j}{n(\lam_i-\lam_j)^2}(p_i-p_j)^2.
\end{align}
\end{lemma}

\begin{proof}[Proof of \Cref{lem:general-filter-coordinates}]
By monotone approximation, it suffices to prove the lemma when $g$ is smooth. Let $\Phi:\Rbb_{>0}\to\Rbb$ be an antiderivative of $\psi$ and define
\begin{align*}
F(\theta):=\Ebb\tr\big(\Phi(\AB_\theta/n)\big), \quad \text{where}\quad \AB_\theta:=\sum_i e^{\theta_i}\zB_i\zB_i^\top.
\end{align*}
Here, the sum runs over indices $i$ such that $\lam_i>0$, so that $\AB_\theta=\AB$ when $\theta_i=\log\lam_i$.

We will show that $F$ is convex. Letting $\theta$ vary along an affine direction parametrized by $t$ with $\dot\theta=\frac{\dif\theta}{\dif t}$ constant,
\begin{align*}
\dot\AB_\theta=\sum_i \dot\theta_i e^{\theta_i}\zB_i\zB_i^\top,\quad \ddot\AB_\theta=\sum_i \dot\theta_i^2e^{\theta_i}\zB_i\zB_i^\top.
\end{align*}
Since $\ran\dot\AB_\theta \subseteq \ran\AB_\theta$, it holds for every $\vB\in\Rbb^n$,
\begin{align*}
\vB^\top \dot\AB_\theta \AB_\theta^{-1} \dot\AB_\theta \vB &= \sup_{u\in\Rbb^n} 2\uB^\top \dot\AB_\theta \vB - \uB^\top \AB_\theta \uB \\
&=\sup_{\uB\in\Rbb^n} \sum_i e^{\theta_i}\big(2\dot\theta_i(\zB_i^\top\vB)(\zB_i^\top\uB)-(\zB_i^\top\uB)^2 \big)\\
&\le \sum_i \dot\theta_i ^2e^{\theta_i}(\zB_i^\top\vB)^2=\vB^\top \ddot\AB_\theta \vB,
\end{align*}
where the inequality follows from $2ab-b^2\le a^2$. Hence
\begin{align}\label{eq:general-filter-log-convex}
\ddot\AB_\theta \succeq \dot\AB_\theta \AB_\theta^{-1} \dot\AB_\theta.
\end{align}
At a fixed $t$, let $a_k>0$ be the positive eigenvalues of $\AB_\theta/n$, and write matrix entries in the corresponding eigenbasis. Differentiating gives
\begin{align*}
\frac{\dif}{\dif t}\tr\left(\Phi(\AB_\theta/n)\right) &= \frac1n \tr\left(\psi(\AB_\theta/n) \dot\AB_\theta\right), \\
\frac{\dif^2}{\dif t^2}\tr\left(\Phi(\AB_\theta/n)\right)
&= \frac1n \sum_k\psi(a_k) \ddot\AB_{\theta,kk}+ \frac{1}{n^2} \sum_k\psi'(a_k) \dot\AB_{\theta,kk}^2 + \frac{2}{n^2} \sum_{k<\ell}\frac{\psi(a_k)-\psi(a_\ell)}{a_k-a_\ell} \dot\AB_{\theta,k\ell}^2,
\end{align*}
where the divided difference is interpreted as $\psi'(a_k)$ when $a_k=a_\ell$. From \Cref{eq:general-filter-log-convex}, we can bound
\begin{align*}
\sum_k\psi(a_k) \ddot\AB_{\theta,kk}
&=\tr\left(\psi(\AB_\theta/n)\ddot\AB_\theta\right)\\
&\ge \tr\left(\psi(\AB_\theta/n)\dot\AB_\theta \AB_\theta^{-1} \dot\AB_\theta\right)\\
&= \frac1n \sum_k\frac{\psi(a_k)}{a_k} \dot\AB_{\theta,kk}^2 + \frac1n \sum_{k<\ell}\left(\frac{\psi(a_k)}{a_\ell}+\frac{\psi(a_\ell)}{a_k}\right) \dot\AB_{\theta,k\ell}^2.
\end{align*}
Combining the preceding displays and using that $\psi(z)=g(z)/z$,
\begin{align*}
&\frac{\dif^2}{\dif t^2}\tr \left(\Phi(\AB_\theta/n)\right)\\
&\ge \frac{1}{n^2} \sum_k \left(\psi'(a_k)+\frac{\psi(a_k)}{a_k}\right) \dot\AB_{\theta,kk}^2 + \frac{1}{n^2}\sum_{k<\ell}\left(2\frac{\psi(a_k)-\psi(a_\ell)}{a_k-a_\ell}+\frac{\psi(a_k)}{a_\ell}+\frac{\psi(a_\ell)}{a_k}\right) \dot\AB_{\theta,k\ell}^2\\
&= \frac{1}{n^2}\sum_k\frac{g'(a_k)}{a_k} \dot\AB_{\theta,kk}^2 + \frac{1}{n^2} \sum_{k<\ell} \frac{a_k+a_\ell}{a_ka_\ell}\frac{g(a_k)-g(a_\ell)}{a_k-a_\ell} \dot\AB_{\theta,k\ell}^2 \ge0,
\end{align*}
due to monotonicity of $g$. Thus we have shown that $F$ is convex.

Moreover, $F$ is invariant under permutation since $\zB_i$ are i.i.d. Let $\theta'$ be equal to $\theta$ with coordinates $i,j$ permuted. Since
\begin{align*}
\left.\frac{\partial F}{\partial\theta_i}\right|_{\theta_i=\log\lam_i}=\Ebb\frac{\lam_i}{n}\zB_i^\top\psi(\AB/n)\zB_i=\Ebb\eB_i^\top\GB\eB_i=p_i,
\end{align*}
convexity of $F$ implies
\begin{align*}
0&\le\la\nabla F(\theta)-\nabla F(\theta'),\theta-\theta'\ra\\
&=2(\log\lam_i-\log\lam_j)(p_i-p_j).
\end{align*}
Hence $\lam_i>\lam_j$ implies $p_i\ge p_j$, and $\lam_i=\lam_j$ implies $p_i=p_j$ by symmetry. Therefore $(p_i)_{i\ge 1}$ is nonincreasing.

For the second claim, let $\QB_s$ be the rotation of angle $s$ on the span of $\eB_i,\eB_j$ and the identity on its orthogonal complement. Denote by $\Ebb_s$ the expectation under the transformation $\XB\mapsto \XB\QB_s^\top$, equivalently $\SigmaB\mapsto\SigmaB_s := \QB_s \SigmaB\QB_s^\top$. By orthogonal equivariance,
\begin{align*}
\left.\frac{\dif}{\dif s} \eB_i^\top \QB_s (\Ebb\GB)\QB_s^\top\eB_j \right|_{s=0} &= \left.\frac{\dif}{\dif s} \Ebb_s\GB_{ij}\right|_{s=0}\\
&= \left.\frac{\dif}{\dif s} \left(\sin(s)\cos(s) (\Ebb\GB_{ii} - \Ebb\GB_{jj}) + (\cos^2(s)-\sin^2(s)) \Ebb\GB_{ij} \right)\right|_{s=0} \\
&= p_i-p_j.
\end{align*}
On the other hand, the score identity gives
\begin{align*}
\left.\frac{\dif}{\dif s}\Ebb_s\GB_{ij}\right|_{s=0}
&=
\Ebb\left[\GB_{ij}\left(-\frac n2\tr(\SigmaB^{-1}\dot\SigmaB_0)+\frac12\tr\left(\XB\SigmaB^{-1}\dot\SigmaB_0\SigmaB^{-1}\XB^\top\right)\right)\right]\\
&=
\frac12\Ebb\left[\GB_{ij}\tr\left(\XB\SigmaB^{-1}(\lam_i-\lam_j)(\eB_i\eB_j^\top+\eB_j\eB_i^\top)\SigmaB^{-1}\XB^\top\right)\right]\\
&=
\frac{\lam_i-\lam_j}{\lam_i\lam_j}\Ebb\left[\GB_{ij}(\XB\eB_i)^\top(\XB\eB_j)\right]\\
&=
\frac{\lam_i-\lam_j}{\sqrt{\lam_i\lam_j}}\Ebb \GB_{ij}\zB_i^\top\zB_j.
\end{align*}
By Cauchy--Schwarz, we thus have
\begin{align*}
(p_i-p_j)^2 \le\frac{(\lam_i-\lam_j)^2}{\lam_i\lam_j}\Ebb\GB_{ij}^2\,\Ebb(\zB_i^\top\zB_j)^2
=\frac{n(\lam_i-\lam_j)^2}{\lam_i\lam_j} \Ebb\GB_{ij}^2.
\end{align*}
We remark that this argument is essentially an application of the Cram\'{e}r--Rao bound to the unbiased estimator $(p_i-p_j)^{-1}\GB_{ij}$ for $(\lam_i-\lam_j)^{-1} (\SigmaB_s)_{ij}$, when~$s$ is unknown.
\end{proof}

\begin{lemma}\label{lem:general-filter-diffuse}
There exist constants $c_0,c_1,c_3,c_5> 1$ such that if $n\ge c_0$, $k\le n/c_3$, and
\begin{align}\label{eq:s1s2-comp}
\frac{\sum_{j>k}\lam_j^2}{n} \le \frac{1}{c_5}\left(\frac{\sum_{j>k}\lam_j}{n}\right)^2,
\end{align}
then for all $i\le k$,
\begin{align}\label{eq:general-filter-diffuse}
c_1\Ebb \BB_{ii}\ge \lam_i\min\left\{1,\frac{\sum_{j>k}\lam_j}{n\lam_i}\right\}^2.
\end{align}
\end{lemma}

\begin{proof}[Proof of \Cref{lem:general-filter-diffuse}]
If $\sum_{j>k}\lam_j=0$, the statement is trivial. Otherwise, by the small-ball concentration inequality, with probability at least $1-\exp(-n/c_0)$,
\begin{align*}
\sum_{j>k}\lambda_j \zB_j\zB_j^\top \succeq \bigg(\sum_{j>k}\lambda_j - c_1 \sqrt{n\sum_{j>k}\lambda_j^2} \bigg)\IB \succeq \frac12 \sum_{j>k}\lambda_j \IB.
\end{align*}
Under this event and $\|\zB_i\|^2\le 2n$, for all $i\le k$,
\begin{align*}
\zB_i^\top \AB_{-i}^{-1} \zB_i \le \frac{2\|\zB_i\|^2}{\sum_{j>k}\lambda_j} \le \frac{4n}{\sum_{j>k}\lambda_j}.
\end{align*}
Since $\GB\preceq\mathbf 1\{\hat\SigmaB>0\}$ and $\SigmaB\succeq\lam_i\uB_i\uB_i^\top$,
\begin{align*}
\BB_{ii} &\ge\lam_i(\eB_i^\top(\IB-\GB)\eB_i)^2 \\
&\ge\lam_i\left(1-\eB_i^\top\mathbf 1\{\hat\SigmaB>0\}\eB_i\right)^2 \\
&= \lam_i \left(1+\lam_i\zB_i^\top\AB_{-i}^{-1}\zB_i\right)^{-2},
\end{align*}
where the last equality uses the Woodbury identity. Combining the two bounds gives, when $n\ge c_0$,
\begin{align*}
c_1\Ebb \BB_{ii}\ge \lam_i\left(1+\frac{n\lam_i}{\sum_{j>k}\lam_j}\right)^{-2}\ge \frac{\lam_i}{4}\min\left\{1,\frac{\sum_{j>k}\lam_j}{n\lam_i}\right\}^2,
\end{align*}
as was to be shown.
\end{proof}

\subsection{Proof for ramp filters}

We define a monotone shrinkage filter $g$ to be `ramp-like' if there exist $0<\rho_0<\rho_1<\infty$ such that $g(\rho_0)=0$ and $g(\rho_1)=1$. In this subsection, we show a partial result for ramp-like filters.

\begin{lemma}\label{lem:ramp-filter-comparison}
There exist constants $c_0,c_1,c_2,c_3$ such that the following holds. Let $g$ be a ramp-like filter with $g(\rho_0)=0$ and $g(\rho_1)=1$. If $n\ge c_0$ and $k$ satisfies
\begin{align*}
k\le\frac{n}{c_3},\quad \lam_{k+1}+\frac1n\sum_{j>k}\lam_j\le\frac{\rho_0}{c_2},
\end{align*}
then there exists a threshold $\rho\in[\rho_0,\rho_1]$ such that
\begin{align*}
\text{with probability at least $0.99$}, \quad \Ebb[\excessRisk_{\mu}(\hat \wB^{\pcr}_{\rho}) | \XB ] \le c_1 \Ebb  \excessRisk_{\mu}(\hat \wB_g).
\end{align*}
\end{lemma}

We first require the following concentration bounds.

\begin{lemma}\label{lem:ramp-concentration}
In the setting of \Cref{lem:ramp-filter-comparison}, it holds with probability at least $1-\exp(-n/c_0)$ that
\begin{align}\label{eq:ramp-filter-covariance}
\frac12(\SigmaB+\rho_0\IB)\preceq\hat\SigmaB+\rho_0\IB\preceq\frac32(\SigmaB+\rho_0\IB),
\end{align}
and simultaneously for every $i$,
\begin{align}\label{eq:ramp-filter-square}
\AB_{-i}^2\preceq c(n\VB_{-i}+n^2\alpha_k\IB).
\end{align}
\end{lemma}

\begin{proof}[Proof of \Cref{lem:ramp-concentration}]
Define the operators
\begin{align*}
\TB &:= (\SigmaB+\rho_0\IB)^{-1/2} \SigmaB (\SigmaB+\rho_0\IB)^{-1/2}, \\
\hat\TB &:=(\SigmaB+\rho_0\IB)^{-1/2} \hat\SigmaB (\SigmaB+\rho_0\IB)^{-1/2} = \frac1n \sum_{i=1}^n \yB_i\yB_i^\top, \quad \yB_i := (\SigmaB+\rho_0\IB)^{-1/2} \xB_i \sim N(0,\TB).
\end{align*}
It holds that $\|\TB\| \le 1$ and
\begin{align*}
\tr(\TB) &=\sum_i\frac{\lam_i}{\lam_i+\rho_0}\le k+ \sum_{j>k} \frac{\lam_j}{\rho_0} \le\frac{n}{c_3} + \frac{n}{c_2}.
\end{align*}
Then $\|\hat\TB-\TB\|\le 1/2$ by standard concentration, and
\begin{align*}
(\SigmaB+\rho_0\IB)^{-1/2} (\hat\SigmaB+\rho_0\IB) (\SigmaB+\rho_0\IB)^{-1/2} = \hat\TB + \rho_0(\SigmaB+\rho_0\IB)^{-1} = \IB + \hat\TB-\TB.
\end{align*}
The first claim follows. For the second, define the sets of indices
\begin{align*}
H &:= \{j:j\le k\} \cup \{j:j>k, \; \lam_j>\sqrt{\alpha_k}\},\\
T &:= \{j: \lam_j>0, \; j\notin H\}.
\end{align*}
We have $|H| \le n/c_3 + n$, so $\|[\zB_h]_{h\in H}\|^2 \le cn$ with probability at least $1-\exp(-n/c_0)$ for constants $c,c_0$. Then
\begin{align*}
\left(\sum_{j\in H, j\ne i}\lam_j\zB_j\zB_j^\top\right)^2\preceq \|[\zB_h]_{h\in H}\|^2 \sum_{j\in H, j\ne i}\lam_j^2\zB_j\zB_j^\top \preceq cn \sum_{j\in H, j\ne i}\lam_j^2\zB_j\zB_j^\top.
\end{align*}
For the remaining indices, the weighted covariance bound gives
\begin{align*}
\left\|\sum_{j\in T, j\ne i}\lam_j\zB_j\zB_j^\top\right\|\le \left\|\sum_{j\in T}\lam_j\zB_j\zB_j^\top\right\|\le c\left(\sum_{j>k}\lam_j + n\sqrt{\alpha_k}\right)\le 2cn\sqrt{\alpha_k},
\end{align*}
where the second inequality again holds with probability at least $1-\exp(-n/c_0)$. Combining the two bounds yields \Cref{eq:ramp-filter-square}, as desired.
\end{proof}

\begin{proof}[Proof of \Cref{lem:ramp-filter-comparison}]
We first show that for all indices $i$,
\begin{align}\label{eq:ebii}
\Ebb \BB_{ii}\ge c\min\left\{\lam_i,\frac{\alpha_k}{\lam_i}\right\}.
\end{align}
Here, the right-hand side is understood as zero if $\lam_i=0$. Note the decomposition
\begin{align*}
\BB_{ii} = \lam_i (1-\GB_{ii})^2 + \sum_{j\ne i} \lam_j \GB_{ij}^2.
\end{align*}
Since $g$ vanishes on $[0,\rho_0]$ and $0\le g\le1$, one has $g(z)\le z/\rho_0$, thus
\begin{align*}
p_i\le\frac1{\rho_0}\Ebb\eB_i^\top\hat\SigmaB\eB_i=\frac{\lam_i}{\rho_0}.
\end{align*}
If $p_i\le1/2$ (in particular if $i>k$), then $\Ebb \BB_{ii}\ge \lam_i(1-p_i)^2\ge \lam_i/4$ is immediate. Otherwise $i\le k$ and $\lam_i>\rho_0/2$, while for $j>k$ one has $p_j\le1/c_2$ and $\lam_j\le \lam_{k+1} \le \rho_0/c_2$. Choosing $c_2$ to be large, by \Cref{lem:general-filter-coordinates},
\begin{align*}
\Ebb \BB_{ii}\ge\sum_{j>k}\lam_j\Ebb \GB_{ij}^2\ge\frac{c}{n\lam_i}\sum_{j>k}\lam_j^2.
\end{align*}
If \Cref{eq:s1s2-comp} holds, $\Ebb \BB_{ii}\ge c\alpha_k/\lam_i$ then follows from \Cref{lem:general-filter-diffuse}; otherwise, this follows directly. Thus \Cref{eq:ebii} is established.

Next, let $\mathcal E$ be the intersection of the event in \Cref{lem:ramp-concentration} and $\max_{i\le k}\|z_i\|^2 \le 2n$, so $\Pr\mathcal E \ge 1-\exp(-n/c_0)$. Under~$\mathcal E$, for $i\le k$,
\begin{align*}
\|(\IB-\GB)\eB_i\|_{\hat\SigmaB}^2 &= \frac{\lam_i}{n}\|(\IB-g(\AB/n))\zB_i\|^2 \\
&= \frac{\lam_i}{n}\|(1-\GB_{ii})\zB_i-\AB_{-i}\PsiB\zB_i\|^2 \\
&\le2\lam_i(1-\GB_{ii})^2\frac{\|\zB_i\|^2}{n}+\frac{2\lam_i}{n}\zB_i^\top\PsiB\AB_{-i}^2\PsiB\zB_i \\
&\le 2\lam_i(1-\GB_{ii})^2\frac{\|\zB_i\|^2}{n} + 2c\left(\sum_{j\ne i}\lam_j\GB_{ji}^2 + n\alpha_k \lam_i\zB_i^\top\PsiB^2\zB_i\right). && \explain{by \Cref{eq:ramp-filter-square}}
\end{align*}
Here, we have used
\begin{align*}
\lam_i\zB_i^\top\PsiB\VB_{-i}\PsiB\zB_i
&=\sum_{j\ne i}\lam_i\lam_j^2(\zB_j^\top\PsiB\zB_i)^2
=\sum_{j\ne i}\lam_j\GB_{ji}^2.
\end{align*}
Because $g$ vanishes on $[0,\rho_0]$,
\begin{align*}
\lam_i\zB_i^\top\PsiB^2\zB_i
=\frac1n\eB_i^\top\hat\SigmaB^{-1}\GB^2\eB_i &\le\frac1n\eB_i^\top\hat\SigmaB^{-1}\mathbf 1\{\hat\SigmaB>\rho_0\IB\}\eB_i\\
&\le\frac2n\eB_i^\top(\hat\SigmaB+\rho_0\IB)^{-1}\eB_i\\
&\le\frac4n\eB_i^\top(\SigmaB+\rho_0\IB)^{-1}\eB_i
\le \frac{4}{n\lam_i}. && \explain{by \Cref{eq:ramp-filter-covariance}}
\end{align*}
Also, $\alpha_k\le2\rho_0^2/c_2^2$ by definition of~$k$ and $\|\PsiB\|\le1/(n\rho_0)$, so conditional on $\max_{i\le k}\|z_i\|^2\le 2n$,
\begin{align*}
\lam_i\zB_i^\top\PsiB^2\zB_i\le \lam_i\|\PsiB\|^2\|\zB_i\|^2 \le \frac{4\lam_i}{c_2^2 n\alpha_k}.
\end{align*}
Hence we have shown that with probability at least $1-\exp(-n/c_0)$, for some constant $c'$,
\begin{align*}
n\alpha_k\lam_i\zB_i^\top\PsiB^2\zB_i \le c\min\left\{\lam_i,\frac{\alpha_k}{\lam_i}\right\} \le c'\Ebb\BB_{ii}.
\end{align*}
Moreover on the same event,
\begin{align*}
2\lam_i(1-\GB_{ii})^2\frac{\|\zB_i\|^2}{n}
+2c\sum_{j\ne i}\lam_j\GB_{ji}^2
&\le c\left(\lam_i(1-\GB_{ii})^2+\sum_{j\ne i}\lam_j\GB_{ji}^2\right) =c\BB_{ii}.
\end{align*}
Finally for $i>k$, the assumptions give $p_i\le \lam_i/\rho_0\le 1/c_2$, and we can directly bound
\begin{align*}
\Ebb\|(\IB-\GB)\eB_i\|_{\hat\SigmaB}^2
\le\Ebb\eB_i^\top\hat\SigmaB\eB_i
=\lam_i \le c\lam_i(1-p_i)^2
\le c'\Ebb\BB_{ii}.
\end{align*}
Therefore, we have
\begin{align}\label{eq:kind}
\Ebb\left[\mathbf 1_{\mathcal E}\sum_i\wB_i^{*2}\|(\IB-\GB)\eB_i\|_{\hat\SigmaB}^2\right]\le c\sum_i\wB_i^{*2}\Ebb\BB_{ii}=c\Ebb\bias(\hat\wB_g).
\end{align}
Now define
\begin{align*}
\rho:=\inf\{z\ge0:g(z)>1/2\},
\quad
\PB:=\mathbf 1\{\GB>\IB/2\}.
\end{align*}
By monotonicity, $\rho\in[\rho_0,\rho_1]$ and
$\mathbf 1\{g(z)>1/2\}=\mathbf 1\{z>\rho\}$ for every $z\ne\rho$.
Since $\rho>0$ is deterministic, under Gaussian design, it is almost surely not an empirical
eigenvalue. Thus~$\PB$ is equal to the PCR projection at
threshold~$\rho$ almost surely. 

For every $z\in[0,1]$, we have
\begin{align*}
(\mathbf 1\{z>1/2\}-z)^2\le(1-z)^2 &\quad\implies\quad (\PB-\GB)^2\preceq(\IB-\GB)^2, \\
\mathbf 1\{z>1/2\}\le4z^2 &\quad\implies\quad \PB\preceq4\GB^2.
\end{align*}
Since $\PB,\GB,\hat\SigmaB$ commute and
$\ran(\PB-\GB)\subseteq\ran\mathbf 1\{\hat\SigmaB>\rho_0\IB\}$,
it follows that on $\mathcal E$,
\begin{align*}
\|(\PB-\GB)\eB_i\|_{\SigmaB}^2
&\le
2\|(\PB-\GB)\eB_i\|_{\hat\SigmaB}^2
+\rho_0\|(\PB-\GB)\eB_i\|^2
&& \explain{by \Cref{eq:ramp-filter-covariance}}\\
&\le3\|(\PB-\GB)\eB_i\|_{\hat\SigmaB}^2\\
&\le3\|(\IB-\GB)\eB_i\|_{\hat\SigmaB}^2.
\end{align*}
The event $\mathcal E$ is invariant under coordinate sign changes,
so by symmetry and \Cref{eq:kind},
\begin{align*}
\Ebb\left[\mathbf 1_{\mathcal E}
\|(\PB-\GB)\wB^*\|_{\SigmaB}^2\right]
&=
\sum_i\wB_i^{*2}\Ebb\left[\mathbf 1_{\mathcal E}
\|(\PB-\GB)\eB_i\|_{\SigmaB}^2\right]\\
&\le
3\Ebb\left[\mathbf 1_{\mathcal E}
\sum_i\wB_i^{*2}\|(\IB-\GB)\eB_i\|_{\hat\SigmaB}^2\right]\\
&\le c\Ebb\bias(\hat\wB_g).
\end{align*}
Therefore,
\begin{align*}
\Ebb\left[\mathbf 1_{\mathcal E}
\|(\IB-\PB)\wB^*\|_{\SigmaB}^2\right]
&\le
2\Ebb\left[\mathbf 1_{\mathcal E}
\|(\IB-\GB)\wB^*\|_{\SigmaB}^2\right]
+2\Ebb\left[\mathbf 1_{\mathcal E}
\|(\PB-\GB)\wB^*\|_{\SigmaB}^2\right] \le c\Ebb\bias(\hat\wB_g).
\end{align*}
Moreover, $\PB\preceq4\GB^2$ gives, conditional on $\XB$,
\begin{align*}
\variance(\hat\wB_\rho^{\pcr})
&=
\frac{\sigma^2}{n}\tr(\SigmaB\hat\SigmaB^{-1}\PB)\\
&\le
4\frac{\sigma^2}{n}\tr(\SigmaB\hat\SigmaB^{-1}\GB^2)
=4\variance(\hat\wB_g).
\end{align*}
Combining the bias and variance bounds, we have shown that
\begin{align*}
\Ebb\left[\mathbf 1_{\mathcal E}
\Ebb[\excessRisk_\mu(\hat\wB_\rho^{\pcr})\mid\XB]\right]
\le c\Ebb\excessRisk_\mu(\hat\wB_g).
\end{align*}
We may assume $\Pr(\mathcal E^c)\le0.005$.
If $\Ebb\excessRisk_\mu(\hat\wB_g)=0$, the preceding inequality implies
that PCR risk vanishes almost surely on $\mathcal E$, proving
the claim. Otherwise, taking $c_1\ge200c$, Markov's inequality gives
\begin{align*}
\Pr\left(
\Ebb[\excessRisk_\mu(\hat\wB_\rho^{\pcr})\mid\XB]
>c_1\Ebb\excessRisk_\mu(\hat\wB_g)
\right)
\le\Pr(\mathcal E^c)+\frac{c}{c_1}\le0.01.
\end{align*}
This concludes the proof.
\end{proof}

\subsection{Proof of Theorem \ref{thm:general-comparison}}

\begin{proof}[Proof of \Cref{thm:general-comparison}]
Define the $v$-quantile of~$g$ for $v\in[0,1]$ as
\begin{align*}
q_v := \inf\{z\ge 0: g(z) \ge v\} \in [0,\infty].
\end{align*}
Suppose $q_{1/3}=0$. Since $g(z)\ge 1/3$ for all $z>0$, we may decompose $g=(1/3)g_1 + (2/3)g_2$ where
\begin{align*}
g_1(z)=1\{z>0\}, \quad g_2(z) = \frac{3g(z)-1\{z>0\}}{2}.
\end{align*}
Then $g_1,g_2$ are monotone shrinkage and in particular $g_1$ is OLS, thus equivalent to PCR with threshold~$0$. Thus by Markov's inequality and \Cref{lem:glue}, with probability at least $0.99$,
\begin{align*}
\Ebb\excessRisk_{\mu}(\hat \wB_g) \ge \frac19 \Ebb\excessRisk_{\mu}(\hat \wB_{g_1}) \quad \text{and}\quad \inf_{\rho\ge 0} \Ebb[\excessRisk_{\mu}(\hat \wB^{\pcr}_{\rho})|\XB] \le \Ebb[ \excessRisk_{\mu}(\hat \wB_{g_1})|\XB] \le 100 \Ebb\excessRisk_{\mu}(\hat \wB_{g_1}),
\end{align*}
so the theorem is proved. Next, if $q_{2/3}=\infty$, it holds that $\GB\preceq \frac23\IB$ so that
\begin{align*}
\Ebb\bias(\hat\wB_g) &= \sum_i \wB_i^{*2} \Ebb \eB_i^\top (\IB-\GB) \SigmaB (\IB-\GB) \eB_i \\
&\ge \sum_i \lam_i\wB_i^{*2} \Ebb(\eB_i^\top(\IB-\GB)\eB_i)^2 \ge \frac19 \|\wB^*\|_{\SigmaB}^2.
\end{align*}
Since the zero estimator (obtained by taking $\rho\to\infty$) always achieves constant risk $\|\wB^*\|_{\SigmaB}^2$, the theorem is proved.

We may thus suppose $0<q_{1/3}\le q_{2/3}<\infty$. Assume $n\ge c_0$ and set
\begin{align*}
k:=\frac{n}{c_3},\quad \rho:=c_3\left(\lam_{k+1}+\frac1n\sum_{i>k}\lam_i\right).
\end{align*}
If $\rho\le q_{1/3}$, decompose $g=(g_1+g_2+g_3)/3$ where
\begin{align*}
g_1(z)&= \min\{3g(z),1\}, \quad g_2(z)= \min\{\max\{3g(z)-1,0\},1\}, \quad g_3(z)= \max\{3g(z)-2,0\}.
\end{align*}
It must hold that $g_2(0.5q_{1/3})=0$ and $g_2(2q_{2/3})=1$, so $g_2$ is a ramp filter and
\begin{align*}
\Ebb\excessRisk_{\mu}(\hat \wB_g) \ge \frac19 \Ebb\excessRisk_{\mu}(\hat \wB_{g_2})
\end{align*}
by \Cref{lem:glue}. The theorem then follows from \Cref{lem:ramp-filter-comparison}.

It remains to consider the case $\rho>q_{1/3}$. We directly compare with the PCR upper bound in \Cref{thm:pcr-risk} with threshold $\rho$. For the variance, since $g(z)\ge 1/3$ for every $z>\rho$, we have conditional on $\XB$,
\begin{align*}
\variance (\hat \wB^{\pcr}_\rho) &= \frac{\sigma^2}{n}\tr(\SigmaB \hat\SigmaB_{>\rho}^{-1}) \le9\frac{\sigma^2}{n}\tr(\SigmaB\hat\SigmaB^{-1}\GB^2).
\end{align*}
Hence by Markov's inequality,
\begin{align*}
\text{with probability at least $0.995$}, \quad \variance (\hat \wB^{\pcr}_\rho) \le c\Ebb\variance (\hat \wB_g).
\end{align*}

To compare the bias, let $k^*$ be the critical index in \Cref{thm:pcr-risk}, so that $k^*\le k$, $\lam_{k^*}>1.1\rho$ and
\begin{align*}
\frac1n\sum_{i>k^*}\lam_i&\le\frac{k-k^*}{n}\lam_{k^*+1}+\frac1n\sum_{i>k}\lam_i \le\frac1{c_3}\left(1.1\rho+\frac1{c_2n}\sum_{i>k^*}\lam_i\right)+\frac{\rho}{c_3}.
\end{align*}
Choosing $c_2,c_3$ sufficiently large, we ensure that
\begin{align*}
\frac1n\sum_{i>k^*}\lam_i\le\frac{3\rho}{c_3},\quad \lam_{k^*+1}\le 2\rho,\quad \frac1n\sum_{i>k^*}\lam_i^2\le\frac{6\rho^2}{c_3}.
\end{align*}
Then with probability at least $0.995$, the bias bound in \Cref{thm:pcr-risk} gives
\begin{align*}
  \frac{1}{c_1} \bias(\hat\wB^{\pcr}_\rho) &\le \Bigg( \frac{\sum_{i>k^*}\lambda_i}{n} + \sqrt{\frac{\sum_{i>k^*}\lambda_i^2}{n}} \Bigg)^2 \|\wB^*\|^2_{\SigmaB_{0:k^*}^{-1}} + \|\wB^*\|^2_{\SigmaB_{k^*:\infty}} \\
  &\le c\rho^2 \|\wB^*\|^2_{\SigmaB_{0:k^*}^{-1}} + \|\wB^*\|^2_{\SigmaB_{k^*:\infty}}.
\end{align*}
We proceed to show a matching bias lower bound for $\hat\wB_g$ in this regime. Define the set of `intermediate' indices
\begin{align*}
J:=\{j:k^*<j\le k,\ p_j\le 1/4\}.
\end{align*}
By a similar concentration argument as in \Cref{lem:ramp-concentration}, with probability at least $1-\exp(-n/c_0)$,
\begin{align*}
\hat\SigmaB+\rho\IB\preceq\frac32(\SigmaB+\rho\IB), \quad \hat\SigmaB_{\le k^*}\succeq 0.99\SigmaB_{\le k^*}\succeq \rho\IB.
\end{align*}
It follows that $\SigmaB\succeq \frac13 \hat\SigmaB$ on the range of $\mathbf 1\{\hat\SigmaB>\rho\IB\}$. Since $\rho>q_{1/3}$, we also have $\GB^2\succeq \frac19 \mathbf 1\{\hat\SigmaB>\rho\IB\}$, so the variance of $\hat\wB_g$ is lower bounded as
\begin{align*}
\variance(\hat\wB_g) &\ge\frac{\sigma^2}{9n}\tr\left(\SigmaB\hat\SigmaB^{-1}\mathbf 1\{\hat\SigmaB>\rho\IB\}\right)\\
&\ge\frac{\sigma^2}{27n}\tr\left(\mathbf 1\{\hat\SigmaB>\rho\IB\}\right)\\
&\ge ck^*\frac{\sigma^2}{n}.
\end{align*}
Moreover by Cauchy--Schwarz,
\begin{align*}
p_i^2 = \left(\Ebb \la\hat\SigmaB^{1/2}\eB_i,\hat\SigmaB^{-1/2}\GB\eB_i\ra\right)^2 \le\Ebb\eB_i^\top\hat\SigmaB\eB_i\,\Ebb\eB_i^\top\hat\SigmaB^{-1}\GB^2\eB_i=\lam_i\Ebb\eB_i^\top\hat\SigmaB^{-1}\GB^2\eB_i.
\end{align*}
Summing over $i$ gives
\begin{align*}
\Ebb \variance(\hat\wB_g) \ge \frac{\sigma^2}{n}\sum_i p_i^2.
\end{align*}
Therefore from the two bounds,
\begin{align*}
\Ebb\excessRisk_\mu(\hat\wB_g) \ge ck^*\frac{\sigma^2}{n} \quad\text{and}\quad \Ebb\excessRisk_\mu(\hat\wB_g) \ge \frac{\sigma^2}{16n} \#\{i:p_i>1/4\}.
\end{align*}
On the other hand, we may assume $\Ebb\excessRisk_\mu(\hat\wB_g) \le \sigma^2/c$ for any constant $c$, as otherwise the zero estimator achieves risk $\|\wB^*\|_{\SigmaB}^2 \le b\sigma^2$ and the theorem is proved. Choosing $c$ sufficiently large, this and $k=n/c_3$ imply
\begin{align*}
k^* \le k/8, \quad \#\{i:p_i>1/4\} \le k/8 \quad\implies\quad |J| \ge k/2.
\end{align*}
In particular, since $p_i$ is nonincreasing by \Cref{lem:general-filter-coordinates}, it follows that $p_j\le1/4$ for every $j>k$.

Next, arguing as in \Cref{eq:ebii}, we show that for all~$i$,
\begin{align}\label{eq:ebii-2}
\Ebb\BB_{ii} \ge c\min\left\{\lam_i, \frac{\rho^2}{\lam_i}\right\}.
\end{align}
Indeed, if $p_i\le 1/2$ (in particular if $i>k$), then $\Ebb \BB_{ii}\ge c\lam_i$. Suppose $p_i>1/2$. Then for every $j\in J$ or $j>k$, it holds that $i<j$ and $p_j\le1/4$, so \Cref{lem:general-filter-coordinates} gives
\begin{align}\label{eq:general-comparison-pair}
\lam_j\Ebb \GB_{ji}^2\ge \frac{\lam_j^2}{16n\lam_i}.
\end{align}
If $\lam_{k+1}\ge\rho/(2c_3)$, then summing \Cref{eq:general-comparison-pair} over $j\in J$ gives
\begin{align*}
\Ebb\BB_{ii} \ge \frac{|J|}{16n} \frac{ \lam_{k+1}^2}{\lam_i} \ge c\frac{\rho^2}{\lam_i}.
\end{align*}
Otherwise $n^{-1}\sum_{j>k}\lam_j\ge\rho/(2c_3)$, and $\Ebb\BB_{ii}\ge c\rho^2/\lam_i$ follows as before from \Cref{lem:general-filter-diffuse} if \Cref{eq:s1s2-comp} holds, and summing \Cref{eq:general-comparison-pair} over $j>k$ otherwise.

Now from \Cref{eq:ebii-2}, for indices $i\le k^*$, $\lam_i>\rho$ implies $\Ebb\BB_{ii}\ge c\rho^2/\lam_i$; for indices $i>k^*$, $\lam_i \le \lam_{k^*+1}\le c\rho$ implies $\Ebb\BB_{ii}\ge c\lam_i$. It follows that
\begin{align*}
\rho^2 \|\wB^*\|^2_{\SigmaB_{0:k^*}^{-1}} + \|\wB^*\|_{\SigmaB_{k^*:\infty}}^2 &= \sum_{i\le k^*} \frac{\rho^2}{\lam_i} \wB_i^{*2} + \sum_{i>k^*} \lam_i \wB_i^{*2} \\
&\le c\sum_i \wB_i^{*2}\Ebb\BB_{ii} = c\Ebb\bias(\hat\wB_g),
\end{align*}
Hence with probability at least 0.99, the PCR risk is at most a constant multiple of the expected risk of~$g$, as was to be shown.
\end{proof}

\subsection{Proof of Theorem \ref{thm:pcr-sd}}

\begin{proof}[Proof of \Cref{thm:pcr-sd}]
We construct a series of instances $(\SigmaB_n, \wB_n^*)$, dependent only on $\eps$, such that for each instance, there exists a PCR filter that is polynomially better (w.r.t. $n$) than any $g\in\mathcal G_{\eps,\delta}$. We may suppose $\eps\le 1$, as otherwise $G_{\eps,\delta} \subset G_{1,\delta}$. Choose $\alpha>0$ such that $(1+\alpha)^4 = 1+\eps$. For every sample size $n$, let
\begin{align*}
    \SigmaB_n
    :=\diag\left(\frac12,(1+\alpha)^2s_n,s_n\IB_{m_n}\right),\quad\wB_n^*
    :=\sqrt{2\left(1-\frac{1}{m_n}\right)}\eB_1
      +\frac{1}{1+\alpha}\sqrt{\frac{1}{s_n m_n}}\eB_2, \quad\sigma_n:=1,
\end{align*}
where
\begin{align*}
m_n:=\lfloor n^{1/2}\rfloor, \quad s_n := \frac{1}{2((1+\alpha)^2+m_n)}.
\end{align*}
It holds that $\tr(\SigmaB_n)=1$ and $\|\wB_n^*\|_{\SigmaB_n}^2=1$ (if $b<1$, $\sigma_n$ can be scaled accordingly). We claim that
\begin{itemize}
\item with $\rho_n:=(1+\alpha)s_n$, with probability at least $0.99$, it holds that $\Ebb[\excessRisk(\hat\wB_{\rho_n}^{\pcr})|\XB]=\bigO(n^{-1})$;
\item for any $g\in\mathcal G_{\eps,\delta}$, it holds uniformly in~$g$ that $\Ebb\excessRisk(\hat\wB_g)=\Omega(n^{-1/2})$.
\end{itemize}

First, we upper bound the risk of PCR at threshold~$\rho_n$. We apply the strengthened version of \Cref{thm:pcr-risk} given in \Cref{sec:pcr-proof} with slack~$\alpha/2$. Denote $S_1(k):=\sum_{i>k}\lam_i$. The selected index is
\begin{align*}
    k^*
    &:=\min\left\{k:(1+\alpha/2)\rho_n+\frac{S_1(k)}{c_2 n}
    \ge {\lambda_{k+1}}\right\}.
\end{align*}
We claim that $k^*=2$. Indeed, for sufficiently large $n$,
\begin{align*}
&(1+\alpha/2)\rho_n+\frac{S_1(0)}{c_2 n} <\frac12 = \lam_1,\\
&(1+\alpha/2)\rho_n+\frac{S_1(1)}{c_2 n} <(1+\alpha)^2s_n=\lambda_2,
\end{align*}
whereas
\begin{align*}
    (1+\alpha/2)\rho_n+\frac{S_1(2)}{c_2 n} > \rho_n >s_n = \lambda_3.
\end{align*}
Moreover since $m_n=\lfloor n^{1/2}\rfloor$, by standard matrix concentration, with probability at least $1-\exp(-\alpha^2 n/c_0)$,
\begin{align}
    \frac{1}{1+\alpha/2}\SigmaB_n
    \preceq
    \hat\SigmaB_n
    \preceq
    \left(1+\frac\alpha2\right)\SigmaB_n.    \label{eq:nonstep-relative-covariance}
\end{align}
In particular, for every $j\ge 3$,
\begin{align*}
    \frac{s_n}{1+\alpha/2}
    \le
    \hat\lambda_j
    \le
    \left(1+\frac\alpha2\right)s_n<\rho_n,
\end{align*}
while for $j\le 2$,
\begin{align*}
    \hat\lambda_j\ge \frac{\lambda_2}{1+\alpha/2}=\frac{(1+\alpha)^2s_n}{1+\alpha/2} > \rho_n,
\end{align*}
so exactly two empirical eigenvalues exceed $\rho_n$, i.e., $\hat{k}=2$. By \Cref{thm:pcr-risk}, with probability at least $0.99$,
\begin{align*}
    \Ebb[\excessRisk_{\mu_n} (\hat\wB_{\rho_n}^{\pcr})|\XB]
    &\le c_1
    \frac{m_ns_n^2}{\alpha^2 n}
    \left(
        4\left(1-\frac{1}{m_n}\right)
        +\frac{1}{(1+\alpha)^4s_n^2 m_n}
    \right)
    +\frac2n \le
    \frac{c}{\alpha^2 n}.
\end{align*}
Hence the PCR risk is $\bigO(1/n)$.

Now fix any $g\in\mathcal G_{\eps,\delta}$ and define
\begin{align*}
    a_g:=\sup\{s:g(s)\le \delta\},
    \qquad
    b_g:=\inf\{s:g(s)\ge 1-\delta\}.
\end{align*}
By definition, $b_g/a_g \ge 1+\eps$. We show that the risk of $g$ is uniformly lower bounded by $\Omega(n^{-1/2})$, where constants with polynomial dependence on $\eps,\delta$ are hidden. On the event in \Cref{eq:nonstep-relative-covariance}, we have
\begin{align*}
    \variance(\hat\wB_g)
    &=\frac1n\sum_{j=1}^{m_n+2}
    \frac{\hat u_j^\top\SigmaB_n\hat u_j}{\hat\lambda_j}
    g(\hat\lam_j)^2
    \ge
    \frac{m_n}{(1+\alpha/2)n}
    g\left(
        \frac{s_n}{1+\alpha/2}
    \right)^2.
\end{align*}
For the bias, \Cref{lem:effbias} with $\tau=1+\alpha/2$ yields
\begin{align*}
\Ebb \bias(\hat\wB_g)
    &\ge c\alpha^2\sum_{i}\lambda_i (1-g((1+\alpha/2)\lambda_i))^2 \wB_i^{*2} \\
    &\ge \frac{c\alpha^2}{m_n}
    \left(1-g((1+\alpha/2)(1+\alpha)^2s_n)
    \right)^2.
\end{align*}
If $a_g=0$ or $s_n> (1+\alpha/2)a_g$, then
\begin{align*}
g\left(\frac{s_n}{1+\alpha/2}\right)
    \ge \delta,
\end{align*}
and $\variance(\hat\wB_g) = \Omega(n^{-1/2})$. Otherwise,
\begin{align*}
(1+\alpha/2)(1+\alpha)^2s_n \le (1+\alpha/2)^2(1+\alpha)^2 a_g < (1+\eps) a_g \le b_g,
\end{align*}
and
\begin{align*}
\Ebb \bias(\hat\wB_g) > \frac{c\alpha^2\delta^2}{m_n} = \Omega(n^{-1/2}).
\end{align*}
The proof is complete.
\end{proof}

\begin{proof}[Proof of \Cref{lem:non-step-examples}]
Since all considered filters are continuous, it suffices to show that
\begin{align*}
\frac{\inf\{z: g(z)\ge 1-\delta\}}{\sup\{z: g(z)\le \delta\}} = \frac{g^{-1}(1-\delta)}{g^{-1}(\delta)} \ge 2
\end{align*}
for each family and stated value of~$\delta$, which we verify below.

\textit{Gradient descent.}~
The following ratio is increasing in~$t$, hence taking $t=1$ yields the lower bound
\begin{align*}
\frac{g^{-1}(2/3)}{g^{-1}(1/3)} = \frac{1-(1/3)^{1/t}}{1-(2/3)^{1/t}} \ge 2.
\end{align*}

\textit{Iterated Tikhonov.}~
The following ratio is decreasing in~$p$, hence taking $p\to\infty$,
\begin{align*}
\frac{g^{-1}(2/3)}{g^{-1}(1/3)} = \frac{3^{1/p}-1}{(3/2)^{1/p}-1} \ge \frac{\log 3}{\log (3/2)} \approx 2.71.
\end{align*}

\textit{Ridge PCR.}~
The ratio is always equal to~$2$.

\textit{Power-exponential filter.}~
For $\delta = 1/(2^p+1)$, we have
\begin{align*}
\frac{g^{-1}(1-\delta)}{g^{-1}(\delta)} = \left(\frac{\log(2^p+1)}{\log(2^{-p}+1)}\right)^{1/p} \ge \left(\frac{1}{2^{-p}}\right)^{1/p} =2.
\end{align*}
The proof is complete.
\end{proof}

\section{General Risk Upper Bounds}\label{sec:master-ub}

\subsection{Background on Schur multipliers}

\paragraph{Matrix operator.}
A matrix operator is a linear map between matrices, i.e., a fourth-order tensor.
We use $\circ$ to denote composition of matrix operators, or a matrix operator applied to a matrix, and $\otimes$ for the Kronecker product. 
For matrices $\AB, \BB$, and $\XB$ of appropriate shape, we follow the convention that
\begin{align*}
(\AB\otimes\BB)\circ \XB := \BB\XB\AB^\top. 
\end{align*}
For matrices $\AB, \BB, \CB$, and $\DB$ of appropriate shape, we have 
\begin{align*}
(\AB\otimes \BB)\circ (\CB\otimes \DB) = (\AB\CB)\otimes (\BB\DB).
\end{align*}
We use $\diag$ to convert a vector to a diagonal matrix, and a matrix to a vector of its diagonal entries, i.e., for a vector $\xB$ and a squared matrix $\AB$,
\begin{align*}
    \diag(\xB) := \begin{bmatrix}
        \ddots & & \\
        & \xB_i & \\
        & & \ddots
    \end{bmatrix},\quad 
    \diag(\AB) := \begin{bmatrix}
    \vdots \\ 
        \AB_{ii} \\
        \vdots 
    \end{bmatrix}.
\end{align*}
We use $\odot$ for the Schur or Hadamard product, i.e., entry-wise multiplication; for matrices $\AB, \BB$ and vector~$\xB$ of appropriate shape, we have
\begin{align*}
(\AB \odot \BB) \xB = \diag\big( \AB \diag(\xB) \BB^\top \big).
\end{align*}
A matrix operator $\Tcal$ is \emph{self-adjoint} if $\la \Tcal\circ \XB, \YB\ra = \la \XB, \Tcal\circ\YB\ra$ for all $\XB$ and $\YB$ of appropriate shape.

\begin{remark}[Schur multiplier as a spectral operator]
For symmetric matrices $\AB,\BB$, the left and right matrix multiplication operators, 
\(
\IB\otimes \AB
\),
and \(\BB^\top \otimes \IB\),
are both self-adjoint; additionally, they commute, that is,
\begin{align*}
\IB\otimes \AB \circ \BB^\top \otimes \IB = \BB^\top \otimes \IB \circ \IB\otimes \AB = \BB^\top \otimes \AB.
\end{align*}
Hence they can be simultaneously diagonalized. Specifically, we have
\begin{align*}
\IB\otimes \AB = \sum_{i,j}\lambda_i \big(\vB_j\vB_j^\top \big)\otimes \big(\uB_i\uB_i^\top \big),\quad 
\BB^\top \otimes \IB = \sum_{i,j}\gamma_{j} \big(\vB_j\vB_j^\top \big)\otimes \big(\uB_i\uB_i^\top \big).
\end{align*}
So the Schur multiplier $k(\AB, \BB)$ can be understood as applying a scalar kernel function $k:\Rbb\otimes \Rbb \to \Rbb$ to left-right matrix multiplication operators $(\IB\otimes \AB, \BB^\top \otimes \IB)$ with operator eigenvalues $(\lambda_i, \gamma_j)_{i,j}$.
Therefore, the Schur multiplier $k(\AB, \BB)$ is sometimes known as a \emph{double operator integral} in the literature.
\end{remark}

The Schur multiplier has the following equivalent forms:
\begin{align*}
    k(\AB, \BB) \circ \XB 
    &=  \sum_{i,j}k(\lambda_i, \gamma_j) \uB_i\uB_i^\top  \XB \vB_j\vB_j^\top\\ 
\iff \quad \UB^\top \big(k(\AB,\BB)\circ \XB\big) \VB & =   \KB  \odot (\UB^\top \XB \VB) ,\ \ \text{where}\ \ \KB := [k(\lambda_i,\gamma_j)]_{i,j} \in\Rbb^{n\times m}.
\end{align*}
So the canonical Schur multiplier $k(\AB, \BB)$ in a matrix space is \emph{equivariant} to the standard Schur multiplier associated with the spectral kernel matrix $\KB$ in the \emph{rotated matrix space}, i.e., 
\begin{equation*}
\begin{aligned}
k(\AB,\BB): \Rbb^{n\times m} &\to \Rbb^{n\times m} \\
 \XB  &\mapsto  k(\AB, \BB)\circ \XB
\end{aligned}
\quad  \iff\quad 
\begin{aligned}
\Scal_{\KB}:   \UB^\top   \big( \Rbb^{n\times m}\big)  \VB  &\to \UB^\top \big( \Rbb^{n\times m} \big) \VB \\ 
    \ZB &\mapsto \KB\odot \ZB.
\end{aligned}
\end{equation*}

\begin{remark}[Schur multiplier as a diagonal operator]
Viewing the canonical Schur multiplier as a standard Schur multiplier in the rotated space and vectorizing, we have
\begin{align*}
  \vec (\KB\odot \ZB) = \diag(\vec \KB) \vec \ZB.
\end{align*}
So the Schur multiplier can also be understood as a diagonal operator under the canonical vectorization. 
\end{remark}

\begin{proof}[Proof of \Cref{lemma:divided-difference}]
Following the introduced notation for the eigendecomposition of $\AB,\BB$,
we have
\begin{align*}
\uB_i\uB_i^\top  \AB \vB_j\vB_j^\top = \lambda_i \uB_i\uB_i^\top \vB_j \vB_j^\top,\quad \uB_i\uB_i^\top  \BB \vB_j\vB_j^\top = \gamma_j \uB_i\uB_i^\top \vB_j \vB_j^\top,
\end{align*}
which implies 
\begin{align*}
    f^{[1]}(\AB, \BB)\circ (\AB - \BB) 
    &= \sum_{i, j} f^{[1]}(\lambda_i,\gamma_j) \uB_i\uB_i^\top (\AB - \BB) \vB_j\vB_j^\top \\
    &= \sum_{i, j} f^{[1]}(\lambda_i,\gamma_j) (\lambda_i - \gamma_j) \uB_i\uB_i^\top \vB_j \vB_j^\top\\
    &= \sum_{i, j} \big(f(\lambda_i) - f(\gamma_j)\big) \uB_i\uB_i^\top \vB_j \vB_j^\top\\
    &= f(\AB) - f(\BB).
\end{align*}
Thus $f^{[1]}(\AB, \BB)$ is a matrix divided difference.
\end{proof}
A similar calculation leads to the following identities.
\begin{corollary}[chain rule]\label{lemma:schur:left-right-commute}
For all symmetric matrices $\AB,\BB$, and any compatible matrix $\CB$, we have
\begin{align*}
    f(\AB) \CB - \CB f(\BB) = f^{[1]}(\AB, \BB)\circ (\AB \CB - \CB \BB).
\end{align*}
In particular,
we have 
\begin{align*}
    \big( f(\AB) - f(\BB) \big) \BB &= r(\AB, \BB) \circ (\AB - \BB) \quad \text{for the kernel}\ \ r: (x,y) \mapsto y f^{[1]}(x,y); \\ 
    \AB \big( f(\AB) - f(\BB) \big) \BB &= s(\AB, \BB) \circ (\AB - \BB) \quad \text{for the kernel}\ \ s: (x,y) \mapsto x y f^{[1]}(x,y).
\end{align*}
\end{corollary}

The following result is extremely useful for bounding Schur norms of various kernels.
\begin{proposition}[factorization rule]\label{lemma:schur:factorization-bound}
If kernel $k'$ admits a factorization $k'(x, y) := a(x) k(x,y)b(y)$, then
    \begin{align*} 
\|k'(I, J)\|_{\mathfrak{S}}\le \|a(I)\|_{\infty} \|b(J)\|_{\infty} \|k(I, J)\|_{\mathfrak{S}}.
\end{align*}
\end{proposition}
\begin{proof}[Proof of \Cref{lemma:schur:factorization-bound}]
Taking any finite subsets $I,J$, we have
\begin{align*}
    k'(I, J)\odot \ZB = \big( \diag(a(I))k(I,J) \diag(b(J)) \big) \odot \ZB = \diag(a(I)) \big(k(I,J)\odot \ZB \big) \diag(b(J)),
\end{align*}
so the observation follows by taking the operator norm and supremum over subsets $I,J$.
\end{proof}

\subsection{Useful lemmas}

\begin{proof}[Proof of \Cref{lemma:schur:examples}]
For ridge, we have $r(x,y)=1$, so the Schur norm is~$1$.

\textbf{GD.}~
We only give the case when~$t$ is even for simplicity. For GD, the kernel is 
\begin{align*}
    r(x,y) &:= \frac{g(x) -g(y)}{g^*(x) - g^*(y)}  \\ 
    &= \frac{(1-\eta y )^t - (1-\eta x)^t}{x/\big(x+1/(\eta t)\big)-y/\big(y+1/(\eta t)\big)}  \\
    &= t \frac{(1-\eta y )^t - (1-\eta x)^t}{\eta x- \eta y}  \bigg(\eta x+\frac{1}{ t} \bigg) \bigg(\eta y+\frac{1}{t}\bigg), && 0\le x, y\le \frac{1}{\eta}.
\end{align*}
Without loss of generality, assume $\eta = 1$; otherwise, replace $\eta x$ and $\eta y$ by $x$ and $y$ respectively.
Observe that 
\begin{align*}
    r(x,y) &:= t \frac{(1-  y )^t - (1-  x)^t}{x-y}  \bigg( x+\frac{1}{ t} \bigg) \bigg( y+\frac{1}{t}\bigg)  \\
    &= t \frac{\sum_{s=0}^{t-1}(1-y)^{s}(x-y)(1-x)^{t-1-s}}{x-y}  \bigg( x+\frac{1}{ t} \bigg) \bigg( y+\frac{1}{t}\bigg)  \\ 
    &= t \sum_{s=0}^{t-1} \bigg( x+\frac{1}{ t} \bigg)  (1-x)^{t-1-s}  (1-y)^{s} \bigg( y+\frac{1}{t}\bigg), && 0\le x,y\le 1.
\end{align*}
By symmetry, we have 
\begin{align*}
    r(x ,y) = r^*(x, y) + r^*(y,x),\quad r^*(x, y) :=  t \sum_{s=0}^{t/2-1} \bigg( x+\frac{1}{ t} \bigg)  (1-x)^{t-1-s}  (1-y)^{s} \bigg( y+\frac{1}{t}\bigg).
\end{align*}
So $\|r(I, I)\|_{\Sfrak}\le 2\|r^*(I,I)\|_{\Sfrak} $, and it suffices to bound the right-hand side. Write
\begin{align*}
    r^*(x, y) &:=  t \sum_{s=0}^{t/2-1} \bigg( x+\frac{1}{ t} \bigg)  (1-x)^{t-1-s}  (1-y)^{s} \bigg( y+\frac{1}{t}\bigg) \\
    &= \underbrace{t \sum_{s=0}^{t/2-1} \bigg( x+\frac{1}{ t} \bigg)  (1-x)^{t-1-s}  (1-y)^{s}y}_{r^*_1} + \underbrace{\sum_{s=0}^{t/2-1} \bigg( x+\frac{1}{ t} \bigg)  (1-x)^{t-1-s}  (1-y)^{s}}_{r^*_2}.
\end{align*}
For the second component, we have 
\begin{align*}
    \|r^*_2(I,I)\|_{\Sfrak} & \le  \sum_{s=0}^{t/2-1} \sup_{0\le x\le 1}\bigg( x+\frac{1}{ t} \bigg)  (1-x)^{t-1-s} \sup_{0\le y\le 1}  (1-y)^{s} && \explain{\Cref{lemma:schur:factorization-bound}} \\
    &\le  \frac{t}{2} \sup_{0\le x\le 1}\bigg( x+\frac{1}{ t} \bigg)  (1-x)^{t/2} &&\explain{$s\le t/2-1$} \\
    &\le \frac{t}{2} \bigg( \frac{2}{t} + \frac{1}{t} \bigg) = \frac{3}{2}. && \explain{$(1-x)^k \le \min\bigg\{1, \frac{1}{kx} \bigg\}$} 
\end{align*}
For the first component, by Abel summation, we have 
\begin{align*}
    r^*_1(x,y) &:= t \bigg( x+\frac{1}{ t} \bigg)\sum_{s=0}^{t/2-1}   (1-x)^{t-1-s}  (1-y)^{s}y \\
    &= t \bigg( x+\frac{1}{ t} \bigg) \sum_{s=0}^{t/2-1}  (1-x)^{t-1-s} \big( (1-y)^s - (1-y)^{s+1} \big) \\ 
    &= t \bigg( x+\frac{1}{ t} \bigg) \Bigg(\sum_{s=0}^{t/2-1}   \big( (1-x)^{t-1-s}  - (1-x)^{t-s} \big) (1-y)^s +   (1-x)^{t} - (1-x)^{t/2}(1-y)^{t/2} \Bigg)\\
    &= t \sum_{s=0}^{t/2-1} \bigg( x+\frac{1}{ t} \bigg)  x (1-x)^{t-1-s}  (1-y)^s +  t\bigg( x+\frac{1}{ t} \bigg)  (1-x)^{t} - t\bigg( x+\frac{1}{ t} \bigg)(1-x)^{t/2}(1-y)^{t/2}. 
\end{align*}
Then by \Cref{lemma:schur:factorization-bound} we have
\begin{align*}
    &\ \|r^*_1(I,I)\|_{\Sfrak}  \\
    & \le t \sum_{s=0}^{t/2-1} \sup_{ x} \bigg( x+\frac{1}{ t} \bigg)  x (1-x)^{t-1-s} \sup_{y} (1-y)^s \\
    &\qquad +  t \sup_{ x } \bigg( x+\frac{1}{ t} \bigg)  (1-x)^{t} + t \sup_x \bigg( x+\frac{1}{ t} \bigg)(1-x)^{t/2} \sup_y (1-y)^{t/2} && \explain{\Cref{lemma:schur:factorization-bound}} \\
    &\le \frac{t^2}{2} \sup_{ x} \bigg( x+\frac{1}{ t} \bigg)  x (1-x)^{t/2} +  2 t \sup_x \bigg( x+\frac{1}{ t} \bigg)(1-x)^{t/2} && \explain{$s\le t/2-1$} \\
    &\le \frac{t^2}{2} \bigg(\frac{16}{t^2} + \frac{2}{t^2}  \bigg) + 2t \bigg( \frac{2}{t} + \frac{1}{t} \bigg)
    =15. && \explain{$(1-x)^k \le \min\bigg\{1, \frac{1}{kx}, \bigg(\frac{2}{kx}\bigg)^2 \bigg\}$} 
\end{align*}
Putting things together, we have 
\begin{align*}
    \|r(I, I)\|_{\Sfrak} \le 2 \|r^*(I, I)\|_{\Sfrak} \le 2\|r^*_1(I, I)\|_{\Sfrak} +2\|r^*_2(I, I)\|_{\Sfrak} \le 33.
\end{align*}

\textbf{PCR.}~
We prove a slightly more general result. For $\eps\in (0,1]$, setting $I = \Rbb_{\ge 0}$ and $H = [(1+\eps)\rho, \infty)$, we claim that
    \begin{align*}
        \|r(I,H )\|_{\Sfrak}\le 2+\frac{4}{\eps}, \quad \|r(H, H)\|_{\Sfrak} =0,\quad \|r(I,\{0\})\|_{\Sfrak}\le 2 .
    \end{align*}
Let $I_0:=(\rho,\infty)$, $I_1:=[0,\rho]$ so that $I=I_0\cup I_1$. Choose the diagonal entries of $r$ to be zero. For $x\in I$ and $y\in H$, we have
\begin{align*}
r(x,y)
:=\frac{g(x)-g(y)}{g^*(x)-g^*(y)}
=
\begin{dcases}
\frac{1}{g^*(y)-g^*(x)} & x\in I_1,\\
0 & x\in I_0.
\end{dcases}
\end{align*}
In particular, $r(H,H)=0$. Moreover,
\begin{align*}
r(x,0)
=
\begin{dcases}
\frac{x+\rho}{x} & x>\rho,\\
0 & 0\le x\le\rho,
\end{dcases}
\quad\implies\quad
\|r(I,\{0\})\|_{\Sfrak}
=\sup_{x>\rho}\frac{x+\rho}{x}\le2.
\end{align*}
It remains to bound $\|r(I,H)\|_{\Sfrak}$. Since the rows indexed
by $I_0$ vanish, $\|r(I,H)\|_{\Sfrak}=\|r(I_1,H)\|_{\Sfrak}$. For $x\in I_1$ and $y\in H$, we have
\begin{align*}
g^*(x)\le\frac12,\quad
g^*(y)\ge\frac{1+\eps}{2+\eps}>\frac12,
\end{align*}
and therefore we have the absolutely convergent expansion
\begin{align*}
r(x,y)
=\frac{1}{g^*(y)-g^*(x)}
=\sum_{\ell=0}^{\infty}g^*(x)^\ell g^*(y)^{-(\ell+1)}.
\end{align*}
Applying \Cref{lemma:schur:factorization-bound} gives
\begin{align*}
\|r(I_1,H)\|_{\Sfrak}
&\le
\sum_{\ell=0}^{\infty}
\big\|g^*(I_1)^\ell\big\|_\infty
\big\|g^*(H)^{-(\ell+1)}\big\|_\infty\\
&\le
\sum_{\ell=0}^{\infty}
\left(\frac12\right)^\ell
\left(\frac{2+\eps}{1+\eps}\right)^{\ell+1} = 2+\frac4\eps,
\end{align*}
as claimed.
\end{proof}

We define the auxiliary kernels
\begin{align*}
h: (x, y) \mapsto \frac{xy}{\lambda} g^{[1]}(x,y),\quad 
s: (x, y) \mapsto (x+\lambda) y \psi^{[1]}(x,y).
\end{align*}

\begin{lemma}\label{lemma:UB:kernel-comparisons}
For any two index sets $I, J\subset \Rbb_{\ge 0}$,
we have 
\begin{align*}
    \|h(I, J)\|_{\mathfrak{S}} \le \|r(I, J)\|_{\mathfrak{S}},\quad 
    \|s(I,J)\|_{\mathfrak{S}} \le \|r(I,J)\|_{\mathfrak{S}} + \|r(I,\{0\})\|_{\mathfrak{S}}.
\end{align*}
In particular, we have $\|s(I,J)\|_{\mathfrak{S}} \le 2\|r(I,J)\|_{\mathfrak{S}}$ if $0\in J$.
\end{lemma}
\begin{proof}[Proof of \Cref{lemma:UB:kernel-comparisons}]
Note that
\begin{align*}
r(x,0) = \frac{g(x)}{g^*(x)} = (x+\lambda) \frac{g(x)}{x} = (x+\lambda) \psi(x).
\end{align*}
The following identities are straightforward to verify:
\begin{align*}
h(x,y) = \frac{x}{x+\lambda} r(x,y) \frac{y}{y+\lambda},\quad 
s(x,y) = r(x,y)\frac{\lambda}{y+\lambda} - r(x, 0).
\end{align*} 
The claims follow from \Cref{lemma:schur:factorization-bound}.
\end{proof}

For the next lemma, we recall the definition of $\tilde\alpha_k$ from~\eqref{eq:tilde-alpha}.

\begin{lemma}\label{lemma:UB:schur-tail}
Let $\uB^*$ be a vector measurable w.r.t. $\XB_{\le k}$, and let $k: I \otimes J \to \Rbb$ be a fixed kernel matrix. Then with probability at least $1-\delta-\exp(-n/c_0)$, it holds that if $\sigma(\bar\AB)\subseteq I$ and $\sigma(\bar\AB_{\le k}) \subseteq J$, then
\begin{align*}
    \big\|\big( k(\bar\AB, \bar\AB_{\le k}) \circ \bar\AB_{>k}  \big)\uB^* \big\|\le  \|k(I, J)\|_{\mathfrak{S}} \cdot c_1 \sqrt{\tilde\alpha_k} \|\uB^*\|.
\end{align*}
\end{lemma}
\begin{proof}[Proof of \Cref{lemma:UB:schur-tail}]
Let the eigendecompositions of $\bar\AB$ and $\bar\AB_{\le k}$ be 
\begin{align*}
\bar\AB = \UB \hat \LambdaB \UB^\top = \sum_{i} \hat\lambda_i \uB_i\uB_i^\top , \quad \bar\AB_{\le k} = \VB \hat \GammaB \VB^\top = \sum_j \hat\gamma_j \vB_j \vB_j^\top.
\end{align*}
By \Cref{prop:kernel-schur}, for any $\epsilon>0$ there exists a Hilbert space $\Hbb$ and Schur factorization 
\begin{align*}
\fB: I \to \Hbb, \  \gB: J \to \Hbb,\quad \text{with}\ \ 
a:= \sup_{x\in I }\|\fB(x)\|<\infty,\ b:= \sup_{y \in J}\|\gB(y)\| < \infty,
\end{align*}
such that 
\begin{align*}
\text{for all $x\in I$ and $y\in J$},\ \ k(x, y) = \la \fB(x), \gB(y)\ra,\ \  \text{and}\ \  ab \le \|k(I, J)\|_{\mathfrak{S}} +\epsilon.
\end{align*}
We may extend $k$ by zero outside $I\times J$, and $f,g$ by zero outside $I,J$, respectively, so that the factorization is preserved. Define the following submatrices
\begin{align*}
    \KB:= \begin{bmatrix}
        k(\hat\lambda_i, \hat\gamma_j)
    \end{bmatrix}_{ij},\ \ 
    \FB := \begin{bmatrix}
    \vdots \\ 
        \fB_i^\top \\
        \vdots 
    \end{bmatrix},\ \  
    \GB:= \begin{bmatrix}
    \vdots \\
        \gB_j^\top  \\
        \vdots
    \end{bmatrix},\ \  \text{where}\ \ 
    \fB_i := \fB(\hat\lambda_i),
    \gB_j := \gB(\hat\gamma_j) \in \Hbb.
\end{align*}
Then we have 
\begin{align*}
\KB = \FB \GB^\top,\quad \sup_i \|\fB_i\| \le a,\quad \sup_j \|\gB_j\|\le b.
\end{align*}
Furthermore, let 
\begin{align*}
\vB^*:= \VB^\top \uB^*.
\end{align*}
Under this setup, we have
\begin{align*}
\big(  k(\bar\AB, \bar\AB_{\le k}) \circ \bar\AB_{>k} \big) \uB^*
&= \UB \big(\KB \odot( \UB^\top \bar\AB_{>k} \VB) \big) \VB^\top \uB^* && \explain{Schur multiplier} \\ 
&= \UB \big(\KB \odot( \UB^\top \bar\AB_{>k} \VB) \big) \vB^* && \explain{$\vB^*:= \VB^\top \uB^*$} \\ 
&= \UB \diag\big( \KB \diag(\vB^*) \VB^\top \bar\AB_{>k}\UB \big), && \explain{Schur product identity}  \\
&= \UB \diag\big( \FB\GB^\top  \diag(\vB^*) \VB^\top \bar\AB_{>k}\UB \big), && \explain{Schur decomposition}  \\
\implies \quad 
\big\|\big(  k(\bar\AB, \bar\AB_{\le k}) \circ \bar\AB_{>k} \big) \uB^*\big\|^2
&= \big\| \diag\big( \FB\GB^\top \diag(\vB^*) \VB^\top \bar\AB_{>k}\UB \big)\big\|^2. && \explain{$\UB^\top \UB = \IB$} \\
&= \sum_i \big( \eB_i^\top \FB\GB^\top \diag(\vB^*) \VB^\top \bar\AB_{>k} \UB  \eB_i \big)^2 \\
&= \sum_i \big( \fB_i^\top \GB^\top \diag(\vB^*) \VB^\top \bar\AB_{>k} \UB  \eB_i \big)^2 && \explain{$\eB_i^\top \FB = \fB_i^\top$} \\
&\le \sum_i \|\fB_i\|^2 \big\| \GB^\top \diag(\vB^*) \VB^\top \bar\AB_{>k} \UB   \eB_i  \big\|^2 && \explain{Cauchy--Schwarz} \\
&\le a^2 \sum_i  \big\|\GB^\top \diag(\vB^*) \VB^\top \bar\AB_{>k} \UB   \eB_i  \big\|^2 && \explain{$\sup_i\|\fB_i\|\le a$} \\
&= a^2   \big\|\GB^\top \diag(\vB^*) \VB^\top \bar\AB_{>k}     \big\|^2_{F} && \explain{$\sum_i  \eB_i \eB_i^\top, \UB\UB^\top = \IB$} \\
     &= a^2 \|\bar\AB_{>k}\VB \diag(\vB^*) \GB\|^2_F.
\end{align*}
In what follows, we treat $\XB_{\le k}$ as fixed and only consider the randomness of $\XB_{>k}$.
We control the $L^p(\XB_{>k})$-norm of the random variable on the right-hand side.
Recall that $\VB$, $\vB^*$, and $\GB$ depend only on $\XB_{\le k}$ and are independent of $\XB_{>k}$. Define
\begin{align*}
\tilde\alpha:=\alpha_k+ \frac{p}{n}\lambda_{k+1}^2, \quad p:=\max\{4,\log(1/\delta)\}\le n/c_0.
\end{align*}
Applying~\eqref{eq:lp} at order~$2p$, followed by Minkowski's inequality, we have
\begin{align*}
   \big\|  \|\bar\AB_{>k}\VB \diag(\vB^*) \GB\|^2_F \big\|_{L^p(\XB_{>k})} 
   &\le c_1^2\tilde\alpha \| \VB \diag(\vB^*) \GB\|_F^2 && \explain{\Cref{lemma:basic:tail}} \\
   &= c_1^2\tilde\alpha \| \diag(\vB^*) \GB\|_F^2 && \explain{$\VB^\top \VB = \IB$} \\
   &= c_1^2\tilde\alpha\sum_{j} \vB_j^{*2} \|\gB_j\|^2  \\
   &\le c_1^2\tilde\alpha b^2 \|\vB^*\|^2 && \explain{$\sup_j\|\gB_j\|\le b$} \\
   &= c_1^2\tilde\alpha b^2 \|\uB^*\|^2. && \explain{$\vB^* = \VB^\top \uB^*$}
\end{align*}
Putting them together, we have 
\begin{align*}
   \big\|  \big(k(\bar\AB, \bar\AB_{\le k}) \circ \bar\AB_{>k} \big) \uB^* \big\|_{L^p(\XB_{>k})} 
   &\le a \big\|  \|\bar\AB_{>k}\VB \diag(\vB^*) \GB\|_F \big\|_{L^p(\XB_{>k})}  \\
   &\le a b c_1\sqrt{\tilde\alpha} \|\uB^*\| \\
   &\le \big(\|k(I, J)\|_{\mathfrak{S}} + \epsilon\big) c_1 \sqrt{\tilde\alpha} \|\uB^*\|. && \explain{$ab\le \|k(I, J)\|_{\mathfrak{S}} + \epsilon$}
\end{align*}
Taking $\epsilon\to 0$, we have 
\begin{align*}
    \big\|  \big(k(\bar\AB, \bar\AB_{\le k}) \circ \bar\AB_{>k} \big) \uB^* \big\|_{L^p(\XB_{>k})}  \le \|k(I, J)\|_{\mathfrak{S}} c_1\sqrt{\tilde\alpha}  \|\uB^*\|.
\end{align*}
Since $\tilde\alpha\le c\tilde\alpha_k$, applying Markov's inequality and rescaling constants gives the claim.
\end{proof}

\begin{lemma}\label{lemma:UB:change-of-geometry}
Let $\lambda$ and $k$ be such that  
\begin{align*}
k\le n/c_3,\quad 
    \tilde\lambda:= \lambda+\frac{\sum_{i>k}\lambda_i}{n} \ge c_2 \lambda_{k+1}.
\end{align*}
Under \Cref{lemma:basic:concentration,lemma:basic:tail-regularization}, we have
\begin{align*}
    \bigg\| \frac{1}{\sqrt{n}}\SigmaB^{1/2}\XB^\top (\bar \AB + \lambda\IB)^{-1}\bigg\|\le 9.
\end{align*}
\end{lemma}
\begin{proof}[Proof of \Cref{lemma:UB:change-of-geometry}]
    We have 
\begin{align*}
     \bigg\| \frac{1}{\sqrt{n}}\SigmaB^{1/2}\XB^\top (\bar \AB + \lambda\IB)^{-1}\bigg\|
     &\le  \bigg\| \frac{1}{\sqrt{n}}\SigmaB^{1/2}_{\le k}\XB^\top_{\le k} (\bar \AB + \lambda\IB)^{-1}\bigg\| +  \bigg\| \frac{1}{\sqrt{n}}\SigmaB^{1/2}_{>k}\XB^\top_{>k} (\bar \AB + \lambda\IB)^{-1}\bigg\| \\
     &\le \sqrt{2}\big\|\bar\AB_{\le k} (\bar \AB + \lambda\IB)^{-1}\big\| + \sqrt{ \lambda_{k+1}} \big\| \bar\AB_{>k}^{1/2} ( \bar\AB + \lambda\IB)^{-1}\big\|. && \explain{\Cref{lemma:basic:concentration}}
\end{align*}
By \Cref{lemma:basic:tail-regularization}, we have 
\begin{align*}
    \frac{1}{2}\tilde\lambda\IB\preceq \bar\AB_{>k} + \lambda\IB \preceq 2 \tilde\lambda\IB \quad\implies\quad 
    \bar \AB + \lambda\IB \succeq  \frac{1}{2}\tilde\lambda\IB.
\end{align*}
So the first term is bounded by 
\begin{align*}
    \big\|\bar\AB_{\le k} (\bar \AB + \lambda\IB)^{-1}\big\|\le 1 + \big\|(\bar\AB_{> k} + \lambda \IB) (\bar \AB + \lambda\IB)^{-1}\big\| \le 5,
\end{align*}
and the second term is bounded by 
\begin{align*}
   \sqrt{ \lambda_{k+1}} \big\| \bar\AB_{>k}^{1/2} ( \bar\AB + \lambda\IB)^{-1}\big\|
   \le \sqrt{ \lambda_{k+1}} \big\|( \bar\AB + \lambda\IB)^{-1/2}\big\|  
   \le \sqrt{2\lambda_{k+1}/\tilde\lambda} \le \sqrt{2/c_2}.
\end{align*}
This completes the proof.
\end{proof}

\subsection{Proof of Theorem \ref{thm:master-ub}}
We introduce some additional notation. It is more convenient to work with the normalized Gram matrices
\begin{align*}
\bar\AB:= \frac{1}{n}\XB\XB^\top ,\quad 
\bar\AB_{\le k}:= \frac{1}{n} \XB_{\le k} \XB_{\le k}^\top,\quad 
\bar\AB_{> k} := \frac{1}{n} \XB_{> k} \XB_{> k}^\top.
\end{align*}
\begin{proof}[Proof of \Cref{thm:master-ub}]
The bias iterate is 
\begin{align*}
     \bigg(\IB-\frac{1}{n}\XB^\top \psi(\bar\AB)\XB \bigg)\wB^*
     &= \begin{bmatrix}
         \big(\IB- \frac{1}{n}\XB_{\le k}^\top \psi(\bar\AB) \XB_{\le k} \big) \wB^*_{\le k} \\ 
         -\frac{1}{n}\XB_{>k}^\top \psi(\bar \AB) \XB_{\le k}\wB^*_{\le k}
     \end{bmatrix}
     + \begin{bmatrix}
         -\frac{1}{n}\XB_{\le k}^\top \psi(\bar\AB) \XB_{> k}  \wB^*_{> k} \\ 
         \big( \IB - \frac{1}{n}\XB_{>k}^\top \psi(\bar \AB) \XB_{> k}\big) \wB^*_{> k}
     \end{bmatrix}
\end{align*}
Then the bias error is
\begin{align*}
    \bias = \bigg\| \bigg(\IB-\frac{1}{n}\XB^\top \psi(\bar\AB)\XB \bigg)\wB^* \bigg\|_{\SigmaB}
    \le \underbrace{\bigg\| \bigg(\IB-\frac{1}{n}\XB^\top \psi(\bar\AB)\XB \bigg)\begin{bmatrix}
        \wB^*_{\le k} \\
        0
    \end{bmatrix} \bigg\|_{\SigmaB}}_{\bias_1} +  \underbrace{\bigg\| \bigg(\IB-\frac{1}{n}\XB^\top \psi(\bar\AB)\XB \bigg)\begin{bmatrix}
        0\\
        \wB^*_{> k}
    \end{bmatrix} \bigg\|_{\SigmaB}}_{\bias_2}.
\end{align*}
Here, we use the unsquared bias. The first term is further decomposed as
\begin{align*}
 \bias_1  & := \bigg\|\SigmaB^{1/2} \bigg(\IB-\frac{1}{n}\XB^\top \psi(\bar\AB)\XB \bigg)\begin{bmatrix}
        \wB^*_{\le k} \\
        0
    \end{bmatrix} \bigg\| \\ 
 &\le \bigg\| \SigmaB^{1/2}\bigg(\IB-\frac{1}{n}\XB^\top \psi(\bar\AB_{\le k})\XB \bigg)\begin{bmatrix}
        \wB^*_{\le k} \\
        0
    \end{bmatrix} \bigg\|
    + \bigg\|\SigmaB^{1/2} \frac{1}{n}\XB^\top \big( \psi(\bar\AB)- \psi(\bar\AB_{\le k}) \big) \XB \begin{bmatrix}
        \wB^*_{\le k} \\
        0
    \end{bmatrix} \bigg\| \\
&=  \bigg\| \begin{bmatrix} \SigmaB^{1/2}_{\le k}\big(\IB-\frac{1}{n}\XB^\top_{\le k} \psi(\bar\AB_{\le k})\XB_{\le k} \big)
    \wB^*_{\le k} \\
    \SigmaB^{1/2}_{>k}\frac{1}{n}\XB_{> k}^\top \psi(\bar\AB_{\le k})\XB_{\le k} \wB^*_{\le k}
\end{bmatrix} \bigg\|
+ \bigg\| \SigmaB^{1/2}\frac{1}{n}\XB^\top \big( \psi(\bar\AB)- \psi(\bar\AB_{\le k}) \big) \XB_{\le k} 
    \wB^*_{\le k} \bigg\| \\
 &\le \underbrace{\bigg\|\SigmaB_{\le k}^{1/2} \bigg( \IB-\frac{1}{n}\XB_{\le k}^\top \psi(\bar\AB_{\le k})\XB_{\le k} \bigg)\wB^*_{\le k} \bigg\|}_{\text{head bias}} + \underbrace{\bigg\|\SigmaB_{> k}^{1/2} \frac{1}{n}\XB_{> k}^\top \psi(\bar\AB_{\le k})\XB_{\le k} \wB^*_{\le k} \bigg\|}_{\text{spillover}} \\
 &\qquad +   \underbrace{\bigg\|\SigmaB^{1/2} \frac{1}{n}\XB^\top \big( \psi(\bar\AB)- \psi(\bar\AB_{\le k})\big)\XB_{\le k} \wB^*_{\le k} \bigg\|}_{\text{LTO}}.
\end{align*}

\paragraph{Leave-tail-out error.}
For the LTO error, 
let 
\begin{align*}
\uB^*_{\le k}:= \sqrt{n} \XB_{\le k} (\XB^\top_{\le k} \XB_{\le k})^{-1}\wB^*_{\le k} \quad \implies\quad \|\uB^*_{\le k}\|\le \sqrt{2}\|\wB^*_{\le k}\|_{\SigmaB_{\le k}^{-1}},
\end{align*}
which is independent of $\bar\AB_{>k}$,
then we have 
\begin{align*}
 &\ \text{LTO} 
:=  \frac{1}{\sqrt{n}}\big\|\SigmaB^{1/2} \XB^\top\big( \psi(\bar\AB) -   \psi(\bar\AB_{\le k})\big)\bar\AB_{\le k} \uB^*_{\le k} \big\|  \\
&\le 9 \big\|(\bar\AB+\lambda\IB) \big( \psi(\bar\AB) -   \psi(\bar\AB_{\le k})\big)\bar\AB_{\le k} \uB^*_{\le k} \big\| && \explain{\Cref{lemma:UB:change-of-geometry}} \\
&= 9 \big\| \big( s(\bar \AB, \bar \AB_{\le k})\circ \bar \AB_{> k}\big) \uB^*_{\le k} \big\| &&  \explain{\Cref{lemma:schur:left-right-commute} and $s(x,y) := (x+\lambda)y \psi^{[1]}(x,y)$}  \\
&\le 9c_1 \sqrt{\tilde\alpha_k}\|s(I, H)\|_{\mathfrak{S}}\cdot  \|\uB_{\le k}^*\|  && \explain{\Cref{lemma:UB:schur-tail}} \\
&\le 9\sqrt{2} c_1 \sqrt{\tilde\alpha_k}\|s(I, H)\|_{\mathfrak{S}}\cdot \|\wB^*_{\le k}\|_{\SigmaB_{\le k}^{-1}}. && \explain{$\|\uB^*_{\le k}\|\le \sqrt{2}\|\wB^*_{\le k}\|_{\SigmaB_{\le k}^{-1}}$}
\end{align*}

\paragraph{Head bias error.}
For the bias error in the $k$-dimensional head space, we have 
\begin{align*}
     \text{head bias} 
    &= \big\|\SigmaB_{\le k}^{1/2} \big( \IB- \hat\SigmaB_{\le k}\psi (\hat\SigmaB_{\le k}) \big)\wB^*_{\le k} \big\|  \\
    &= \big\|\SigmaB_{\le k}^{1/2} \big( \IB-g(\hat\SigmaB_{\le k}) \big)\wB^*_{\le k} \big\| && \explain{$g(z) = z\psi(z)$} \\
    &\le \big\|\SigmaB_{\le k}^{1/2} \big( \IB-g(\SigmaB_{\le k}) \big)\wB^*_{\le k} \big\| + \big\|\SigmaB_{\le k}^{1/2} \big( g(\SigmaB_{\le k})- g(\hat\SigmaB_{\le k}) \big)\wB^*_{\le k} \big\|  \\
    &\le \big\|\SigmaB_{\le k}^{1/2} \big( \IB-g(\SigmaB_{\le k}) \big)\wB^*_{\le k} \big\| + \Big\|\SigmaB_{\le k}^{1/2} \big( g(\SigmaB_{\le k})- g(\hat\SigmaB_{\le k}) \big)\hat\SigmaB_{\le k}^{1/2}\Big\|\cdot \big\|\hat\SigmaB_{\le k}^{-1/2}\wB^*_{\le k} \big\|  \\
    &\le \big\|\SigmaB_{\le k}^{1/2} \big( \IB-g(\SigmaB_{\le k}) \big)\wB^*_{\le k} \big\| + \Big\|\SigmaB_{\le k}^{1/2} \big( g(\SigmaB_{\le k})- g(\hat\SigmaB_{\le k}) \big)\hat\SigmaB_{\le k}^{1/2}\Big\| \cdot \sqrt{2}\big\|\SigmaB_{\le k}^{-1/2}\wB^*_{\le k} \big\|, && \explain{$\SigmaB_{\le k}\preceq 2 \hat\SigmaB_{\le k}$}
\end{align*}
in which the concentration error can be bounded via the Schur norm:
\begin{align*}
   &\ \Big\|  \SigmaB_{\le k}^{1/2} \big( g(\SigmaB_{\le k})- g(\hat\SigmaB_{\le k}) \big) \hat\SigmaB_{\le k}^{1/2} \Big\| \\ 
    &= \Big\| \SigmaB_{\le k}^{1/2} \Big(  g^{[1]}( \SigmaB_{\le k} , \hat\SigmaB_{\le k} ) \circ (\SigmaB_{\le k} -\hat \SigmaB_{\le k} ) \Big) \hat \SigmaB_{\le k}^{1/2}  \Big\| && \explain{Schur multiplier} \\
    &= \Big\| \lambda h( \SigmaB_{\le k} , \hat\SigmaB_{\le k} ) \circ \Big(\SigmaB_{\le k}^{-1/2} ( \SigmaB_{\le k} -\hat \SigmaB_{\le k} ) \hat \SigmaB_{\le k}^{-1/2} \Big) \Big\| && \explain{\Cref{lemma:schur:left-right-commute} and definition of $h$}  \\
    &\le \lambda  \|h(H\cup\{0\}, H\cup\{0\})\|_{\mathfrak{S}} \cdot \Big\| \SigmaB_{\le k}^{-1/2} ( \SigmaB_{\le k} -\hat \SigmaB_{\le k} ) \hat \SigmaB_{\le k}^{-1/2} \Big\| &&\explain{Schur norm} \\
    &\le \lambda  \|h(H\cup\{0\}, H\cup\{0\})\|_{\mathfrak{S}} \cdot c_1 \sqrt{\frac{k+\log(1/\delta)}{n}}. && \explain{\Cref{lemma:basic:concentration}}
\end{align*}
Since $h(0,\cdot)=h(\cdot,0)=0$, it holds that
\begin{align*}
\|h(H\cup\{0\}, H\cup\{0\})\|_{\mathfrak{S}} = \|h(H,H)\|_{\mathfrak{S}} \le \|r(H,H)\|_{\mathfrak{S}},
\end{align*}
therefore the head bias error is bounded as 
\begin{align*}
    \text{head bias} 
    \le \big\| \big( \IB-g(\SigmaB_{\le k}) \big)\wB^*_{\le k} \big\|_{\SigmaB_{\le k}} +   \sqrt{2} c_1  \|h(H, H)\|_{\mathfrak{S}} \cdot\lambda \sqrt{\frac{k+\log(1/\delta)}{n} }  \|\wB^*_{\le k} \|_{\SigmaB_{\le k}^{-1}}.
\end{align*}

\paragraph{Spillover error.}
The error caused by head signal spillover to the tail subspace  can be controlled by simple concentration:
\begin{align*}
\text{spillover}
&:=    \bigg\|\SigmaB_{> k}^{1/2} \frac{1}{n}\XB_{> k}^\top \psi(\bar\AB_{\le k})\XB_{\le k} \wB^*_{\le k} \bigg\| \\
&=  \frac{1}{\sqrt{n} } \big\|\SigmaB_{> k}^{1/2} \XB_{> k}^\top \psi(\bar\AB_{\le k})\bar \AB_{\le k} \uB^*_{\le k} \big\| && \explain{$\uB^* := \sqrt{n}\XB_{\le k} (\XB_{\le k}^\top \XB_{\le k})^{-1}\wB^*_{\le k}$} \\
&=\frac{1}{\sqrt{n} } \big\|\SigmaB_{> k}^{1/2} \XB_{> k}^\top g(\bar\AB_{\le k}) \uB^*_{\le k} \big\| && \explain{$g(z) := z \psi(z)$} \\
&\le c_1 \sqrt{\tilde\alpha_k} \big\|g(\bar\AB_{\le k}) \uB^*_{\le k}\big\| && \explain{\Cref{lemma:basic:tail}} \\
&\le c_1 \sqrt{\tilde\alpha_k} \big\|g(\bar\AB_{\le k})\big\|\cdot \| \uB^*_{\le k}\|  \\
&\le  c_1 \sqrt{2\tilde\alpha_k} \left(\| r(I,H) \|_{\mathfrak{S}} + \| r(I,\{0\}) \|_{\mathfrak{S}} \right) \cdot\|\wB^*_{\le k}\|_{\SigmaB_{\le k}^{-1}}. && \explain{$\|\uB^*_{\le k}\|\le \sqrt{2}  \|\wB^*_{\le k}\|_{\SigmaB_{\le k}^{-1}}$}
\end{align*}
For the last inequality, we have also used that for any $x\in I$,
\begin{align*}
|g(y)| \le |g(x)| + |r(x,y)||g^*(y)-g^*(x)| \le |r(x,0)| + |r(x,y)|.
\end{align*}

\paragraph{Tail error.}
The tail error is bounded by
\begin{align*}
\bias_{2} &:= 
\bigg\| \SigmaB^{1/2} \bigg(\IB-\frac{1}{n}\XB^\top \psi(\bar\AB)\XB \bigg)\begin{bmatrix}
        0\\
        \wB^*_{> k}
    \end{bmatrix} \bigg\| \\
    &\le \bigg\| \SigmaB^{1/2} \frac{1}{n}\XB^\top \psi(\bar\AB)\XB \begin{bmatrix}
        0\\
        \wB^*_{> k}
    \end{bmatrix} \bigg\|  + \bigg\| \SigmaB^{1/2} \begin{bmatrix}
        0\\
        \wB^*_{> k}
    \end{bmatrix} \bigg\| \\
&= \bigg\| \SigmaB^{1/2} \frac{1}{n}\XB^\top \psi(\bar\AB)\XB_{>k}
        \wB^*_{> k}\bigg\|  + \big\|\SigmaB_{>k}^{1/2} \wB^*_{>k}\big\|  \\
 &\le 9 \bigg\| (\bar \AB +\lambda\IB) \psi(\bar\AB) \frac{1}{\sqrt{n}}\XB_{>k}
        \wB^*_{> k}\bigg\|  + \big\|\SigmaB_{>k}^{1/2} \wB^*_{>k}\big\| && \explain{\Cref{lemma:UB:change-of-geometry}}  \\
&= 9 \bigg\| r(\bar\AB, 0)\frac{1}{\sqrt{n}}\XB_{>k}
        \wB^*_{> k}\bigg\|  + \big\|\SigmaB_{>k}^{1/2} \wB^*_{>k}\big\| && \explain{$r(x,0)= g(x)/g^*(x) = (x+\lambda) \psi(x) $} \\
    &\le 9 \| r(I, \{0\}) \|_{\mathfrak{S}}\cdot c_1 \big\|\SigmaB_{>k}^{1/2} \wB^*_{>k}\big\| + \big\|\SigmaB_{>k}^{1/2} \wB^*_{>k}\big\| && \explain{\Cref{lemma:basic:concentration}} \\
    &= \Big(1+ 9c_1 \| r(I, \{0\}) \|_{\mathfrak{S}} \Big) \| \wB^*_{>k}\|_{\SigmaB_{>k}}.
\end{align*}

\paragraph{Variance error.}
Since $r(x,0) = g(x)/g^*(x)$, it holds that $|g| \le \|r(I,\{0\})\|_{\mathfrak{S}} \cdot |g^*|$, so the first part follows from the monotonicity of trace. On the other hand, by \Cref{lemma:UB:change-of-geometry}, we have 
\begin{align*}
    \bigg\| \frac{1}{n} (\bar \AB +\lambda\IB)^{-1}\XB \SigmaB \XB^\top (\bar \AB +\lambda\IB)^{-1}  \bigg\|\le 81,
\end{align*}
so that
\begin{align*}
    \variance(\hat\wB_g) &= \frac{\sigma^2}{n}  \tr\bigg(\frac{1}{n} \XB \SigmaB \XB^\top (\bar \AB +\lambda\IB)^{-2} r(\bar \AB, 0)^2\bigg) \\
    &\le \frac{81\sigma^2}{n} \tr\left(r(\bar \AB, 0)^2\right) \\
    &\le \frac{81\sigma^2}{n}\big\|r(\bar\AB, 0) \big\|^2 \rank(r(\bar \AB, 0)) \\
    &\le \frac{81\sigma^2}{n} \|r(I, \{0\})  \|^2_{\Sfrak} \rank(g(\bar\AB)).
\end{align*}
This completes the proof.
\end{proof}

\subsection{Proof of Corollary \ref{thm:holder}}

We require the following smoothness bound for the Schur norm.

\begin{proposition}\label{thm:besov}
Let $I\subseteq\Rbb_{\ge 0}$ be an interval and $\beta\in(0,1]$. There exists a universal constant $c>0$ such that for any function $f\in C^{1,\beta}(I)$,
\begin{align*}
\|f^{[1]}\|_{\mathfrak S(I,I)}
\le \frac{c}{\beta}\|f\|_{C^{1,\beta}(I)}.
\end{align*}
\end{proposition}

\begin{proof}[Proof of \Cref{thm:besov}]
In fact, the following more general statement is true: for any function $f\in B_{\infty,1}^1(I)\cap C^1(I)$, where $B_{\infty,1}^1(I)$ denotes the Besov space restricted to $I$ \citep{triebel1983theory},
\begin{align*}
\|f^{[1]}\|_{\mathfrak S(I,I)}
\le c\|f\|_{B_{\infty,1}^1(I)}.
\end{align*}
Indeed, this is a consequence of Theorems~1.6.1,~3.1.10, and~3.3.6 of \citet{aleksandrov2016operator}. The claim then follows from the continuous embedding $C^{1,\beta}(\Rbb)\subset B_{\infty,1}^1(\Rbb)$ and the estimates
\begin{align*}
\omega_2(f,h)_\infty &= \sup_{x, 0\le u\le h} \left| \int_0^u f'(x+u+v)-f'(x+v)\dif v\right| \le \|f\|_{C^{1,\beta}} h^{1+\beta}, \\
\|f\|_{B_{\infty,1}^1} &\lesssim \|f\|_\infty + \int_0^1 \frac{\omega_2(f,h)_\infty}{h^2}\dif h \le \|f\|_\infty + \frac{\|f\|_{C^{1,\beta}}}{\beta},
\end{align*}
see \citet[Theorem~2.5.12]{triebel1983theory}.
\end{proof}

\begin{proof}[Proof of \Cref{thm:holder}]
We apply \Cref{thm:master-ub} with $k=k^*$ and $I=H=\Rbb_{\ge 0}$. Denote 
\begin{align*}
q_j(z):=z^j(1-g(\lam z)), \quad j=0,1,2.
\end{align*}
Defining diagonal values by continuous extension, the relative kernel~$r$ can be shown after some algebra to satisfy
\begin{align*}
r(\lam x,\lam y)
&=-q_0^{[1]}(x,y)-2q_1^{[1]}(x,y)-q_2^{[1]}(x,y) +q_0(x)+q_0(y)+q_1(x)+q_1(y).
\end{align*}
Then by \Cref{thm:besov} and the definitions of $C_\beta,C_\psi$,
\begin{align*}
\|r(I,I)\|_{\mathfrak S}
&\lesssim \max_{j\in\{0,1,2\}}
\left(\|q_j^{[1]}\|_{\mathfrak S}+\|q_j\|_\infty\right)
\lesssim C_\beta,\quad
\|r(I,\{0\})\|_{\mathfrak S}
\le C_\psi.
\end{align*}
Also, the condition \eqref{eq:spec-cond} holds with probability at least $1-\exp(-n/c_0)$ due to
\Cref{lemma:basic:concentration}. Finally, the $\lam_{k+1}^2$ term can be absorbed by assuming $\log(1/\delta) \le n/c_0$ without loss of generality and noting that
\begin{align*}
\alpha_k+\frac{\lam_{k+1}^2 \log(1/\delta)}{n}
&\le \alpha_k+
\frac{\log(1/\delta)}{c_2^2 n}\left(\lam+\frac{\sum_{i>k}\lam_i}{n}\right)^2
\le c\left( \alpha_k+\frac{\lam^2 \log(1/\delta)}{n} \right).
\end{align*}
This completes the proof.
\end{proof}

\section{General Risk Lower Bounds}\label{sec:master-lb}

In this section, we prove the following bias lower bound for monotone shrinkage filters.

\begin{theorem}\label{thm:master-lb}
Under \Cref{assum:lbb} with Gaussian design $\xB\sim\Ncal(0,\SigmaB)$, there exist constants $c_0,\ldots,c_3>1$ such that the following holds. Let $g:\R_{\ge 0}\to[0,1]$ be a monotone shrinkage filter, let sample size $n\ge c_0$ and let $k$ be any index satisfying $k\le n/c_3$. For $i\le k$ and $j\ne i$, denote
\begin{align*}
\Lambda_i:=\lam_i+\lam_{k+1}+\frac{\sum_{j>k}\lam_j}{n},
\quad
\Gamma_{ij}:=(g(0.9\lam_i)-g(1.1\lam_j))_+.
\end{align*}
Then
\begin{align*}
c_1 \Ebb\bias(\hat\wB_g)
&\ge \sum_{i\le k}\frac{\lam_i}{\Lambda_i^2} \left(\frac1n\sum_{j\ne i}\lam_j^2\Gamma_{ij}^2\right) \wB_i^{*2} + \big\|(\IB-g(1.1\SigmaB))\wB^*\big\|_{\SigmaB}^2.
\end{align*}
If $g$ is concave, the second term may be replaced by $\big\|(\IB-g(\SigmaB))\wB^*\big\|_{\SigmaB}^2$. 

Moreover, we have the conditional lower bounds:
\begin{align*}
c_2\lam_{k+1} \le \frac{\sum_{i>k}\lam_i}{n} &\quad\implies\quad {c_1}\Ebb\bias(\hat\wB_g)
\ge \|\wB^*\|_{\SigmaB_{k:\infty}}^2,\\
c_2\frac{\sum_{i>k}\lam_i^2}{n}
\le \left(\frac{\sum_{i>k}\lam_i}{n}\right)^2 &\quad\implies\quad {c_1}\Ebb\bias(\hat\wB_g)
\ge \sum_{i\le k}\lam_i
\min\left\{1,\frac{\sum_{j>k}\lam_j}{n\lam_i}\right\}^2\wB_i^{*2}.
\end{align*}
\end{theorem}

\subsection{Proof of Theorem \ref{thm:master-lb}}

Fix $i\le k$ and $j\ne i$ and define the leave-one-out (LOO) projection
\begin{align*}
\PB:=\mathbf 1\{\AB_{-j}>0.9n\lam_i\IB\}.
\end{align*}
Also recall the notation $\PsiB := n^{-1}\psi(\AB/n)$. We first state some intermediate lemmas.

\begin{lemma}\label{lem:general-filter-loo}
For every $i\le k$ with $\lam_i>0$ and $j\ne i$, conditional on $\{\zB_\ell:\ell\ne j\}$, it holds for every $\uB\in\Rbb^n$ that
\begin{align*}
\Ebb_{\zB_j}\big(\zB_j^\top\PsiB\uB\big)^2
\ge c\Gamma_{ij}^2 \big\|\AB_{-j}^{-1}\PB\uB \big\|^2.
\end{align*}
\end{lemma}

\begin{proof}[Proof of \Cref{lem:general-filter-loo}]
If $\lam_j=0$, then
\begin{align*}
\Ebb_{\zB_j}\big(\zB_j^\top\PsiB\uB\big)^2
= \frac{1}{n^2} \|\psi(\AB/n)\uB\|^2
\ge g(0.9\lam_i)^2 \big\|\AB_{-j}^{-1}\PB\uB \big\|^2.
\end{align*}
Assume $\lam_j>0$ and let $\UB:=2\PB-\IB$. Note that
\begin{align*}
\UB\AB_{-j}\UB=\AB_{-j},
\quad
\UB\PB=\PB,
\quad
\UB(\IB-\PB)=-(\IB-\PB).
\end{align*}
Then the cross-correlation term
\begin{align*}
\big(\zB_j^\top\PsiB\PB\uB\big)
\big(\zB_j^\top\PsiB(\IB-\PB)\uB\big)
\end{align*}
is distributionally symmetric under the transformation $\zB_j\mapsto\UB\zB_j$, thus has mean zero. Hence
\begin{align*}
\Ebb_{\zB_j}\big(\zB_j^\top\PsiB\uB\big)^2
\ge
\Ebb_{\zB_j}\big(\zB_j^\top\PsiB\PB\uB\big)^2,
\end{align*}
and it suffices to prove the claim when $\uB=\PB\uB$.

Choose an orthonormal eigenbasis such that
\begin{align*}
\frac{\AB_{-j}}n=\sum_{\ell=1}^n\gamma_\ell\wB_\ell\wB_\ell^\top,
\quad
\uB=\sum_{\ell:\gamma_\ell>0.9\lam_i}u_\ell\wB_\ell.
\end{align*}
Applying \Cref{lem:loo} with $\lam=1.1\lam_j$, on the event $\|\zB_j\|^2\le 1.05n$, for every $\gamma_\ell>0.9\lam_i$ we have
\begin{align*}
h_{g,\ell}(\zB_j)
&\ge
\frac{\Gamma_{ij}}{\gamma_\ell}
\left(1-\frac{\|\zB_j\|^2}{1.1n}\right) \ge
\frac{c\Gamma_{ij}}{\gamma_\ell}.
\end{align*}
Since $h_{g,\ell}$ is invariant under sign change of the coordinates of $\zB_j$ in the basis $(\wB_\ell)$, averaging over signs gives
\begin{align*}
\Ebb_{\zB_j}\big(\zB_j^\top\PsiB\uB\big)^2
&\ge
\frac{1}{n^2}\sum_{\ell:\gamma_\ell>0.9\lam_i}
u_\ell^2
\Ebb_{\zB_j}\left[
(\zB_j^\top\wB_\ell)^2h_{g,\ell}(\zB_j)^2
\mathbf 1\{\|\zB_j\|^2\le 1.05n\}
\right]\\
&\ge
\frac{c}{n^2}\Gamma_{ij}^2 \sum_{\ell:\gamma_\ell>0.9\lam_i}\frac{u_\ell^2}{\gamma_\ell^2}\\
&=
c\Gamma_{ij}^2 \big\|\AB_{-j}^{-1}\PB\uB\big\|^2,
\end{align*}
where we have used that
\begin{align*}
\Ebb_{\zB_j}\left[
(\zB_j^\top\wB_\ell)^2
\mathbf 1\{\|\zB_j\|^2\le 1.05n\}
\right]
\end{align*}
is bounded below by a constant for sufficiently large~$n$. This proves the claim.
\end{proof}

\begin{lemma}\label{lem:general-filter-window}
For every $i\le k$ with $\lam_i>0$ and $j\ne i$,
\begin{align*}
\Ebb\big\|\AB_{-j}^{-1}\PB\zB_i\big\|^2
\ge\frac{c}{n\Lambda_i^2}.
\end{align*}
\end{lemma}

\begin{proof}
Decompose
\begin{align*}
\AB_{-j}
&=
\underbrace{\sum_{\ell<i,\ell\ne j}\lam_\ell\zB_\ell\zB_\ell^\top}_{(i)}
+
\underbrace{\sum_{i\le\ell\le k,\ell\ne j}\lam_\ell\zB_\ell\zB_\ell^\top}_{(ii)}
+
\underbrace{\sum_{\ell>k,\ell\ne j}\lam_\ell\zB_\ell\zB_\ell^\top}_{(iii)}.
\end{align*}
By standard concentration and \Cref{lemma:basic:concentration}, with probability at least $1-\exp(-n/c_0)$,
\begin{align*}
\|(ii)\| &\le
\lam_i\big\|[\zB_\ell]_{i\le\ell\le k}\big\|^2
\le cn\lam_i, \\
\|(iii)\| &\le \|\AB_{>k}\| \le
cn\left(\lam_{k+1}+\frac1n\sum_{\ell>k}\lam_\ell\right),
\end{align*}
so that $\|(ii)\|+\|(iii)\|\le cn\Lambda_i$.

Denote the projection to $\spn\{\zB_\ell:\ell<i, \ell\ne j\}$ by $\PiB_i$. Since $\rank\PiB_i\le k\le n/c_3$, it holds with probability at least $1-\exp(-n/c_0)$ that $\|\PiB_i\zB_i \|^2\le 0.01n$ and $0.99n\le\|\zB_i\|^2\le2n$. Also define the projection
\begin{align*}
\QB:=\mathbf 1\{\AB_{-j}>c_6n\Lambda_i\IB\}.
\end{align*}
Taking $c_6>1$, it holds that $\QB\preceq\PB$. Since the range of $(i)$ is contained in $\ran\PiB_i$,
\begin{align*}
(\IB-\PiB_i)\QB
&=
(\IB-\PiB_i)\AB_{-j}
\big(\AB_{-j}|_{\ran\QB}\big)^{-1}\QB\\
&=
(\IB-\PiB_i)\big((ii)+(iii)\big)
\big(\AB_{-j}|_{\ran\QB}\big)^{-1}\QB,
\end{align*}
so we can ensure, for sufficiently large~$c_6$,
\begin{align*}
\|\QB(\IB-\PiB_i)\|\le \frac{\|(ii)\|+\|(iii)\|}{c_6n\Lambda_i} \le\frac1{20}.
\end{align*}
It follows that
\begin{align*}
\|\QB\zB_i\|
\le \|\PiB_i\zB_i\|+ \|\QB(\IB-\PiB_i)\|\|\zB_i\| \le \frac{\sqrt{n}}{5}.
\end{align*}
Also, since $\AB_{-j}\preceq0.9n\lam_i\IB$ on the range of $\IB-\PB$, we have the series of inequalities
\begin{align*}
\frac{\lam_i}{n}\|(\IB-\PB)\zB_i\|^4
&\le
\frac1n\zB_i^\top(\IB-\PB)\AB_{-j}(\IB-\PB)\zB_i \le
0.9\lam_i\|(\IB-\PB)\zB_i\|^2.
\end{align*}
It follows that $\|(\IB-\PB)\zB_i\|^2\le0.9n$ and hence
\begin{align*}
\|(\PB-\QB)\zB_i\|^2
&=
\|\zB_i\|^2-\|\QB\zB_i\|^2-\|(\IB-\PB)\zB_i\|^2 \ge0.05n.
\end{align*}
Therefore, using that $\PB,\QB$ commute with $\AB_{-j}$,
\begin{align*}
\big\|\AB_{-j}^{-1}\PB\zB_i\big\|^2
&\ge
\big\|\AB_{-j}^{-1}(\PB-\QB)\zB_i\big\|^2 \\
&\ge
(c_6n\Lambda_i)^{-2}\|(\PB-\QB)\zB_i\|^2
\ge
\frac{c}{n\Lambda_i^2}.
\end{align*}
Taking expectations proves the claim.
\end{proof}

We also show a lower bound for the effective bias.

\begin{lemma}\label{lem:effbias}
For any monotone shrinkage filter $g\in\mathcal G$, we have
\begin{align*}
    \Ebb \effBias(\hat\wB_g) \ge \bigg(1-\frac{1}{\tau}\bigg)^2 \big\|(\IB - g(\tau\SigmaB))\wB^* \big\|^2_{\SigmaB},\quad \tau\ge 1.
\end{align*}
\end{lemma}
\begin{proof}
For $x\ge 0$, $y>0$, it holds that
\begin{align*}
1-g(x) \ge \left(1-\frac{x}{y}\right)(1-g( y )).
\end{align*}
Plugging in $x=\hat\SigmaB_{ii}$ and $y=\tau\lam_i$ and applying Jensen's inequality,
\begin{align*}
\Ebb(1-\GB_{ii})^2 \ge \big(\Ebb(1-\GB_{ii})\big)^2 \ge \left(1-\frac{\Ebb\hat\SigmaB_{ii}}{\tau\lam_i}\right)^2 (1-g(\tau\lam_i))^2 = \bigg(1-\frac{1}{\tau}\bigg)^2 (1-g(\tau\lam_i))^2.
\end{align*}
Multiplying by $\lam_i\wB_i^{*2}$ and summing over~$i$ gives the statement.
\end{proof}

\begin{proof}[Proof of \Cref{thm:master-lb}]
Combining the preceding lemmas, we obtain
\begin{align*}
\Ebb(\zB_j^\top\PsiB\zB_i)^2
\ge c\Gamma_{ij}^2 \Ebb\big\|\AB_{-j}^{-1}\PB\zB_i\big\|^2
\ge\frac{c\Gamma_{ij}^2}{n\Lambda_i^2}, \quad i\le k, \ j\ne i.
\end{align*}
It follows that
\begin{align*}
\Ebb\bias(\hat\wB_g) &\ge \sum_{i\le k} \lam_i\wB_i^{*2} \sum_{j\ne i}\lam_j^2 \Ebb(\zB_j^\top\PsiB\zB_i)^2 \\
&\ge c\sum_{i\le k}\frac{\lam_i}{\Lambda_i^2} \left(\frac1n\sum_{j\ne i}\lam_j^2\Gamma_{ij}^2\right) \wB_i^{*2}.
\end{align*}
The residual lower bound follows from \Cref{lem:effbias} with $\tau=1.1$. If $g$ is concave, Jensen's inequality instead directly gives
\begin{align*}
\eB_i^\top(\IB-\GB)\eB_i
\ge 1-g(\eB_i^\top\hat\SigmaB\eB_i)
=1-g\left(\frac{\lam_i\|\zB_i\|^2}{n}\right)
\end{align*}
and hence
\begin{align*}
\Ebb\BB_{ii}
&\ge\lam_i\Ebb\left(1-g\left(\frac{\lam_i\|\zB_i\|^2}{n}\right)\right)^2
\ge\lam_i(1-g(\lam_i))^2.
\end{align*}
For the first conditional lower bound, from \Cref{lemma:basic:concentration}, with probability at least $1-\exp(-n/c_0)$,
\begin{align*}
\lambda_{\min}(\AB_{>k}) \ge \frac{2}{3}\sum_{i>k}\lambda_i  - c_1 n\lambda_{k+1} \ge \frac12\sum_{i>k}\lambda_i
\end{align*}
by choosing $c_2$ sufficiently large. Since~$g$ is a shrinkage filter, it follows that
\begin{align*}
\PsiB\preceq \AB^{-1}\preceq \AB_{>k}^{-1} \preceq \frac{2}{\sum_{i>k}\lam_i} \IB.
\end{align*}
Therefore for $i>k$, on $\|\zB_i\|^2\le 2n$,
\begin{align*}
\GB_{ii} =\lam_i\zB_i^\top\PsiB\zB_i \le \frac{4n\lam_i}{\sum_{j>k}\lam_j} \le \frac{4}{c_2} \le \frac12 \quad\implies\quad \BB_{ii} \ge \frac{\lam_i}{4},
\end{align*}
which proves the claim. The final bound follows from \Cref{lem:general-filter-diffuse}.
\end{proof}

\subsection{Proof of Corollary \ref{cor:same-index}}

\begin{proof}[Proof of \Cref{cor:same-index}]
We derive the head and tail lower bounds separately. The residual term follows directly from \Cref{thm:master-lb}.

\emph{The head term.}~
We apply \Cref{thm:master-lb} with $k=k^*$. By minimality of $k^*$,
\begin{align*}
\frac{\sum_{j>k^*}\lam_j}{n} \le \frac{\sum_{j>k^*-1}\lam_j}{n} \le c_2\lam_{k^*}.
\end{align*}
Hence for every $i\le k^*$, we have $\Lambda_i \le (c_2+2)\lam_i$ and
\begin{align*}
\alpha_{k^*} \le \left(\frac{\sum_{j>k^*}\lam_j}n\right)^2 + \lam_{k^*+1} \left(\frac{\sum_{j>k^*}\lam_j}n\right) \le c\lam_i^2.
\end{align*}
If
\begin{align*}
\frac{\sum_{j>k^*}\lam_j^2}{n}
\le \frac{1}{c_5}\left(\frac{\sum_{j>k^*}\lam_j}{n}\right)^2
\end{align*}
for sufficiently large~$c_5$, we immediately obtain from \Cref{thm:master-lb},
\begin{align*}
c_1\Ebb\bias(\hat\wB_g)
\ge \sum_{i\le k^*}\lam_i
\left(\frac{\sum_{j>k^*}\lam_j}{n\lam_i}\right)^2\wB_i^{*2} \ge
c\alpha_{k^*}\|\wB^*\|_{\SigmaB_{0:k^*}^{-1}}^2.
\end{align*}
Otherwise,
\begin{align*}
\lam_{k^*+1}\left(\frac{\sum_{j>k^*}\lam_j}{n}\right) \ge \frac{\sum_{j>k^*}\lam_j^2}{n}
> \frac{1}{c_5}\left(\frac{\sum_{j>k^*}\lam_j}{n}\right)^2
\end{align*}
implies that
\begin{align*}
\frac{\sum_{j>k^*}\lam_j}{n} < c_5\lam_{k^*+1},
\end{align*}
and so choosing $c_2$ sufficiently large, we have
\begin{align*}
\lam \ge c_2\lam_{k^*+1} - \frac{\sum_{i>k^*}\lam_i}{n} \ge 1.1\lam_{k^*+1}.
\end{align*}
Now fix $i\le k^*$ and consider the coefficient of $\wB_i^{*2}$ in \Cref{thm:master-lb}. By the definition of $C_g$, it holds that either $1-g(1.1\lam_i)\ge C_g$ or $g(0.9\lam_i)-g(\lam)\ge C_g$. In the first case, the residual term gives
\begin{align*}
\lam_i\left(1-g(1.1\lam_i)\right)^2
\ge
C_g^2\lam_i
\ge
cC_g^2\frac{\alpha_{k^*}}{\lam_i}.
\end{align*}
In the second case, for every $j>k^*$,
\begin{align*}
\Gamma_{ij}
&=(g(0.9\lam_i)-g(1.1\lam_j))_+ \ge
g(0.9\lam_i)-g(\lam)
\ge C_g,
\end{align*}
which further implies
\begin{align*}
\frac{\lam_i}{\Lambda_i^2} \Bigg(\frac1n\sum_{j>k^*}\lam_j^2\Gamma_{ij}^2\Bigg)
&\ge cC_g^2 \frac{\sum_{j>k^*}\lam_j^2}{n\lam_i} \ge cC_g^2\frac{\alpha_{k^*}}{\lam_i}.
\end{align*}
In both cases, the bias is lower bounded by $C_g^2 \alpha_{k^*}\|\wB^*\|_{\SigmaB_{0:k^*}^{-1}}^2$.

\emph{The tail term.}~
Note that
\begin{align*}
\lam_{k^*+1} \le \frac{1}{c_2} \left(\lam + \frac{\sum_{i>k^*}\lambda_i}{n}\right) \le \max\left\{\frac{2\lam}{c_2}, \frac{2}{c_2}\frac{\sum_{i>k^*}\lambda_i}{n}\right\}.
\end{align*}
If $\lam_{k^*+1} \le 2\lam/c_2$, then for all $i>k^*$, $1-g(1.1\lam_i) \ge 1-g(\lam)$ by taking $c_2>2.2$. Then
\begin{align*}
\big\|(\IB-g(1.1\SigmaB))\wB^* \big\|_{\SigmaB}^2 \ge (1-g(\lam))^2 \|\wB^*\|_{\SigmaB_{k^*:\infty}}^2.
\end{align*}
Otherwise, \Cref{thm:master-lb} immediately gives the tail term after increasing~$c_2$. The proof is complete.
\end{proof}

\subsection{Proof of Corollary \ref{thm:tikhonov}}

\begin{proof}[Proof of \Cref{thm:tikhonov}]
Set the reference filter as $g^*(z)=z/(z+\lam/p)$. From Bernoulli's inequality,
\begin{align*}
g^*(z)\le g_{\lam,p}(z)\le\min\{1,pz/\lam\}\le2g^*(z).
\end{align*}
The variance is then equal to $\sigma^2D/n$ up to a constant factor \citep[Proposition~2.1]{wu2026risk}. The min thresholding can be removed as in the proof of \Cref{thm:gd:finite-snr}. We proceed to bound the bias.

\emph{The upper bound.}~
Substituting $u=\lam/(x+\lam), v=\lam/(y+\lam) \in (0,1]$, we may write
\begin{align*}
r(x,y) = r_p(u,v)
:=
\frac{(p-(p-1)u)(p-(p-1)v)}{p}
\frac{u^p-v^p}{u-v}, \quad u\ne v.
\end{align*}
For integer $p$, this is equal to the GD kernel in
\Cref{lemma:schur:examples}, evaluated at $(1-u,1-v)$
with $t=p$, multiplied by the row and column factors
\begin{align*}
\frac{p-(p-1)u}{p+1-pu},\ \frac{p-(p-1)v}{p+1-pv}\le1.
\end{align*}
Hence $\|r_p\|_{\Sfrak} = \bigO(1)$ for integer $p$ (all Schur norms in $u,v$ are taken over $[0,1]^2$).

For noninteger $p$, write $p=mq$, where
$m=\lfloor p\rfloor$ and $1<q<2$. It is straightforward to check that
\begin{align*}
r_p(u,v) = \underbrace{\frac{p-(p-1)u}{q(m-(m-1)u^q)}}_{\le 1} \underbrace{\frac{p-(p-1)v}{q(m-(m-1)v^q)}}_{\le 1} \frac{q(u^q-v^q)}{u-v} r_m(u^q,v^q).
\end{align*}
Moreover, the integral representation
\begin{align*}
\frac{u^q-v^q}{u-v}
&=\frac{\sin(\pi(q-1))}{\pi(u-v)}
\int_0^\infty t^{q-2}
\left(
\frac{u^2}{u+t} - \frac{v^2}{v+t}
\right)\dif t \\
&= \frac{\sin(\pi(q-1))}{\pi}
\int_0^\infty t^{q-2}
\left(
\frac{u}{u+t}
+\frac{t}{u+t}\frac{v}{v+t}
\right)\dif t
\end{align*}
and \Cref{lemma:schur:factorization-bound} imply
\begin{align*}
\|(z\mapsto z^q)^{[1]}\|_{\Sfrak}
\le
\frac{2\sin(\pi(q-1))}{\pi}
\int_0^\infty\frac{t^{q-2}}{1+t}\dif t
=2.
\end{align*}
Thus
\begin{align*}
\|r_p\|_{\Sfrak}
\le q\|r_m\|_{\Sfrak}
\|(z\mapsto z^q)^{[1]}\|_{\Sfrak}
\le4\|r_m\|_{\Sfrak} =\bigO(1),
\end{align*}
and $\|r(\Rbb_{\ge 0},\Rbb_{\ge 0})\|_{\Sfrak} = \bigO(1)$ uniformly in $p$. The bias upper bound now follows by applying \Cref{thm:master-ub} with~$\lam/p$ in place of~$\lam$ and $I=H=\Rbb_{\ge 0},\delta=\exp(-k^*/c_0)$. The extra bias terms can be absorbed into the variance since $\lam/p\le\tilde\lam$,
$c_2\lam_{k+1}\le\tilde\lam$ and
\begin{align*}
\left(\frac{\sum_{i>k^*}\lam_i^2}{n} + \frac{k^*\lam_{k^*+1}^2}{n} +\frac{k^*}{n}\left(\frac{\lam}{p}\right)^2
\right)\|\wB^*\|_{\SigmaB_{0:k}^{-1}}^2 &\le
c\tilde\lam^2\|\wB^*\|_{\SigmaB_{0:k}^{-1}}^2\frac{D}{n}
\le c\|\wB^*\|_{\SigmaB_{0:k}}^2\frac{D}{n}
\le cb\sigma^2\frac{D}{n}.
\end{align*}

\emph{The lower bound.}~
We apply \Cref{cor:same-index} with~$\lam/p$ in place of~$\lam$. Note that the filter $g_{\lam,p}$ is concave and $1-g_{\lam,p}(\lam/p)=(1+1/p)^{-p}\ge e^{-1}$. Also, for $z\le3\lam/p$,
\begin{align*}
1-g_{\lam,p}(1.1z)\ge \left(1+\frac{3.3}{p}\right)^{-p}\ge e^{-3.3},
\end{align*}
whereas for $z\ge3\lam/p$,
\begin{align*}
g_{\lam,p}(0.9z)-g_{\lam,p}(\lam/p)
&\ge
\left(1+\frac1p\right)^{-p}
-
\left(1+\frac{2.7}{p}\right)^{-p} \ge
e^{-1}-\frac1{3.7}>0.
\end{align*}
Thus $C_g$ is bounded below uniformly in~$p$. The claimed bound follows.
\end{proof}

\end{document}